\documentclass[twoside]{article}

\usepackage[accepted]{aistats2025}
\usepackage[round]{natbib}

\usepackage{amsmath,amssymb,amsthm,bm,dsfont}
\usepackage{float}
\usepackage{xcolor}
\usepackage{mathtools}
\usepackage{booktabs,multirow,bigstrut,makecell}
\usepackage{enumitem}
\usepackage{algorithm}
\usepackage{algorithmic}
\usepackage{comment}
\usepackage[hidelinks]{hyperref}
\usepackage{tikz}
\usetikzlibrary{positioning,fit,quotes,arrows}
\usepackage{caption}
\usepackage{subcaption}

\definecolor{dkgreen}{rgb}{0,0.6,0}
\definecolor{gray}{rgb}{0.5,0.5,0.5}
\definecolor{mauve}{rgb}{0.58,0,0.82}

\definecolor{color1}{HTML}{53446B}
\definecolor{color2}{HTML}{807DBA}
\definecolor{color3}{HTML}{C0AECE}
\definecolor{color4}{HTML}{CE9292}
\definecolor{color5}{HTML}{DAB6B6}
\definecolor{color6}{HTML}{A40407}
\definecolor{color7}{HTML}{7484A8}
\definecolor{color8}{HTML}{F5BD1C}

\newcommand{\DG}[1]{
  {\color{orange} [DG: {#1}]}
 }

\newtheorem{thm}{Theorem}[section]
\newtheorem{lem}[thm]{Lemma}

\newtheorem{coro}[thm]{Corollary}

\theoremstyle{definition}
\newtheorem{dfn}{Definition}[section]

\newtheorem{asp}{Assumption}

\theoremstyle{remark}

\newcommand{\bb}{\bm{b}}

\newcommand{\br}{\bm{r}}

\newcommand{\Ib}{\mathbf{I}}

\newcommand{\Lb}{\mathbf{L}}

\newcommand{\Xb}{\mathbf{X}}
\newcommand{\Yb}{\mathbf{Y}}

\newcommand{\bA}{\bm{A}}
\newcommand{\bB}{\bm{B}}

\newcommand{\bH}{\bm{H}}
\newcommand{\bI}{\bm{I}}

\newcommand{\bR}{\bm{R}}
\newcommand{\bS}{\bm{S}}

\newcommand{\bX}{\bm{X}}

\newcommand{\bZ}{\bm{Z}}

\newcommand{\bbeta}{\bm{\beta}}

\newcommand{\balpha}{\bm{\alpha}}

\newcommand{\bmu}{\bm{\mu}}

\newcommand{\btheta}{\bm{\theta}}

\newcommand{\bpi}{\bm{\pi}}

\newcommand{\bphi}{\bm{\phi}}
\newcommand{\bpsi}{\bm{\psi}}
\newcommand{\bvarepsilon}{\bm{\varepsilon}}

\newcommand{\bSigma}{\bm{\Sigma}}
\newcommand{\bPhi}{\bm{\Phi}}
\newcommand{\bPsi}{\bm{\Psi}}

\newcommand{\bbE}{\mathbb{E}}

\newcommand{\bbN}{\mathbb{N}}

\newcommand{\bbP}{\mathbb{P}}

\newcommand{\bbR}{\mathbb{R}}

\newcommand{\cA}{\mathcal{A}}
\newcommand{\cB}{\mathcal{B}}

\newcommand{\cD}{\mathcal{D}}

\newcommand{\cH}{\mathcal{H}}
\newcommand{\cI}{\mathcal{I}}

\newcommand{\cM}{\mathcal{M}}
\newcommand{\cN}{\mathcal{N}}

\newcommand{\cQ}{\mathcal{Q}}

\newcommand{\cS}{\mathcal{S}}
\newcommand{\cT}{\mathcal{T}}
\newcommand{\cU}{\mathcal{U}}
\newcommand{\cV}{\mathcal{V}}

\newcommand{\argmin}{\mathop{\mathrm{argmin}}}
\newcommand{\argmax}{\mathop{\mathrm{argmax}}}

\newcommand{\bbone}{{\mathds{1}}}   

\newcommand{\norm}[1]{\lVert#1\rVert}
\newcommand{\absbig}[1]{\left\lvert#1\right\rvert}
\newcommand{\abs}[1]{\lvert#1\rvert}

\newcommand{\prth}[1]{\left(#1\right)}
\newcommand{\brck}[1]{\left[#1\right]}
\newcommand{\brce}[1]{\left\{#1\right\}}

\newcommand{\indep}{\mathrel{\perp\!\!\!\perp}}
\newcommand{\nindep}{\mathrel{\not\!\perp\!\!\!\perp}}

\DeclareMathOperator{\diag}{diag}
\DeclareMathOperator{\Var}{Var}

\begin{document}

%

%

\twocolumn[

\aistatstitle{Harnessing Causality in Reinforcement Learning With Bagged Decision Times}

\aistatsauthor{ Daiqi Gao \And Hsin-Yu Lai \And  Predrag Klasnja \And Susan A. Murphy }

\aistatsaddress{ Harvard University \And  Allen Institute \And University of Michigan \And Harvard University } ]

\begin{abstract}
We consider reinforcement learning (RL) for a class of problems with bagged decision times. A bag contains a finite sequence of consecutive decision times. The transition dynamics are non-Markovian and non-stationary within a bag. All actions within a bag jointly impact a single reward, observed at the end of the bag. For example, in mobile health, multiple activity suggestions in a day collectively affect a user's daily commitment to being active. Our goal is to develop an online RL algorithm to maximize the discounted sum of the bag-specific rewards. To handle non-Markovian transitions within a bag, we utilize an expert-provided causal directed acyclic graph (DAG). Based on the DAG, we construct states as a dynamical Bayesian sufficient statistic of the observed history, which results in Markov state transitions within and across bags. We then formulate this problem as a periodic Markov decision process (MDP) that allows non-stationarity within a period. An online RL algorithm based on Bellman equations for stationary MDPs is generalized to handle periodic MDPs. We show that our constructed state achieves the maximal optimal value function among all state constructions for a periodic MDP. Finally, we evaluate the proposed method on testbed variants built from real data in a mobile health clinical trial.
\end{abstract}

\section{INTRODUCTION} \label{sec:introduction}

Reinforcement learning (RL) algorithms that maximize rewards on an infinite horizon commonly model the underlying dynamical system as a stationary Markov decision process (MDP).
The state transition and reward function are assumed to be Markovian and time-invariant across all decision times.
However, these assumptions may not hold in many real-world problems when future variables depend on the history beyond the current decision time \citep{tang2024reinforcement} or when the environment changes over time \citep{padakandla2020reinforcement}. 
A special case that violates such assumptions is a sequence of bagged decision times \citep{tang2024reinforcement}.
Each bag is defined as all variables in a finite sequence of consecutive decision times, including the actions, rewards, and other related context variables.  
Here, we consider settings in which the transition dynamics of the variables are non-Markovian and non-stationary within a bag.
Additionally, there is only one reward in a bag, representing the cumulative effect of all actions in the bag.  
Our goal is to construct an online RL algorithm that learns within and between bags to maximize the discounted sum of the bag-specific rewards.
This is in contrast to the more classical goal of constructing an online RL algorithm that learns across decision times to maximize a discounted sum of decision-time-specific rewards.

This paper is motivated by our work on the next implementation of the mobile health (mHealth) intervention, HeartSteps \citep{liao2020personalized,spruijt2022advancing}. 
HeartSteps is designed to help prevent adverse metabolic and cardiac events
by encouraging individuals to increase and maintain physical activity (PA).
A key component focuses on whether or not to deliver an activity suggestion to prompt users to engage in short bouts of PA. 
Our goal is to design an RL algorithm that targets the daily reward of ``commitment to being active.'' 
Here, the daily reward can be impacted by all five decisions regarding the delivery of activity suggestions within a day.
The 30-minute step count after a suggestion is a mediator of the reward and is immediately observed after the intervention.
Each day can be viewed as a bag with five decision times and one reward.
Another example with bagged decision times is online education, where the system may personalize the sequence of learning sessions, materials, or hints within a daily lesson, with the reward being the quiz score at the end of the lesson. 
In autonomous driving, a sequence of actions is taken to complete each task in a series, with the reward of the bag being the successful completion of this task.

\paragraph{Main Contributions}
We consider  RL with bagged decision times.
(1) \textit{RL framework}.
We demonstrate how to utilize an expert-provided causal directed acyclic graph (DAG) to handle non-Markovian transitions within a bag.  Such causal information allows us to construct states such that the state transition is Markovian across and within a bag. 
Further, we show that with the constructed states, the problem can be framed as a periodic MDP, which allows non-stationary state transitions and reward functions in a period. 
Based on the above results, any RL algorithm for stationary MDPs based on Bellman optimality equations can be generalized to handle periodic MDPs and bagged decision times. 
(2) \textit{Optimality}.
There can be multiple constructed states that fit in the framework of a periodic MDP.
We demonstrate that a dynamical Bayesian sufficient statistic of the observed history yields the maximal optimal value function. 
Besides, we extend the Bellman optimality equations for a periodic MDP to accommodate time-varying discount factors.
(3) \textit{Application}.
We construct testbed variants using real data from the HeartSteps clinical trial and evaluate the advantage of the proposed algorithm on them.
We also empirically verify that our model-free RL algorithm is robust to misspecified assumptions in the DAG.

\section{RELATED WORK}

\paragraph{Causal RL.}
Causal RL is a class of methods that model and exploit cause-effect relationships.
\citet{lattimore2016causal,lu2020regret,lu2022efficient} primarily modeled causal relationships between multiple actions to reduce regret in bandits and MDPs.
In our work, causal relationships between all variables in a bag are simultaneously leveraged for state construction and policy learning. 
Mediators are the most useful components for improving the optimal value function.
\citet{zhang2020designing} developed online RL algorithms to find the optimal dynamic treatment regime. 
By leveraging a known causal DAG, they reduced the dimensionality of the policy space by gradually eliminating irrelevant variables while maintaining the optimal value function.
In contrast, we propose a method that directly identifies the best state based on the dynamical Bayesian sufficient statistic in a causal DAG.
To find the optimal mHealth intervention policy, \citet{wang2021optimizing} constructed a stochastic human simulator as a dynamic Bayesian network using psychological insights and empirical data. 
They then learned a policy offline using the simulator.
In our work, we use the causal DAG to develop an RL algorithm that updates the policy online.
\citet{tran2024inferring} estimated the causal effect of long-term treatments using short-term observations in offline RL without assuming mediation.
In contrast, we focus on online policy learning with the help of mediators.
\citet{deng2023causal} provided a review of causal RL for improving sample efficiency,
enhancing knowledge transfer,
and addressing confounding bias.

\paragraph{Causality-Based State Construction}
Previous works have also considered state construction based on causal information.
However, their definitions are insufficient for bagged decision times.
\citet{zhang2020invariant,zhang2021learning} connected causal feature sets \citep{peters2016causal} to state abstractions in block MDP and POMDP, but their definition does not account for time-varying reward functions.
\citet{huang2022action} defined action-sufficient state representation for POMDP, which can include latent (unobserved) states.
In contrast, our definition purposefully distinguishes between observed and unobserved states.
The causal state \citep{shalizi2001computational} refers to the equivalence classes of histories before time $t$ that lead to the same futures after time $t$, but it is not a definition of the state.
We follow \citet{rosas2020causal} to define the state as a dynamical Bayesian sufficient statistic, which is a stochastic process, to emphasize the time-varying nature of states within a bag.
We modified the original definition to ensure that all future rewards, not just the current reward, are independent of the history given the current state and action.

\paragraph{Periodic MDP}
For a periodic MDP, 
\citet{riis1965discounted,veugen1983numerical,hu2014near,wang2023scheduling} assumed the state transition function is known or can be estimated efficiently and find the optimal policy based on value iteration.
Under the general RL framework without a known transition function,
\citet{chen2023learning} proposed a policy gradient method to find optimal policies at different times and implemented it as an online deep RL algorithm.
\citet{aniket2024online} developed an online periodic upper confidence bound reinforcement learning-2 algorithm after augmenting the state with the cycle phase.
All the aforementioned works assume that a reward is observed at each decision time, the transition dynamics are inherently Markovian, and the discount factor is time-invariant.
However, we focus on bagged decision times with non-Markovian transitions and extend a general periodic MDP to allow time-varying discount factors.
A periodic MDP with a time-invariant discount factor can be converted into a stationary MDP by augmenting the state with the time index.
However, for a periodic MDP with time-varying discount factors, a stationary MDP with a properly defined discount factor and reward function can recover its optimal value function but may fail to preserve the ordering of suboptimal policies \citep{pitis2019rethinking}.
This is problematic when, for example, comparing two estimated policies learned from limited data, which are typically suboptimal.

\paragraph{Bagged Rewards}
\citet{tang2024reinforcement} considered a bag of sequential decision times along with a bagged reward. 
They assumed that only the sum of the rewards in the bag is observable, and the state transition function is stationary across all decision times.
They redistributed the bagged reward to each decision time, and the policy is then optimized with off-the-shelf RL algorithms for stationary MDPs.
In contrast, we allow non-stationarity within a bag and do not assume that the observed reward is a sum of instantaneous rewards.




\section{BAGGED DECISION TIMES} \label{sec:settings}

\begin{figure*}[t]
    \centering
    \begin{tikzpicture}[->, thick, main/.style={font=\sffamily}]
    \tikzstyle{fixedwidth} = [draw=none, text width=0.8cm, align=center]
    \matrix [column sep=0.1cm, row sep=0.35cm] {
          \node[fixedwidth] (R0) {$R_{d-1}$};
        & & & & & & & & 
        & \node[fixedwidth] (RK) {$R_{d}$}; 
        & & & & & & & & 
        & \node[fixedwidth] (RRK) {$R_{d + 1}$}; 
        & & & \node[fixedwidth] (R6) {}; \\
        & & \node[fixedwidth] (M1) {$M_{d, 1}$};
        & \node[fixedwidth] (M2) {\dots};
        & & \node[fixedwidth] (MK) {$M_{d, K}$};
        & & & & 
        & & \node[fixedwidth] (MM1) {$M_{d + 1, 1}$};
        & \node[fixedwidth] (MM2) {\dots};
        & & \node[fixedwidth] (MMK) {$M_{d + 1, K}$};
        & \\
        \\
        & \node[fixedwidth] (C1) {$C_{d, 1}$}; & \node[fixedwidth] (A1) {$A_{d, 1}$};
        & \node[fixedwidth] (A2) {\dots};
        & \node[fixedwidth] (CK) {$C_{d, K}$}; & \node[fixedwidth] (AK) {$A_{d, K}$};
        & & & & 
        & \node[fixedwidth] (CC1) {$C_{d + 1, 1}$}; & \node[fixedwidth] (AA1) {$A_{d + 1, 1}$};
        & \node[fixedwidth] (AA2) {\dots};
        & \node[fixedwidth] (CCK) {$C_{d + 1, K}$}; & \node[fixedwidth] (AAK) {$A_{d + 1, K}$};
        & \\
        \\
        & & \node[fixedwidth] (N1) {$N_{d, 1}$};
        & \node[fixedwidth] (N2) {\dots};
        & & \node[fixedwidth] (NK) {$N_{d, K}$};
        & & & & 
        & & \node[fixedwidth] (NN1) {$N_{d + 1, 1}$};
        & \node[fixedwidth] (NN2) {\dots};
        & & \node[fixedwidth] (NNK) {$N_{d + 1, K}$};
        & \\
        \node[fixedwidth] (E0) {$E_{d-1}$};
        & & & & & & & & 
        & \node[fixedwidth] (EK) {$E_{d}$}; 
        & & & & & & & & 
        & \node[fixedwidth] (EEK) {$E_{d + 1}$}; 
        & & & \node[fixedwidth] (E6) {}; \\
    };
    \path[every node/.style={font=\sffamily}]
        (R0) edge [color1] node [right] {} (RK)
        (RK) edge [color1] node [right] {} (RRK)
        (RRK) edge [color1] node [right] {} (R6)
        (E0) edge [color6] node [right] {} (EK)
        (EK) edge [color6] node [right] {} (EEK)
        (EEK) edge [color6] node [right] {} (E6)
        (E0) edge [color6] node [right] {} (R0)
        (EK) edge [color6] node [right] {} (RK)
        (EEK) edge [color6] node [right] {} (RRK)
        (R0) edge [color2] node [right] {} (M1)
        (R0) edge [color2] node [right] {} (MK)
        (M1) edge [color3] node [right] {} (RK)
        (MK) edge [color3] node [right] {} (RK)
        (E0) edge [color4, bend right=15] node [right] {} (M1)
        (E0) edge [color4] node [right] {} (MK)
        (E0) edge [color4] node [right] {} (N1)
        (E0) edge [color4] node [right] {} (NK)
        (N1) edge [color5] node [right] {} (EK)
        (NK) edge [color5] node [right] {} (EK)
        (A1) edge node [right] {} (M1)
        (AK) edge node [right] {} (MK)
        (C1) edge node [right] {} (M1)
        (CK) edge node [right] {} (MK)
        (A1) edge node [right] {} (N1)
        (AK) edge node [right] {} (NK)
        (RK) edge [color2] node [right] {} (MM1)
        (RK) edge [color2] node [right] {} (MMK)
        (MM1) edge [color3] node [right] {} (RRK)
        (MMK) edge [color3] node [right] {} (RRK)
        (EK) edge [color4, bend right=15] node [right] {} (MM1)
        (EK) edge [color4] node [right] {} (MMK)
        (EK) edge [color4] node [right] {} (NN1)
        (EK) edge [color4] node [right] {} (NNK)
        (NN1) edge [color5] node [right] {} (EEK)
        (NNK) edge [color5] node [right] {} (EEK)
        (AA1) edge node [right] {} (MM1)
        (AAK) edge node [right] {} (MMK)
        (CC1) edge node [right] {} (MM1)
        (CCK) edge node [right] {} (MMK)
        (AA1) edge node [right] {} (NN1)
        (AAK) edge node [right] {} (NNK);
    \end{tikzpicture}
    \caption{Causal DAG for bag $d$ and $d + 1$. The arrows pointing to the actions $A_{d, 1:K}$ are omitted.}
    \label{fig:dag}
\end{figure*}
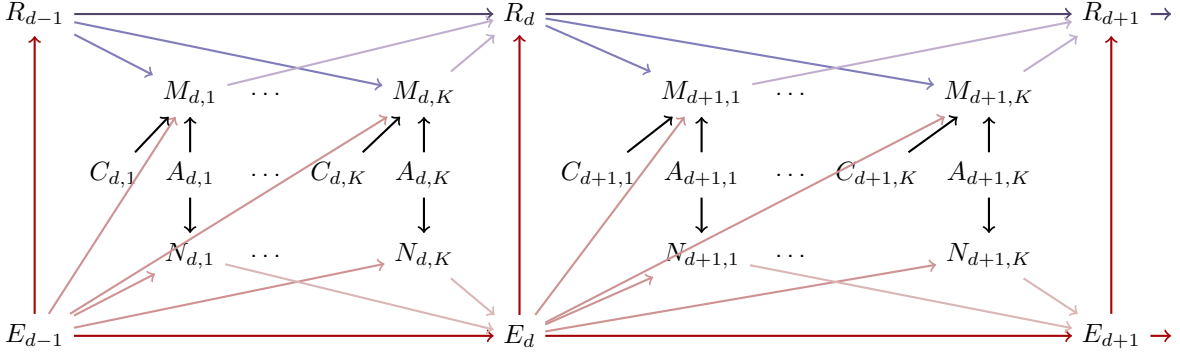

In this section, we define a problem with bagged decision times in RL based on a causal DAG \citep{pearl2009causality}.
In a DAG, a path is said to be \textit{d-separated} by a set of variables $Z$ if it contains a chain or a fork where the middle node is in $Z$, or if it contains a collider where the middle node or any of its descendants is not in $Z$.
A set $X$ and a set $Y$ are said to be d-separated by $Z$ if and only if $Z$ blocks every path from $X$ to $Y$, which guarantees that $X$ is independent of $Y$ given $Z$ in every distribution compatible with the DAG.
Suppose the decision times are indexed by the ordered set $\cI := \{(d, k): d \ge 1, k \in \{1:K\}\}$.
For two indices $(d, k), (t, l) \in \cI$, we have $(d, k) < (t, l)$ if $d = t, k < l$ or $d < t$, and $(d, k) = (t, l)$ if $d = t, k = l$.
We use the notation $i:j$ to represent the index set $i, \dots, j$, where the indices $i, j \in \bbN$.
When $i > j$, $i:j$ is an empty set.

\begin{dfn}[Bagged Decision Times] \label{dfn:bagged}
In a problem with bagged decision times, each variable $V_{d, k}$ is indexed by some $(d, k) \in \cI$.
A variable $V_{d, k}$ occurs before $V_{t, l}$ for any $(t, l) > (d, k)$.
A variable $V_{d, K}$ at time $k = K$ can be denoted as $V_d$.
A bag indexed by $d \ge 0$ contains a sequence of $K$ consecutive decision times $(d, 1:K)$ and a vector $\bB_d$ of all variables indexed by $(d, 1:K)$.
There is an action $A_{d, k} \in \cA_k$ at each time $(d, k)$ and a reward $R_d$ at the end of bag $d$.
The multivariate distribution of all variables $\bB_{0:\infty}$ can be described by a causal DAG.
There exists at least one variable $V_{d, k} \in \bB_d$ s.t. the reward $R_d$ is a descendant of $V_{d, k}$.
Let $\widetilde{\bB}_d \subseteq \bB_d$ be a vector of dimension $p_{\widetilde{\bB}}$ s.t. $\widetilde{\bB}_d$ d-separates $\bB_{d + 1}$ and $\bB_{0:d}$, but any proper subset of $\widetilde{\bB}_d$ does not d-separate $\bB_{d + 1}$ and $\bB_{0:d}$.
Given any value $\bb \in \bbR^{p_{\widetilde{\bB}}}$ and any fixed action vector $a_{1:K} \in \cA_1 \times \dots \times \cA_K$, the conditional distribution of $\bB_{d}$ is stationary across bags, i.e., $\bbP (\bB_{d + 1} | \widetilde{\bB}_{d} = \bb, A_{d + 1, 1:K} = a_{1:K}) = \bbP (\bB_{d} | \widetilde{\bB}_{d - 1} = \bb, A_{d, 1:K} = a_{1:K})$.
\end{dfn}

Definition~\ref{dfn:bagged} assumes stationarity across bags without making assumptions about dynamics within a bag.
The definition of the d-separator $\widetilde{\bB}_d$ does not introduce additional assumptions but only identifies the key structure in a bag.
We intrinsically assume that causal sufficiency is satisfied in the DAG, i.e. all the common causes of measured variables have been measured \citep{scheines1997introduction,spirtes2001causation}, so that conditioning on the history will not open a non-causal path between the actions and the rewards.

In a problem with bagged decision times, our goal is to construct an RL algorithm that maximizes the discounted sum of rewards $\bbE \{ \sum_{d=1}^{\infty} \bar{\gamma}^{d - 1} R_d \}$ given any $\widetilde{\bB}_0 \sim \nu$, where the discount factor $\bar{\gamma} \in [0, 1)$, and the initial distribution of the first d-separator $\widetilde{\bB}_0$ is $\nu$.
Define $H_{d, k}$ as the observed data since the previous action ($A_{d, k - 1}$ if $k \ge 2$ or $A_{d - 1, K}$ if $k = 1$) and before the current action $A_{d, k}$.
Let $\cH_{d, k} := \{ H_{1, 1}, \dots, H_{d, k} \}$ represent all the observed history before $A_{d, k}$.

Figure~\ref{fig:dag} is an example of a causal DAG that represents bagged decision times.
In Figure~\ref{fig:dag},
$A_{d, k} \in \cA_k$ is the action, 
$R_{d} \in \bbR$ is the reward,
$C_{d, k} \in \bbR^p$ is a $p$-dimensional context, 
$E_{d} \in \bbR$ is a mediator of the reward occurring at the end of the bag,
$M_{d, k} \in \bbR$ is a mediator of the reward $R_{d}$ occurring immediately after $A_{d, k}$,
and $N_{d, k} \in \bbR$ is a mediator of $E_{d}$ occurring immediately after $A_{d, k}$.
The reward $R_d$ and the mediator $E_d$ reflect the delayed effects of actions in a bag and evolve slowly from their last values.
The context $C_{d, k}$ is assumed to be exogenous.
The action $A_{d, k}$ impacts $R_{d}$ only through the mediators $M_{d, k}$ and $N_{d, k}$.
The mediators $N_{d, k}$ and $E_d$ are not directly affected by $C_{d, k}$ or $R_{d-1}$.
The variables $C_{d, k}, M_{d, k}, N_{d, k}$ do not directly affect $C_{d, k'}, M_{d, k'}, N_{d, k'}$ at another decision time.
The subset $\widetilde{\bB}_d = \{E_d, R_d\}$ blocks all causal paths from $\bB_{0:d}$ to $\bB_{d+1}$.
An episode consists of the sequence, 
$\{E_0, R_0\} \cup \{ C_{d, 1:K}, \allowbreak A_{d, 1:K}, \allowbreak M_{d, 1:K}, \allowbreak N_{d, 1:K}, \allowbreak E_{d}, \allowbreak R_{d} \}_{d=0}^{\infty}$.
The observed data after $A_{d, k - 1}$ and before $A_{d, k}$ is $H_{d, k} = [A_{d, k - 1}, M_{d, k - 1}, N_{d, k - 1}, C_{d, k}]$ when $k \ge 2$
and $H_{d, 1} = [A_{d - 1, K}, M_{d - 1, K}, N_{d - 1, K}, E_{d - 1}, R_{d - 1}, C_{d, 1}]$ when $k = 1$.



\section{METHODOLOGY} \label{sec:methodology}

In this section, we demonstrate how to define a state such that the state transition is Markovian within and across bags.
We then fit the problem with bagged decision times into the framework of a periodic MDP and compare different methods of state construction.

\subsection{Markov State Transition} \label{sec:state.construction}


Definition~\ref{dfn:bagged} guarantees that the transition across bags is Markovian, since $\bB_{d + 1} \indep \bB_{0:d} | \widetilde{\bB}_d$.
We construct a state at each decision time based on the dynamical Bayesian sufficient statistic \citep[D-BaSS,][]{rosas2020causal} and show that the state transition based on the D-BaSS is Markovian within a bag.

\begin{dfn}[D-BaSS]
    For two stochastic processes $\bH, \bR$, a process $\bS$ is a D-BaSS of $\bH$ with respect to (w.r.t) $\bR$ given a process $\bA$ if, for all index $t \in \cT$, the following conditions hold:
    (1) Precedence: there exists a function $F_t (\cdot)$ s.t. $S_t = F_t (H_{1:t})$ for all $t \in \cT$.
    (2) Sufficiency: $R_{t'} \indep H_{1:t} | S_t, A_t$ for any $t' \ge t$.
    Moreover, a stochastic process $\bZ$ is a minimal D-BaSS of $\bH$ w.r.t. $\bR$ if it is itself a D-BaSS, and for any D-BaSS $\bS$ there exists a function $f_t (\cdot)$ s.t. $f_t (S_t) = Z_t$ for all $t \in \cT$.
\end{dfn}


Now we define the state by taking $\bA = \{A_{d, k}\}_{\cI}$, $\bH = \{H_{d, k}\}_{\cI}$, and $\bR = \{R_{d} \}_{d \ge 1}$. 
Since $R_d$ is observed at the end of a bag, it is associated with time $(d, K)$.
For a D-BaSS $\{S_{d, k}\}_{\cI}$, the sufficiency condition thus requires $R_{d:\infty} \indep \cH_{d, k} | S_{d, k}, A_{d, k}$ for all $k \in \{1:K\}$.
Note that $S_{d, k}$ needs to block all the paths to the future rewards $R_{d:\infty}$.
Definition~\ref{dfn:bagged} guarantees that in a problem with bagged decision times, a D-BaSS always exists at any time $(d, k)$ and can be constructed as $(H_{d, 1} \cap (\bB_d \cup \widetilde{\bB}_{d - 1})) \cup \dots \cup H_{d, k}$, although it may not be minimal.

\begin{asp} \label{asp:state.D.BaSS}
    In a problem with bagged decision times as defined in Definition~\ref{dfn:bagged}, let $\{S_{d, k}\}_{\cI}$ be the D-BaSS of $\{H_{d, k}\}_{\cI}$ w.r.t. $\{R_{d}\}_{d \ge 1}$ given $\{A_{d, k}\}_{\cI}$, and $S_{d, 1:K} \subseteq \bB_d \cup \widetilde{\bB}_{d - 1}$.
\end{asp}

A D-BaSS contains ancestors but not necessarily parents of future rewards, as some parents may not be observed before taking an action.
The assumption $S_{d, 1:K} \subseteq \bB_d \cup \widetilde{\bB}_{d - 1}$ requires that $S_{d, 1}$ does not contain any variable from the last bag except for the d-separator $\widetilde{\bB}_{d - 1}$. 
It will be used to guarantee the stationary state transition across bags.
This is a weak assumption since any variable in $\cH_{d - 1, K} \backslash (\bB_d \cup \widetilde{\bB}_{d - 1})$ is d-separated from $S_{d, 1:K}$ by $\widetilde{\bB}_{d - 1}$.

\begin{lem} \label{lem:markov.transition}
    Under Assumption~\ref{asp:state.D.BaSS}, the state transition is Markovian within and across bags, and we have $R_{d} \indep S_{t, l}, A_{t, l} | S_{d, k}, A_{d, k}$ for any $(t, l) < (d, k)$. Further, $S_{d, k + 1} \indep \widetilde{\bB}_{d - 1} | S_{d, k}, A_{d, k}$ for $k \le K - 1$ and $R_d, S_{d + 1, 1} \indep \widetilde{\bB}_{d - 1} | S_{d, K}, A_{d, K}$.
\end{lem}

Lemma~\ref{lem:markov.transition} shows that a D-BaSS $S_{d,k}$ will d-separate the next state $S_{d, k + 1}$ or $S_{d + 1, 1}$ and the next reward $R_d$ from the previous states and the last bag separator $\widetilde{\bB}_{d - 1}$. 
Although we focus on the bagged decision times, the results in Lemma~\ref{lem:markov.transition} can be generalized to problems where a reward is observed at each time $k$.

Reading off the DAG in Figure~\ref{fig:dag}, we have that
\begin{equation}
    S_{d, k} = [E_{d - 1}, R_{d - 1}, M_{d, 1:(k-1)}, N_{d, 1:(k-1)}, C_{d, k}] \label{equ:state.definition}
\end{equation}
is a D-BaSS (see a proof in Appendix~\ref{sec:proof.example.dag}).
However, $S_{d, k}$ is not the only state that has Markov state transitions.
Some other possible states are
\begin{align}
    S^{\prime}_{d, k} &= [E_{d - 1}, R_{d - 1}, M_{d, 1:(k-1)}, A_{d, 1:(k-1)}, C_{d, k}], \label{equ:state.definition.2} \\
    S^{\prime \prime}_{d, k} &= [E_{d - 1}, R_{d - 1}, A_{d, 1:(k-1)}], \label{equ:state.definition.3}
\end{align}
which also have the properties in Lemma~\ref{lem:markov.transition}.
Note that $A_{d, 1:(k-1)}$ is excluded from the state $S_{d, k}$ since $M_{d, k}$ blocks the path from $A_{d, k}$ to $R_d$ and $N_{d, k}$ blocks the path from $A_{d, k}$ to $E_d$.  
All the above states $S_{d, k}, S^{\prime}_{d, k}, S^{\prime \prime}_{d, k}$ contain the d-separator $\widetilde{\bB}_{d - 1}$.
The advantage of a D-BaSS will be discussed in Theorem~\ref{thm:best.state}.

\subsection{$K$-Periodic MDP} \label{sec:periodic.mdp}

To define a $K$-periodic MDP, we momentarily consider the setting in which there is a state $S_{d, k}$, an action $A_{d, k}$, and a reward $R_{d, k}$ at each time $k$ in bag $d$. 
We extend the definition of \citet{riis1965discounted,hu2014near} to accommodate time-varying discount factors.
The space of probability distributions over a set $\cB$ is denoted as $\Delta (\cB)$.

\begin{dfn}[$K$-Periodic MDP] \label{def:periodic.mdp}
    A $K$-periodic MDP is a tuple $\cM := (\cS_{1:K}, \cA_{1:K}, \bbP_{1:K}, r_{1:K}, \gamma_{1:K}, \nu)$,
    where 
    $\nu \in \Delta (\cS_1)$ is the initial state distribution for $S_{1, 1}$.
    At time $k \in \{1:K\}$,
    $\cS_k$ is the state space, 
    $\cA_k$ is the action space, 
    $r_k: \cS_k \times \cA_k \mapsto \Delta(\bbR)$ is the reward function with $r_k (s_{d, k}, a_{d, k}) = \bbP (R_{d, k} | S_{d, k} = s_{d, k}, A_{d, k} = a_{d, k})$ for any $s_{d, k} \in \cS_k$ and $a_{d, k} \in \cA_k$,
    and $\gamma_k$ is the discount factor applied to rewards after time $k$.
    The $K$ discount factors satisfy that their product $\bar{\gamma} := \prod\nolimits_{k = 1}^K \gamma_k \in [0, 1)$.
    $\bbP_{1:K}$ are $K$ transition probability functions with $\bbP_{k}: \cS_{k - 1} \times \cA_{k - 1} \mapsto \Delta(\cS_{k})$ if $k > 1$ and $\bbP_{1}: \cS_{K} \times \cA_{K} \mapsto \Delta(\cS_{1})$ if $k = 1$. 
    When $k > 1$, the next state $S_{d, k}$ is sampled from $\bbP_{k} (\cdot | S_{d, k - 1}, A_{d, k - 1})$, and when $k = 1, d \ge 2$, the next state $S_{d, k}$ is sampled from $\bbP_{1} (\cdot | S_{d - 1, K}, A_{d - 1, K})$.
\end{dfn}

Definition~\ref{def:periodic.mdp} allows all components $\cS_k, \cA_k, \bbP_k, r_k, \gamma_k$ of an MDP to vary across $K$ decision times.
The discount factors capture time-varying preferences.
Notice that each $\gamma_k$ is unconstrained as long as $\bar{\gamma} \in [0, 1)$, which guarantees long-run discounting \citep{pitis2019rethinking}.
A $K$-Periodic MDP is stationary on the bag level with an infinite horizon, but contains a non-stationary MDP of horizon $K$ within a bag.
When all of $\cS_k, \cA_k, \bbP_k, r_k, \gamma_k$ do not depend on $k$, a $K$-Periodic MDP degenerates to a stationary MDP.

Let $\Pi$ be the class of all possibly history-dependent, stochastic policies.
Let $\widetilde{\Pi}$ be the class of all Markov, stationary, and stochastic policies $\widetilde{\bpi}$, where $\widetilde{\bpi} = \{ \widetilde{\pi}_{1:K} \}$ and $\widetilde{\pi}_{k}: \cS_k \mapsto \Delta (\cA_k)$.
In Appendix~\ref{sec:additional.theory}, we start by defining the value functions for a general policy $\bpi \in \Pi$, but will show that it is sufficient to restrict to $\widetilde{\Pi}$.
Write $\bbP^{\widetilde{\bpi}}$ and $\bbE^{\widetilde{\bpi}}$ as the probability and expectation under a policy $\widetilde{\bpi}$.

The Q-function of a policy $\widetilde{\bpi} \in \widetilde{\Pi}$ at time $k \in \{1:K\}$ in bag $d \ge 1$ for all $s_{d, k} \in \cS_{k}$, $a_{d, k} \in \cA_{k}$ is defined as 
\begin{equation} \label{equ:definition.Q}
    \begin{split}
        \cQ_k^{\widetilde{\bpi}} (s_{d, k}, a_{d, k}) := 
        \bbE^{\widetilde{\bpi}} \{ 
            & \textstyle{\sum_{(t, l): (t, l) \ge (d, k)}} \eta_{d, k}^{t, l} R_{t, l} | \\
            & S_{d, k} = s_{d, k}, A_{d, k} = a_{d, k} \},
    \end{split}
\end{equation}
where $\eta_{d, k}^{t, l} = \allowbreak \prod_{(m, n): (d, k) \le (m, n) < (t, l)} \gamma_n$.
For example, $\cQ_2^{\widetilde{\bpi}} (s_{d, 2}, a_{d, 2}) = \allowbreak \bbE^{\widetilde{\bpi}} \{ R_{d, 2} + \gamma_2 R_{d, 3} + \gamma_2 \gamma_3 R_{d, 4} + \dots | \allowbreak s_{d, 2}, a_{d, 2} \}$.
The optimal Q-function is 
$\cQ_k^{*} (s_{d, k}, a_{d, k}) = \sup_{\widetilde{\bpi} \in \Pi}\cQ_k^{\widetilde{\bpi}} (s_{d, k}, a_{d, k})$.
The value function of $\widetilde{\bpi} \in \widetilde{\Pi}$ is defined as
$\cV_k^{\widetilde{\bpi}} (s_{d, k}) := \bbE^{\widetilde{\bpi}} \{\textstyle{\sum_{(t, l): (t, l) \ge (d, k)}} \eta_{d, k}^{t, l} R_{t, l} | S_{d, k} = s_{d, k} \}$
for all $k \in \{1:K\}$ and all $s_{d, k} \in \cS_k$.
The optimal value function is $\cV_k^{*} (s_{d, k}) := \sup_{\widetilde{\bpi} \in \Pi} \cV_k^{\widetilde{\bpi}} (s_{d, k})$.
A policy $\widetilde{\bpi} \in \Pi$ is optimal if $\cV_1^{\widetilde{\bpi}} (s_{d, 1}) = \cV_1^{*} (s_{d, 1})$ for all $s_{d, 1} \in \cS_1$.

\begin{thm}[Bellman Optimality Equations] \label{thm:optimal.policy}
    We say that functions $\bm{\cQ} := \{ Q_k: \cS_k \times \cA_k \mapsto \bbR: k \in \{1:K\} \}$ satisfy the Bellman optimality equations if
    \begin{equation} \label{equ:Bellman.optimality.equations}
    \begin{split}
        & \cQ_K (s_{d, K}, a_{d, K}) 
        = \bbE \{ R_{d, K} + \gamma_K \max_{a_{d + 1, 1} \in \cA_1} \\
        & \quad \cQ_1 (S_{d + 1, 1}, a_{d + 1, 1}) |
        S_{d, K} = s_{d, K}, A_{d, K} = a_{d, K} \}, \\
        & \cQ_k (s_{d, k}, a_{d, k}) 
        = \bbE \{ R_{d, k} + \gamma_k \max_{a_{d, k + 1} \in \cA_{k + 1}} \\
        & \quad \cQ_{k + 1} (S_{d, k + 1}, a_{d, k + 1}) | 
        S_{d, k} = s_{d, k}, A_{d, k} = a_{d, k} \}
    \end{split}
    \end{equation}
    for $k \in \{1:(K - 1)\}$.
    For a $K$-periodic MDP,
    the functions $\bm{\cQ}$ are the optimal Q-functions, i.e. $Q_k = Q_k^*$ for all $k$, if and only if $\bm{\cQ}$ satisfy the Bellman optimality equations.
    In addition, 
    $\cV^*_k (s_{d, k}) = \max_{a_{d, k} \in \cA_k} \cQ_k^{*} (s_{d, k}, a_{d, k})$
    for all k.
    Define the policy $\bpi^* = \{\pi_{1:K}^*\}$, where
    $\pi_k^* (s_{d, k}) = \argmax_{a_{d, k} \in \cA_k} \cQ_k^{*} (s_{d, k}, a_{d, k})$.
    Then $\bpi^* \in \Pi$ is an optimal policy (ties are broken arbitrarily).
\end{thm}

We only need to find $K$ functions $\bm{\cQ}$ that satisfy the Bellman optimality equations.
Theorem~\ref{thm:optimal.policy} indicates that $\bm{\cQ}$ must be the Q-functions of the optimal policy.

In the following assumption, we provide the condition for a general state $U_{d, k}$ in a problem with bagged decision times to fit in the framework of a periodic MDP.

\begin{asp} \label{asp:markovian.state}
    In a problem with bagged decision times as defined in Definition~\ref{dfn:bagged}, let $\{U_{d, k}\}_{\cI}$ be the states s.t. 
    (1) $U_{d, k} \in \cU_k$ where $\cU_k$ is bag-invariant, 
    (2) the state transition is Markovian within and across bags, 
    (3) $R_{d} \indep U_{t, l}, A_{t, l} | U_{d, K}, A_{d, K}$ for any $(t, l) < (d, K)$, 
    (4) $U_{d, 1:K} \subseteq \bB_d \cup \widetilde{\bB}_{d - 1}$, $U_{d, k + 1} \indep \widetilde{\bB}_{d - 1} | U_{d, k}, A_{d, k}$ for $k \le K - 1$, and $R_d, U_{d + 1, 1} \indep \widetilde{\bB}_{d - 1} | U_{d, K}, A_{d, K}$,
    and (5) precedence: there exists a function $F_{d, k} (\cdot)$ s.t. $U_{d, k} = F_{d, k} (\cH_{d, k})$ for all $(d, k) \in \cI$.
\end{asp}

In Assumption~\ref{asp:markovian.state}, conditions (2) and (3) ensure that the state transition and reward function are Markovian. 
Conditions (1) and (4) ensure that the state transition and reward function are stationary across bags.
Condition (4) assumes that $U_{d, k}, A_{d, k}$ block the paths from $\widetilde{\bB}_{d - 1}$ to the next state $U_{d, k + 1}$ in the bag.
If condition (4) does not hold and the marginal distribution of $\widetilde{\bB}_{d}$ is non-stationary across bags, the effect of $\widetilde{\bB}_{d - 1}$ on $U_{d, k + 1}$ may cause the state transition probability $\bbP_k (U_{d, k + 1} | U_{d, k}, A_{d, k})$ to be non-stationary across bags.
Lemma~\ref{lem:markov.transition} guarantees that any state satisfying Assumption~\ref{asp:state.D.BaSS} will satisfy conditions (2)-(4) in Assumption~\ref{asp:markovian.state}.

\begin{lem} \label{lem:rl.framework}
    With the states $\{U_{d, k}\}_{\cI}$ satisfying Assumption~\ref{asp:markovian.state}, the rewards defined as $R_{d, 1:(K - 1)} = 0$, $R_{d, K} = R_d$, and the discount factors defined as $\gamma_{1:(K - 1)} = 1$, $\gamma_K = \bar{\gamma}$, the problem with bagged decision times is a $K$-periodic MDP.
\end{lem}

Based on Theorem~\ref{thm:optimal.policy} and Lemma~\ref{lem:rl.framework}, the optimal policy for a problem with bagged decision times is $\bpi^* = \{\pi_{1:K}^*\}$, where
\begin{equation} \label{equ:optimal.policy}
    \begin{split}
        \pi_K^* (& u_{d, K}) = \argmax_{a_{d, K} \in \cA_K}
        \bbE \{ R_{d} + \bar{\gamma} \max_{a_{d + 1, 1} \in \cA_{1}} \cQ^*_1 \\
        & (U_{d + 1, 1}, a_{d + 1, 1}) |
        U_{d, K} = u_{d, K}, A_{d, K} = a_{d, K} \}, \\
        \pi_k^* (& u_{d, k}) = \argmax_{a_{d, k} \in \cA_k}
        \bbE \{ 
        \max_{a_{d, k + 1} \in \cA_{k + 1}} \cQ^*_{k + 1} \\
        & (U_{d, k + 1}, a_{d, k + 1}) | U_{d, k} = u_{d, k}, A_{d, k} = a_{d, k} \}
    \end{split}
\end{equation}
for $k \in \{1:(K - 1)\}$.

\subsection{State Construction} \label{sec:compare.state}

Now we provide conditions for the best state construction.

\begin{asp} \label{asp:definition.S.U}
    Suppose the states $\{S_{d, k}\}_{\cI}$ is the minimal D-BaSS of $\{H_{d, k}\}_{\cI}$ w.r.t. $\{R_{d}\}_{d \ge 1}$ given $\{A_{d, k}\}_{\cI}$ with $S_{d, k} \in \cS_k$. 
    Besides, $\cS_k$ is bag-invariant and $S_{d, 1:K} \subseteq \bB_d \cup \widetilde{\bB}_{d - 1}$.
\end{asp}

Assumption~\ref{asp:definition.S.U} differs from Assumption~\ref{asp:state.D.BaSS} by additionally requiring the minimal D-BaSS and the stationarity of $\cS_k$.
Lemmas~\ref{lem:markov.transition} and~\ref{lem:rl.framework} guarantee that $\{S_{d, k}\}_{\cI}$ and $\{U_{d, k}\}_{\cI}$ satisfy the definition of a $K$-periodic MDP.
As discussed in Section~\ref{sec:state.construction}, a D-BaSS always exists, which further guarantees that $\{U_{d, k}\}_{\cI}$ exists.
Consider two policy classes: 
$\Pi^{\cS} := \{ \bpi^{\cS} = \{\pi^{\cS}_{1:K}\}: \pi^{\cS}_k: \cS_k \mapsto \Delta(\cA_k), \forall k \}$ and
$\Pi^{\cU} := \{ \bpi^{\cU} = \{\pi^{\cU}_{1:K}\}: \pi^{\cU}_k: \cU_k \mapsto \Delta(\cA_k), \forall k \}$.
The value function $\cV_k^{\bpi} (u_{d, k})$ refers to 
$\bbE^{\bpi} \{\textstyle{\sum_{(t, l): (t, l) \ge (d, k)}} \eta_{d, k}^{t, l} R_{t, l} | U_{d, k} = u_{d, k} \}$
for $u_{d, k} \in \cU_k$,
though the policy $\bpi$ may not directly depend on $u_{d, k}$.
Define the optimal value functions in each policy class as $\cV_k^{\cU*} (u_{d, k}) := \max_{\bpi^{\cU} \in \Pi^{\cU}} \cV_k^{\bpi^{\cU}} (u_{d, k})$ and $\cV_k^{\cS*} (u_{d, k}) := \max_{\bpi^{\cS} \in \Pi^{\cS}} \cV_k^{\bpi^{\cS}} (u_{d, k})$, respectively.
By Theorem~\ref{thm:optimal.policy}, for the class $\Pi^{\cS}$,
$\cV_k^{\cS*} (u_{d, k}) = \bbE_{S_{d, k} \sim \bbP (\cdot | U_{d, k})} \{\cQ^{\cS*} (S_{d, k}, \pi_k^{\cS*} (S_{d, k})) | U_{d, k} = u_{d, k}\}$.

\begin{thm} \label{thm:best.state}
    Suppose $\{S_{d, k}\}_{\cI}$ satisfies Assumption~\ref{asp:definition.S.U} and $\{U_{d, k}\}_{\cI}$ satisfies Assumption~\ref{asp:markovian.state} with $U_{d, k} \in \cU_k$.
    Then, we have 
    $\cV_k^{\cU*} (u_{d, k}) \le \cV_k^{\cS*} (u_{d, k})$ for all $k \in \{1:K\}$ and all $u_{d, k} \in \cU_k$.
    The inequality strictly holds if and only if 
    $P \{ \cQ_k (S_{d, k}, a) = \max_{a' \in \cA_k} \cQ_k (S_{d, k}, a') | U_{d, k} = u_{d, k} \} < 1$ for all $a \in \cA_k$.
    For any D-BaSS $\{U_{d, k}\}_{\cI}$ of $\{H_{d, k}\}_{\cI}$ w.r.t. $\{R_{d}\}_{d \ge 1}$ given $\{A_{d, k}\}_{\cI}$, we have $\cV_1^{\cU*} (u_{1, 1}) = \cV_1^{\cS*} (u_{1, 1})$ for all $u_{1, 1} \in \cU_1$.
\end{thm}

Theorem~\ref{thm:best.state} suggests that 
the value function of $\Pi^{\cU}$ is strictly less than that of $\Pi^{\cS}$ if no single action is optimal for all states $S_{d, k}$ given $u_{d, k}$, i.e., different states $S_{d, k}$ in the support of $\bbP (S_{d, k} | u_{d, k})$ may have different optimal actions.
Given $u_{d, k}$, we can only make decisions based on the conditional distribution $\bbP (S_{d, k} | u_{d, k})$.
However, if we can observe the outcome $S_{d, k}$ of $u_{d, k}$, we can choose better actions since $S_{d, k}$ explains part of the randomness on the path $U_{d, k} \to R_{d, k}$.

Although any D-BaSS of the observed history yields the maximal optimal value function, the optimal policy may be easier to learn with a minimal D-BaSS.
The regret of an RL algorithm based on a D-BaSS is generally larger than the regret based on a minimal D-BaSS if the minimal D-BaSS has a smaller dimension.
The upper and lower regret bounds for MDPs \citep{auer2008near,he2021nearly,jin2020provably} and the upper bound for periodic MDPs \citep{aniket2024online} have been shown to depend on the size of the state, suggesting that reducing the state dimension will reduce the regret.

In Figure~\ref{fig:dag}, the state $S_{d, k}$ in (\ref{equ:state.definition}) is a D-BaSS.
However, the states $S^{\prime}_{d, k}$ in (\ref{equ:state.definition.2}) and $S^{\prime \prime}_{d, k}$ in (\ref{equ:state.definition.3}) are not a D-BaSS if the causal faithfulness or stability condition holds \citep[Section 2.4]{pearl2009causality}, which requires that any conditional independency is implied by the DAG through d-separation \citep{ramsey2012adjacency}.
A side note is that $S_{d, k}$ may not be the minimal D-BaSS if there exists a function $f_{d, k}$ s.t. $f_{d, k}$ is not invertible and $f_{d, k} (S_{d, k})$ remains a D-BaSS, e.g. when there exists a discretization function $f_{d, k}$ s.t. $R_{d:\infty} \indep \cH_{d, k} | f_{d, k} (S_{d, k}), A_{d, k}$.
In the following Corollary~\ref{lem:state.construction.HS}, assume that $S_{d, k} \in \cS_k$, $S^{\prime}_{d, k} \in \cS^{\prime}_k$, and $S^{\prime \prime}_{d, k} \in \cS^{\prime \prime}_k$.

\begin{coro} \label{lem:state.construction.HS}
    For states $S_{d, k}$ and $S^{\prime}_{d, k}$ in Figure~\ref{fig:dag}, the optimal value function satisfies $\cV_k^{\cS^{\prime}*} (s^{\prime}_{d, k}) < \cV_k^{\cS*} (s^{\prime}_{d, k})$ if and only if $P \{ \cQ_k (S_{d, k}, a) = \max_{a' \in \cA_k} \cQ_k (S_{d, k}, a') | S^{\prime}_{d, k} = s^{\prime}_{d, k} \} < 1$ for all $a \in \cA_k$.
    Similarly, for states $S^{\prime}_{d, k}$ and $S^{\prime \prime}_{d, k}$, $\cV_k^{\cS^{\prime \prime}*} (s^{\prime \prime}_{d, k}) < \cV_k^{\cS^{\prime}*} (s^{\prime \prime}_{d, k})$ if and only if $P \{ \cQ_k (S^{\prime}_{d, k}, a) = \max_{a' \in \cA_k} \cQ_k (S^{\prime}_{d, k}, a') | S^{\prime \prime}_{d, k} = s^{\prime \prime}_{d, k} \} < 1$ for all $a \in \cA_k$.
\end{coro}

Corollary~\ref{lem:state.construction.HS} suggests that it is helpful to identify and leverage the mediators in RL, like $M_{d, k}$ and $N_{d, k}$ in Figure~\ref{fig:dag}.
For example, in HeartSteps, the post 30-minute step count is the mediator between activity suggestions and commitment to being active, and user engagement is the mediator between social media marketing activities and self-brand connection in marketing \citep{ibrahim2024role}. 
Using observed mediators for state construction may improve the value function of the optimal policy.

\subsection{RL Algorithm} \label{sec:intervention.algorithm}

We adapt our algorithm from Randomized Least-Squares Value Iteration (RLSVI) \citep{osband2016generalization}, a Bayesian algorithm for finite-horizon MDPs.
RLSVI updates the posterior distribution of the parameters of the optimal Q-functions based on value iteration, which is derived from the Bellman optimality equation.
It then selects the greedy action with respect to a sample drawn from the posterior distribution.
This approach is model-free and thus robust to misspecified assumptions in the DAG.
For bagged decision times, we update the posterior distribution based on (\ref{equ:optimal.policy}).

\begin{algorithm}[tb]
    \caption{Bagged RLSVI (BRLSVI)}
    \label{alg:brlsvi}
    \textbf{Input}: Hyperparameters $L, \lambda_{d}, \sigma^2$.
    
    \begin{algorithmic}[1] 
    \STATE Warm-up: Randomly take actions $A_{d, k} \sim \text{Bernoulli} (0.5)$ in bag $d \in \{1:L\}$ for $k \in \{1:K\}$.
    \FOR{$d \ge L + 1$}
        \FOR{$k = K, \dots, 1$}
            \STATE Construct $\bX_{1:(d - 1), k}$, $Y_{1:(d - 1), k}$ with $\widetilde{\bbeta}_{d - 1, 1}$ when $k = K$ and with   $\widetilde{\bbeta}_{d, k + 1}$ when $k < K$ using (\ref{equ:BRLSVI.variables}). Obtain $\bmu_{d, k}, \Sigma_{d, k}$ using (\ref{equ:update.beta.normal}). 
            \STATE Draw $\widetilde{\bbeta}_{d, k} \sim N(\bmu_{d, k}, \Sigma_{d, k})$. 
        \ENDFOR
        \FOR{$k = 1, \dots, K$}
            \STATE Observe $H_{d, k}$ and construct $S_{d, k}$. 
            \STATE Take $A_{d, k} = \argmax_{a \in \cA_k} \phi_k (S_{d, k}, a)^T \widetilde{\bbeta}_{d, k}$.
        \ENDFOR
    \ENDFOR
    \end{algorithmic}
\end{algorithm}

We model the Q-functions as a linear function
\begin{equation} \label{equ:q.function}
    \cQ_k (s_{d, k}, a_{d, k}) = \phi_k (s_{d, k}, a_{d, k})^T \bbeta_k,
\end{equation}
where $\phi_k (s_{d, k}, a_{d, k})$ is the basis function of $s_{d, k} \in \cS_k$ and $a_{d, k} \in \cA_k$, and $\bbeta_k$ is the parameter to be learned.
For $t \in \{1:(d - 1)\}$, let
\begin{equation} \label{equ:BRLSVI.variables}
\begin{split}
    \bX_{t, k} =& \phi_k (S_{t, k}, A_{t, k}), \\
    Y_{t, k} =& 
    \begin{cases}
        \max\limits_{a \in \cA_{k + 1}} \phi_{k + 1} (S_{t, k + 1}, a)^T \widetilde{\bbeta}_{d, k + 1} 
        & \text{if } k < K, \\
        R_{t} + \bar{\gamma} \max\limits_{a \in \cA_1} \phi_{1} (S_{t + 1, 1}, a)^T \widetilde{\bbeta}_{d - 1, 1}
        & \text{if } k = K,
    \end{cases}
\end{split}
\end{equation}
where $\widetilde{\bbeta}_{d, 2:K}$ and $\widetilde{\bbeta}_{d - 1, 1}$ are drawn from the posterior distribution.
Define $\Xb_{d, k} = [\bX_{1:(d - 1), k}]^T$, $\Yb_{d, k} = [Y_{1:(d - 1), k}]^T$.
We fit a Bayesian linear regression model for $\Yb_{d, k}$ using $\Xb_{d, k}$.
The posterior of $\bbeta_{d, k}$ is given by a normal distribution $N(\bmu_{d, k}, \Sigma_{d, k})$, where
\begin{equation} \label{equ:update.beta.normal}
    \begin{split}
        \bSigma_{d, k} &= \prth{\Xb_{d, k}^T \Xb_{d, k} / \sigma^2 + \lambda_{d} \Ib}^{-1}, \\
        \bmu_{d, k} &= \bSigma_{d, k} (\Xb_{d, k}^T \Yb_{d, k} / \sigma^2),
    \end{split}
\end{equation}
and $\Ib$ is the identity matrix.
The algorithm Bagged RLSVI for a problem with bagged decision times is summarized in Algorithm~\ref{alg:brlsvi}.

\section{APPLICATION} \label{sec:experiment}

In this section, we present an overview of the motivating problem concerning the redevelopment of the HeartSteps intervention
\citep{liao2020personalized,spruijt2022advancing}. 
While this overview highlights the key elements of the planned intervention, it is a simplified version. 
Additionally, we construct a testbed using real data from HeartSteps V2 and assess the performance of the proposed algorithm across different variants of the testbed, which represent potential environments in the upcoming trial.

\subsection{mHealth Intervention}

Figure~\ref{fig:dag} can be viewed as a simplified version of the expert-formed DAG describing the causal relationships in HeartSteps.
In the currently planned implementation, each user experiences the intervention for 36 weeks, equivalent to $D = 252$ days.
Each day is a bag containing $K = 5$ decision times.
All $K$ actions may impact the reward for the bag.

The action $A_{d, k} \in \cA_k = \{0, 1\}$ represents whether a notification is sent (1) or not (0). The notifications are designed to prompt the user to complete short-term tasks, such as a 10-minute walk.
The reward $R_{d} \in \bbR$ is defined as the daily commitment to being active, measured through surveys.
The context $C_{d, k}$ is the logarithm of the 30-minute step count prior to the $k$th decision time.
The proximal outcome $M_{d, k} \in \bbR$ is the logarithm of the 30-minute step count after the $k$th decision time.
The engagement $E_{d} \in \bbR$ is the square root of the weighted average of daily app page view times over the past 7 days. A higher number of views indicates better engagement with the app.
Since the action $A_{d, k}$ is designed to prompt short-term PA, it is expected to impact the daily reward $R_d$ only through the proximal outcome $M_{d, k}$.
The $K$ interventions are separated by several hours, so $M_{d, 1:K}$ are considered conditionally independent given $R_{d-1}$ and $E_{d-1}$.
App engagement $E_d$ is expected to be affected by the notifications received each day, as users tend to disengage from the app if they receive notifications too frequently. 
We expect the positive effect of actions to go through $M_{d, k}$ and the negative effect to go through $E_d$.
The mediator $N_{d, k}$ is not measured in HeartSteps V2 and thus will not be used in the following analysis.

\subsection{Implementation}

Since the mediator $N_{d, k}$ is not measured, a D-BaSS of $\{H_{d, k}\}_{\cI}$ is $S^{\prime}_{d, k} = [E_{d - 1}, \allowbreak R_{d - 1}, \allowbreak M_{d, 1:(k - 1)}, \allowbreak A_{d, 1:(k - 1)}, \allowbreak C_{d, k}]$.
Denote the expanded state vector as
$ \widetilde{S}^{\prime}_{d, k} := [E_{d - 1}, R_{d - 1}, \widetilde{M}_{d, 1:(K - 1)}, \widetilde{A}_{d, 1:(K - 1)}, C_{d, k}, k]$, 
where $\widetilde{M}_{d, j} = \bbone (j < k) M_{d, j}$ and $\widetilde{A}_{d, j} = \bbone (j < k) A_{d, j}$.
Then we define the basis function $\phi_k (S^{\prime}_{d, k}, A_{d, k})$ as
\begin{equation*}
    \begin{split}
    \phi (\widetilde{S}^{\prime}_{d, k}, & A_{d, k})^T
    := [
    1, k, 
    E_{d - 1}, k E_{d - 1}, 
    R_{d - 1}, k R_{d - 1}] \\
    & \frown [\widetilde{M}_{d, 1:(K - 1)}, \widetilde{A}_{d, 1:(K - 1)}, C_{d, k}] \\
    & \frown \bbone (k = 1) A_{d, k} [1, E_{d - 1}, R_{d - 1}, C_{d, k}]
    \frown \dots \\
    & \frown \bbone (k = K) A_{d, k} [1, E_{d - 1}, R_{d - 1}, C_{d, k}],
    \end{split}
\end{equation*}
where $\frown$ represents vector concatenation and the basis function $\phi$ does not depend on the time index.
Now the Q-function becomes $\cQ_k (s^{\prime}_{d, k}, a_{d, k}) = \phi (\widetilde{s}^{\prime}_{d, k}, a_{d, k})^T \bbeta$,
where the parameter $\bbeta$ is the same for all $k$.
Data across different decision times and bags are pooled to estimate treatment-free effects, but interaction effects are allowed to remain time-varying. 
Compared to modeling $\cQ_k$ separately for each $k$, this approach reduces the total number of parameters from 60 to 35.
More importantly, the effective sample size for estimating the main effect is increased by $K$ times, while time-varying policies are maintained.
The effects of unobserved $A_{d, j}$ and $M_{d, j}$ for future decision times $j \ge k$ are integrated into the variables $k, k E_{d-1}, k R_{d-1}$.
Based on the DAG, there does not exist a direct causal effect between the mediators $M_{d, 1:K}$ in a bag.
The basis function thus assumes that the impact of $M_{d, k}$ on $R_d$ and the impact of $A_{d, k}$ on $E_d$ are approximately the same for each $k$, leading to the terms $k, k E_{d - 1}, k R_{d - 1}$.

\subsection{Simulation Testbed} \label{sec:testbed}

HeartSteps V2 contains 42 users after data cleaning.
In the following numerical experiments, all variables are standardized.
Due to user heterogeneity, each user represents a different environment. 
Details regarding the testbed are provided in Appendices~\ref{sec:data}-\ref{sec:construct.testbed.variants}.

\paragraph{Variants of Testbed}
To construct different variants of the testbed, we follow guidelines from behavioral science that define small, medium, and large standardized treatment effects (STE) \citep{cohen1988statistical}.
The STE reflects the difficulty of learning the optimal policy from noisy data.
Based on these guidelines, we create the following four variants to study how the algorithms perform under different positive and negative effects of actions.
\begin{enumerate}[label={(\alph*)}]
    \item Vanilla testbed. Medium STE.
    \item Enhance the positive effects by increasing the effect of $A_{d, k} \to M_{d, k}$ for all $k$. Large STE.
    \item Enhance the negative effects by increasing the effect of $E_d \to R_d$ and decreasing the effect of $A_{d, k} \to E_d$ for all $k$. Small STE.
    \item Enhance the positive and negative effects simultaneously by combining the changes in Variants (\subref{subfig:v2}) and (\subref{subfig:v3}). Medium STE.
\end{enumerate}

To examine the impact of misspecified assumptions in Figure~\ref{fig:dag}, we construct three variants of the testbed that violate the assumptions in the DAG.
\begin{enumerate}[label={(\alph*)}]
    \setcounter{enumi}{4}
    \item The arrow $R_{d-1} \to E_d$ exists. Medium STE.
    \item The arrows $A_{d, 1:K} \to R_d$ exist, representing the effect of another unobserved mediator. Medium STE.
    \item The arrows $R_{d-1}, E_{d-1} \to C_{d, 1:K}$ exist, so that the contexts are not exogenous. Medium STE.
\end{enumerate}

\subsection{Simulation Results} \label{sec:simulation.results}

\begin{figure}[t]
    \centering
    \begin{subfigure}{0.23\textwidth}
        \centering
        \includegraphics[width=\textwidth]{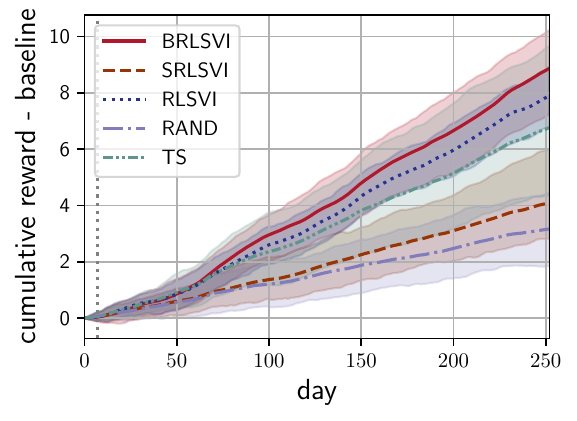}
        \caption{Vanilla testbed.}
        \label{subfig:v1}
    \end{subfigure}
    \hfill
    \begin{subfigure}{0.23\textwidth}
        \centering
        \includegraphics[width=\textwidth]{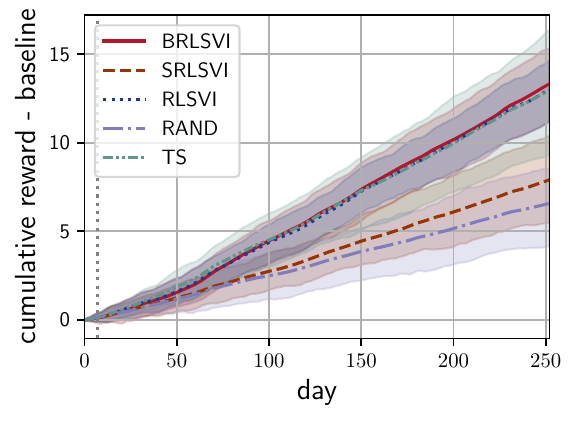}
        \caption{Enhance $M \to R$.}
        \label{subfig:v2}
    \end{subfigure}
    \hfill
    \begin{subfigure}{0.23\textwidth}
        \centering
        \includegraphics[width=\textwidth]{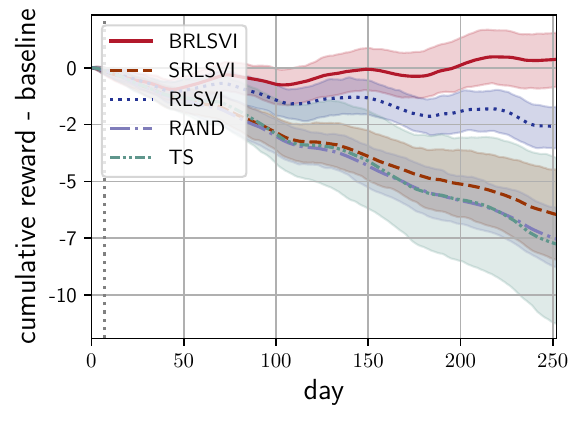}
        \caption{Enhance $A \to E \to R$.}
        \label{subfig:v3}
    \end{subfigure}
    \hfill
    \begin{subfigure}{0.23\textwidth}
        \centering
        \includegraphics[width=\textwidth]{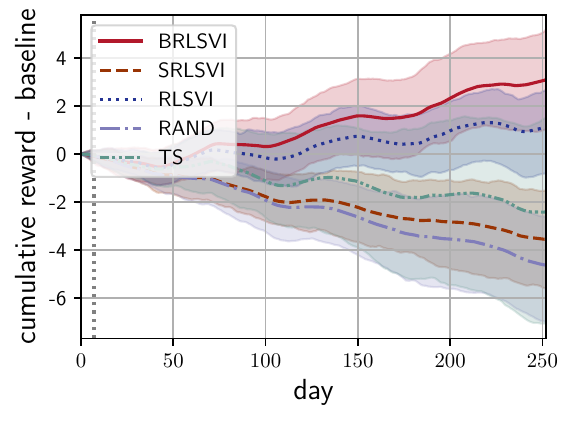}
        \caption{Combine (\subref{subfig:v2}) and (\subref{subfig:v3}).}
        \label{subfig:v4}
    \end{subfigure}
    \hfill
    \begin{subfigure}{0.23\textwidth}
        \centering
        \includegraphics[width=\textwidth]{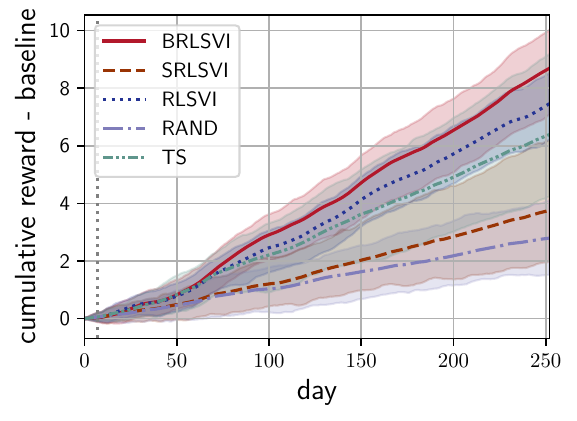}
        \caption{$R_{d-1} \to E_d$ exists.}
        \label{subfig:v5}
    \end{subfigure}
    \hfill
    \begin{subfigure}{0.23\textwidth}
        \centering
        \includegraphics[width=\textwidth]{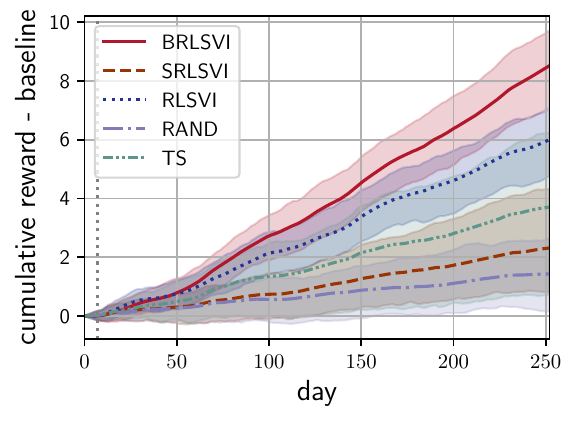}
        \caption{Misspecified mediator.}
        \label{subfig:v6}
    \end{subfigure}
    \hfill
    \begin{subfigure}{0.23\textwidth}
        \centering
        \includegraphics[width=\textwidth]{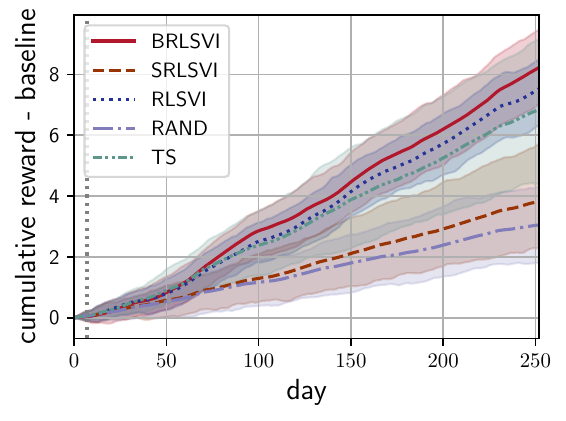}
        \caption{Contexts not exogenous.}
        \label{subfig:v7}
    \end{subfigure}
    \hfill
    \begin{subfigure}{0.23\textwidth}
        \centering
        \caption*{}
        \label{subfig:v8}
    \end{subfigure}
    \caption{
    The average cumulative rewards of BRLSVI, SRLSVI, RLSVI, and RAND, each subtracting the average cumulative rewards of the zero policy.
    }
    \label{fig:simulation.variants}
\end{figure}

The proposed algorithm BRLSVI will be compared against the baseline methods.
Without the periodic MDP framework, one approach is to treat the $K$ actions within a bag as a $K$-dimensional action and apply algorithms designed for a stationary MDP at the bag level.
The bag action $[A_{d, 1:K}]$ has $2^K$ different values when $\cA_k = \{0, 1\}$.
Alternatively, each bag can be treated as a separate episode, allowing the use of episodic RL algorithms.
Therefore, the four baseline methods are:
(1) Stationary RLSVI (SRLSVI) \citep[Algorithm 7]{osband2016generalization}, which treats each bag as a decision point in the standard MDP and chooses a vector of $K$ actions at the beginning of the bag. 
(2) RLSVI \citep[Algorithm 1]{osband2016generalization} with a finite horizon $K$.
(3) Random policy (RAND), which takes action 1 with a 0.5 probability.
(4) Thompson sampling (TS), which assumes a stationary bandit setting and maximizes the proximal outcome $M_{d, k}$ at each time.
The details of these methods are presented in Appendix~\ref{sec:baseline}.

Each replication contains 42 different testbed environments, constructed from the 42 users in HeartSteps V2.
The experiment is repeated 100 times for each method.
We report the cumulative reward $\sum_{d = 1}^D R_d$ averaged across the 42 users for each method, subtracting the average cumulative reward of the zero policy, where the action is always zero.
Figure~\ref{fig:simulation.variants} shows the average and 95\% confidence intervals across the 100 replications for the average cumulative reward.

The proposed method outperforms the baseline methods in all testbed variants except for Variant (\subref{subfig:v2}), particularly when the negative effect of actions is enhanced.
SRLSVI cannot account for real-time information $M_{d, 1:(k - 1)}, A_{d, 1:(k - 1)}, C_{d, k}$ at each time $(d, k)$.
Finite-horizon RLSVI ignores dependencies across bags, especially the delayed effect of the last action in a bag.
Although the upper confidence bound of TS is comparable to that of BRLSVI in Variant (\subref{subfig:v1}), TS exhibits a lower average cumulative reward and a higher variance.
In contrast, BRLSVI considers the long-term effects of actions while leveraging the real-time information at each decision time.
The performance of RLSVI, TS, and BRLSVI is comparable in Variant (\subref{subfig:v2}) since the positive effect is much stronger than the negative effect, and the optimal policy is close to always treating. 
RLSVI and TS can still identify the optimal policy even when the delayed effects are ignored.
As shown in Variants (\subref{subfig:v5})-(\subref{subfig:v7}), the proposed BRLSVI is robust against a misspecified DAG, since the intervention policy is constructed using a model-free algorithm and the state vector $\widetilde{S}^{\prime}_{d, k}$ remains the same for these variants (see Appendix~\ref{sec:misspecified.assumption} for details).
In Appendix~\ref{sec:additional.simulation}, we demonstrate that $S^{\prime}_{d, k}$ yields a higher cumulative reward than $S^{\prime \prime}_{d, k}$ on all testbed variants.

\section{DISCUSSION} \label{sec.discussion}

In this work, we design an RL algorithm for bagged decision times based on a causal DAG.
The causal DAG encodes domain knowledge to construct states and speed up learning, but we also demonstrate that the proposed algorithm remains robust to misspecified assumptions in the DAG.
Although we focus on the case where a single reward is observed per bag, the method can be readily extended to problems with multiple rewards and non-Markovian, non-stationary transition dynamics within a bag.

As per our agreements with the health system from which the patient data was collected, we are unable to release the original dataset. 
However, we have published the testbed with our code here: \href{https://github.com/StatisticalReinforcementLearningLab/Bagged-Decision-Times}{github.com/StatisticalReinforcementLearningLab/\\Bagged-Decision-Times}. 
The testbed is derived from the original dataset with small perturbation to maintain confidentiality.

Definition~\ref{dfn:bagged} assumes that the length of each bag is fixed at $K$.
In our motivating problem described in Section~\ref{sec:experiment}, this assumption naturally holds, as one day can be viewed as a bag with $K = 5$ decision times.
Nonetheless, the framework can be extended to problems with a Markovian bag length using the trick of absorbing states, provided that the bag length is bounded by some constant $\bar{K} \in \mathbb{N}^+$.
To define absorbing states, we need to assume that the same set of variables is defined across all decision times.
If variable definitions differ across time, extra variables can be manually defined for padding.
The reward $R_{d, k}$ takes values from $\mathbb{R} \cup [\perp]$, where $\perp$ denotes a missing value.
If $R_{d, k} = \perp$, the process continues as if it were in the middle of a bag.
If $R_{d, k} \in \mathbb{R}$ is observed, it becomes an absorbing state, with $R_{d, k} = R_{d, k+1} = \dots = R_{d, \bar{K}}$.
Similarly, the context and mediator also become absorbing states, and any actions $A_{d, (k+1):\bar{K}}$ have no effect on the reward. 
The objective remains to optimize the final reward $R_{d, \bar{K}}$ at the end of the bag, while intermediate rewards $R_{d, 1:(\bar{K} - 1)}$ may only be used as state variables in the RL algorithm.
Under this formulation, all bags are defined with a length of $\bar{K}$, while the actual bag length $K_d$ is a random variable and is Markovian across bags.

In the proposed algorithm, actions are selected by maximizing the Q-function with respect to a parameter sampled from the posterior distribution.
In practice, obtaining the intervention assignment probability is crucial for post-study analysis \citep{zhang2022statistical}.
\citet{xu2024adaptation} proposed a method to estimate the assignment probability for finite-horizon RLSVI.
How to estimate this probability for BRLSVI remains an open question.

\subsubsection*{Acknowledgments}
This research was funded by NIH grants P50DA054039, P41EB028242, R01HL125440-06A1, UH3DE028723, and P30AG073107-03 GY3 Pilots.
Susan Murphy holds concurrent appointments at Harvard University and as an Amazon Scholar. This paper describes work performed at Harvard University and is not associated with Amazon.

\bibliography{bibfile}

\section*{Checklist}



 \begin{enumerate}

 \item For all models and algorithms presented, check if you include:
 \begin{enumerate}
   \item A clear description of the mathematical setting, assumptions, algorithm, and/or model. Yes, see Sections~\ref{sec:settings} and~\ref{sec:methodology}.
   \item An analysis of the properties and complexity (time, space, sample size) of any algorithm. Yes, see Appendix~\ref{sec:algorithm.detail}.
   \item (Optional) Anonymized source code, with specification of all dependencies, including external libraries. Yes, see Section~\ref{sec.discussion} and Appendix~\ref{sec:algorithm.detail}.
 \end{enumerate}

 \item For any theoretical claim, check if you include:
 \begin{enumerate}
   \item Statements of the full set of assumptions of all theoretical results. Yes, see Section~\ref{sec:methodology}.
   \item Complete proofs of all theoretical results. Yes, see Appendix~\ref{sec:additional.theory}.
   \item Clear explanations of any assumptions. Yes, see Section~\ref{sec:methodology}.
 \end{enumerate}

 \item For all figures and tables that present empirical results, check if you include:
 \begin{enumerate}
   \item The code, data, and instructions needed to reproduce the main experimental results (either in the supplemental material or as a URL). Yes, see Section~\ref{sec.discussion}.
   \item All the training details (e.g., data splits, hyperparameters, how they were chosen). Yes, see Appendix~\ref{sec:algorithm.detail}.
         \item A clear definition of the specific measure or statistics and error bars (e.g., with respect to the random seed after running experiments multiple times). Yes, see Section~\ref{sec:simulation.results}.
         \item A description of the computing infrastructure used. (e.g., type of GPUs, internal cluster, or cloud provider). Yes, see Appendix~\ref{sec:algorithm.detail}.
 \end{enumerate}

 \item If you are using existing assets (e.g., code, data, models) or curating/releasing new assets, check if you include:
 \begin{enumerate}
   \item Citations of the creator If your work uses existing assets. Yes, see Section~\ref{sec:experiment}.
   \item The license information of the assets, if applicable. Not Applicable.
   \item New assets either in the supplemental material or as a URL, if applicable. Not Applicable.
   \item Information about consent from data providers/curators. Not Applicable.
   \item Discussion of sensible content if applicable, e.g., personally identifiable information or offensive content. Not Applicable.
 \end{enumerate}

 \item If you used crowdsourcing or conducted research with human subjects, check if you include:
 \begin{enumerate}
   \item The full text of instructions given to participants and screenshots. Not Applicable.
   \item Descriptions of potential participant risks, with links to Institutional Review Board (IRB) approvals if applicable. Not Applicable.
   \item The estimated hourly wage paid to participants and the total amount spent on participant compensation. Not Applicable.
 \end{enumerate}

 \end{enumerate}

\newpage
\onecolumn
\aistatstitle{Harnessing Causality in Reinforcement Learning With Bagged Decision Times: 
Supplementary Materials}

\appendix

\tableofcontents

\section{TABLE OF NOTATIONS}

In this section, we summarize the definition of notations.

\begin{table}[!htbp]
    \caption{Table of notations.}
    \label{tab:notations}
    \centering
    \begin{tabular}{cp{15cm}}
        \toprule
        Notation    & Definiton \\
        \midrule
        $d$         & Index of a bag \\
        $k$         & Index of a decision time in a bag \\
        $\cI$       & Index set $\cI := \{(d, k): d \ge 1, k \in \{1:K\}\}$ \\
        $V_{d, k}$  & Represent any variable at time $(d, k)$ \\
        $A_{d, k}$  & The action at time $(d, k)$ \\
        $C_{d, k}$  & The context at time $(d, k)$ \\
        $M_{d, k}$  & A mediator of $R_d$, usually representing the proximal positive effect of actions \\
        $N_{d, k}$  & A mediator of $E_d$, usually representing the proximal negative effect of actions \\
        $E_d$       & A mediator of $R_d$, usually representing the delayed negative effect of actions \\
        $R_d$       & The reward at the end of bag $d$ \\
        $\bB_d$     & The vector of all variables in bag $d$ \\
        $\widetilde{\bB}_d$ & A subset of bag $d$ s.t. $\widetilde{\bB}_d$ d-separates $\bB_{d + 1}$ and $\bB_{0:d}$, but any proper subset of $\widetilde{\bB}_d$ does not d-separate $\bB_{d + 1}$ and $\bB_{0:d}$. \\
        $H_{d, k}$  & The observed data since the previous action ($A_{d, k - 1}$ if $k \ge 2$ or $A_{d - 1, K}$ if $k = 1$) and before the current action $A_{d, k}$ \\
        $\cH_{d, k}$& All the observed history before $A_{d, k}$, i.e. $\cH_{d, k} := \{ H_{1, 1}, \dots, H_{d, k} \}$ \\
        $\nu$       & The distribution of the first d-separator $\widetilde{\bB}_0$ \\
        $r_k$       & The reward function at time $k$ \\
        $\gamma_k$  & The discount factor for the rewards after time $k$ \\
        $\bar{\gamma}$  & The discount factor for the reward of each bag \\
        $\bbP_k$    & The transition function at time $k$ \\
        $\cA_K$     & The action space at time $k$ in a bag \\
        $\cS_k$     & The state space at time $k$ in a bag \\
        $\Delta (\cB)$  & The space of probability distributions over a set $\cB$ \\
        $\cM$       & A K-periodic MDP \\
        $\bH$       & A stochastic process representing the observed history \\
        $\bR$       & A stochastic process representing the rewards \\
        $\bS$       & A stochastic process representing a D-BaSS \\
        $\bZ$       & A stochastic process representing a minimal D-BaSS \\
        $\bA$       & A stochastic process representing the actions \\
        $\Pi$       & The class of all possibly history-dependent, stochastic policies \\
        $\widetilde{\Pi}$   & the class of all Markov, stationary, and stochastic policies \\
        $\cQ_k^{\bpi}$  & The Q-function of a policy $\bpi$ at time $k$ \\
        $\cV_k^{\bpi}$  & The value function of a policy $\bpi$ at time $k$ \\
        $\cV_k^{\cS *}$ & The optimal value function of based on the sequence of states $\cS$ at time $k$ \\
        \bottomrule
    \end{tabular}
\end{table}

\section{ADDITIONAL THEORETICAL RESULTS IN SECTION~\ref{sec:methodology}} \label{sec:additional.theory}

This section contains additional theoretical results and proofs for the results in Section~\ref{sec:methodology}.

Let $\Pi$ be the class of all history-dependent, stochastic (HS) policies $\bpi$, where $\bpi = \{ \pi_{d, 1:K} \}_{d = 1}^{\infty}$ and $\pi_{d, k}: (\cS_1 \times \dots \times \cS_K)^{d-1} \times \cS_1 \times \dots \times \cS_k \mapsto \Delta (\cA_k)$.
In this section, we will define the value functions for a general policy $\bpi \in \Pi$, and show that it is enough to restrict on the class of Markov, stationary and deterministic (MSD) policies $\widetilde{\Pi}$.
We will then prove that the definitions of value functions and Q-functions for $\widetilde{\bpi} \in \widetilde{\Pi}$ here is equivalent to the definitions in Section~\ref{sec:periodic.mdp}.

To show that it is enough to restrict on the MSD policy class for a stationary MDP, 
\citet{bertsekas2007dynamic} proved that for any HS policy, there exists an MSD policy s.t. they generate the same state-action visitation probability.
\citet{agarwal2022reinforcement} instead proved that there exists an optimal policy that is MSD.
However, both the proofs rely on the stationary structure s.t. there is only one value function.
Therefore, we need to prove that the results in a stationary MDP holds analogously for a $K$-bagged MDP when there are $K$ time-dependent value functions.

In the following analysis, we will follow the proof idea in \citet{agarwal2022reinforcement}.
Another possible proof method is to represent the $K$-bagged MDP as a stationary MDP by enhancing the state space to include time $k$ as an additional variable.
The transition function, and reward function will be redefined as a combination of the original time-dependent function and the time index.
However, we will need extra indicators to represent state availability and action availability at different times.
Besides, all rewards need to be rescaled to incorporate the time-dependent discounting factor.
Therefore, we will still go through the general proof to show the existence of an MSD policy and prove the Bellman optimality equations.
In the proof, we will also be able to demonstrate the connection between the $K$ value functions.

The value function of a policy $\bpi \in \Pi$ at time $k$ in bag 1 is defined as 
\begin{equation*}
    \begin{split}
        \cV_k^{\bpi} (s_{1, 1:k}) := \bbE^{\bpi} \brce{ 
            \sum_{l=k}^K \eta_{1, k}^{1, l} R_{1, l} + \sum_{t=2}^{\infty} \sum_{l=1}^K \eta_{1, k}^{t, l} R_{t, l} \middle| S_{1, 1:k} = s_{1, 1:k} }
    \end{split}
\end{equation*}
for all $k \in \{1:K\}$ and all $s_{1, 1:k} \in \cS_1 \times \dots \times \cS_k$.
The optimal value function for bag $1$ is $\cV_k^{*} (s_{1, 1:k}) := \sup_{\bpi \in \Pi} \cV_k^{\bpi} (s_{1, 1:k})$.
A policy $\bpi \in \Pi$ is said to be optimal if $\cV_1^{\bpi} (s_{1, 1}) = \cV_1^{*} (s_{1, 1})$ for all $s_{1, 1} \in \cS_1$.
Similarly, we define the action-value function or Q-function for bag $1$ as
\begin{equation*}
    \begin{split}
        \cQ_k^{\bpi} & (s_{1, 1:k}, a_{1, k}) := 
        \bbE^{\bpi} \brce{ 
            \sum_{l=k}^K \eta_{1, k}^{1, l} R_{1, l} + \sum_{t=2}^{\infty} \sum_{l=1}^K \eta_{1, k}^{t, l} R_{t, l} \middle| S_{1, 1:k} = s_{1, 1:k}, A_{1, k} = a_{1, k} }.
    \end{split}
\end{equation*}
Notice that $\cQ_k^{\bpi}$ does not depend on $a_{1, 1:(k - 1)}$ due to the Markov property.
The dependence on $s_{1, 1:(k - 1)}$ is through the policy $\bpi$.
The optimal Q-function is 
$\cQ_k^{*} (s_{1, 1:k}, a_{1, k}) = \sup_{\bpi \in \Pi}\cQ_k^{\bpi} (s_{1, 1:k}, a_{1, k})$
for all $k \in \{1:K\}$, all $s_{1, 1:k} \in \cS_1 \times \dots \times \cS_k$ and $a_{1, k} \in \cA_k$.

\begin{lem} \label{lem:stationary.policy}
    For a $K$-bagged MDP, there exists a Markov and stationary policy $\bar{\bpi} = \{\bar{\pi}_{1:K}\}$ for each bag, where $\bar{\pi}_k: \cS_k \mapsto \cA_k$ is a deterministic policy at decision time $k$, s.t. $\bar{\bpi}$ is optimal, i.e. $\cV_1^{\bar{\bpi}} (s_{1, 1}) = \cV_1^{*} (s_{1, 1})$.
\end{lem}

In the proof of Lemma~\ref{lem:stationary.policy} in Appendix~\ref{sec:proof.stationary.policy}, we first show that there exists an optimal policy which is Markovian across bags but not necessarily within a bag.
Then we demonstrate that there exists another policy which is Markovian both across and within a bag and yields the same actions as the previous policy.
The construction of the optimal policy here does not depend directly on the optimal value functions, although it can be proved that $\bar{\bpi}$ is actually equivalent to the optimal policy defined in Theorem~\ref{thm:optimal.policy}.

Lemma~\ref{lem:stationary.policy} shows that in order to find an optimal policy, it is sufficient to focus on the class $\widetilde{\Pi}$ of all Markov and stationary policies.
Lemma~\ref{lem:markov.value.function} below presents some properties of $\widetilde{\bpi} \in \widetilde{\Pi}$.

\begin{lem} \label{lem:markov.value.function}
    For a policy $\widetilde{\bpi} \in \widetilde{\Pi}$, 
    \begin{equation*}
        \cV_k^{\widetilde{\bpi}} (s_{1, 1:k}) = \bbE^{\widetilde{\bpi}} \brce{ 
        \sum_{l=k}^K \eta_{1, k}^{1, l} R_{1, l} + \sum_{t=2}^{\infty} \sum_{l=1}^K \eta_{1, k}^{t, l} R_{t, l} \middle| S_{1, k} = s_{1, k} }
    \end{equation*}
    for all $s_{1, 1:k} \in \cS_1 \times \dots, \times \cS_k$.
    Therefore, the value function of $\widetilde{\bpi}$ can be denoted as $\cV_k^{\widetilde{\bpi}} (s_{1, k})$.
    In addition, the optimal value function satisfies
    \begin{equation*}
        \cV_k^{*} (s_{1, 1:k}) = \sup_{\bpi \in \Pi} \cV_k^{\bpi} (s_{1, 1:k}) = \sup_{\widetilde{\bpi} \in \widetilde{\Pi}} \cV_k^{\widetilde{\bpi}} (s_{1, k}) = \cV_k^{*} (s_{1, k})
    \end{equation*}
    for all $s_{1, 1:k}$.
    
    Furthermore, the value function of $\widetilde{\bpi} \in \widetilde{\Pi}$ at time $k$ in bag $d \ge 1$,
    defined as
    \begin{equation*}
        \bbE^{\widetilde{\bpi}} \brce{ 
            \sum_{l=k}^K \eta_{d, k}^{d, l} R_{d, l} + \sum_{t=d + 1}^{\infty} \sum_{l=1}^K \eta_{d, k}^{t, l} R_{t, l} \middle| S_{1, 1} = s_{1, 1}, \dots, S_{d, k} = s_{d, k} },
    \end{equation*}
    is equal to $\cV_k^{\widetilde{\bpi}} (s_{d, k})$
    for all $s_{1, 1}, \dots, s_{d, k} \in (\cS_1 \times \dots \times \cS_K)^{d - 1} \times \cS_1 \times \dots \times \cS_k$.
    Besides, the optimal value function, i.e. the supremum of value function over all $\bpi \in \Pi$, is equal to $\cV_k^{*} (s_{d, k})$.

    Similarly, the Q-function of $\widetilde{\bpi} \in \widetilde{\Pi}$ can be denoted as $\cQ_k^{\widetilde{\bpi}} (s_{1, k}, a_{1, k})$, and the optimal Q-function satisfies
    \begin{equation*}
    \begin{split}
        \cQ_k^{*} (s_{1, 1:k}, a_{1, k}) = \sup\nolimits_{\bpi \in \Pi} \cQ_k^{\bpi} (s_{1, 1:k}, a_{1, k})
        = \sup\nolimits_{\widetilde{\bpi} \in \widetilde{\Pi}} \cQ_k^{\widetilde{\bpi}} (s_{1, k}, a_{1, k}) = \cQ_k^{*} (s_{1, k}, a_{1, k}).
    \end{split}
    \end{equation*}
    The Q-function and optimal Q-function at time $k$ in bag $d \ge 1$ is equal to $\cQ_k^{\widetilde{\bpi}} (s_{d, k}, a_{d, k})$ and $\cQ_k^{*} (s_{d, k}, a_{d, k})$.
\end{lem}

Lemma~\ref{lem:markov.value.function} proves that the value functions and Q-functions defined in this section is equivalent to the definitions in Section~\ref{sec:periodic.mdp}.

\subsection{Proof of Lemma~\ref{lem:stationary.policy}} \label{sec:proof.stationary.policy}

The following Lemma~\ref{lem:upper.bound} provides an upper bound for $\cV_k^* (s_{1, k})$, and is useful for proving Lemma~\ref{lem:stationary.policy}.
\begin{lem} \label{lem:upper.bound}
    For the optimal value functions, we have
    \begin{equation} \label{equ:optimal.V.bound.K}
        \cV_K^* (s_{1, 1:K}) \le \max_{a_{1, K} \in \cA_K} \bbE \brce{R_{1, K} + \gamma_K \cV_1^* (S_{2, 1}) \middle| S_{1, 1:K} = s_{1, 1:K}, A_{1, K} = a_{1, K} },
    \end{equation}
    for all $s_{1, 1:K} \in \cS_1 \times \dots \times \cS_K$
    at decision time $k = K$,
    and
    \begin{equation} \label{equ:optimal.V.bound.k}
        \cV_k^* (s_{1, 1:k}) \le \max_{a_{1, k} \in \cA_k} \bbE \brce{R_{1, k} + \gamma_k \cV_{k + 1}^* (S_{1, 1:(k + 1)}) \middle| S_{1, 1:k} = s_{1, 1:k}, A_{1, k} = a_{1, k} }.
    \end{equation}
    for all $s_{1, 1:k} \in \cS_1 \times \dots \times \cS_k$
    at decision times $k \in \{1:(K - 1)\}$.
    Similarly, for the optimal Q-functions, we have
    \begin{equation} \label{equ:optimal.Q.bound.K}
        \cQ_K^* (s_{1, 1:K}, a_{1, K}) \le \bbE \brce{R_{1, K} + \gamma_K \cV_1^* (S_{2, 1}) \middle| S_{1, 1:K} = s_{1, 1:K}, A_{1, K} = a_{1, K} },
    \end{equation}
    for all $s_{1, 1:K} \in \cS_1 \times \dots \times \cS_K$ and $a_{1, K} \in \cA_K$
    at decision time $k = K$,
    and
    \begin{equation} \label{equ:optimal.Q.bound.k}
        \cQ_k^* (s_{1, 1:k}, a_{1, k}) \le \bbE \brce{R_{1, k} + \gamma_k \cV_{k + 1}^* (S_{1, 1:(k + 1)}) \middle| S_{1, 1:k} = s_{1, 1:k}, A_{1, k} = a_{1, k} }.
    \end{equation}
    for all $s_{1, 1:k} \in \cS_1 \times \dots \times \cS_k$ and $a_{1, k} \in \cA_k$
    at decision times $k \in \{1:(K - 1)\}$.
\end{lem}

\begin{proof}
    We write the distribution of actions within the conditions of expectations to illustrate how the policies depend on other variables.
    Under consistency assumption in causal inference, the reward given the actions is equal to the potential outcomes of such actions.
    For any policy $\bpi \in \Pi$, denote the conditional cumulative reward starting from bag 2 given $s_{1, 1:K}, a_{1, K}$ as
    \begin{equation*}
        \begin{split}
            \widetilde{\cV}_1^{\bpi} (s_{2, 1} | s_{1, 1:K}, a_{1, K})
            = & \bbE^{\bpi} \brce{\sum_{t=2}^{\infty} \sum_{l=1}^K \eta_{1, K}^{t, l} R_{t, l} \middle| S_{1, 1:K} = s_{1, 1:K}, S_{2, 1} = s_{2, 1}, A_{1, K} = a_{1, K} }.
        \end{split}
    \end{equation*}
    Let $\widetilde{\pi}_{t, l, s_{0, 1:K}} (S_{1, 1}, \dots, S_{t, l})$ denote the policy $\pi_{t + 1, l} (s_{0, 1:K}, S_{1, 1}, \dots, S_{t, l})$, given that $s_{0, 1:K}$ are constants.
    By the Markov property,
    \begin{equation*} \label{equ:v.tilde.reexpress}
        \begin{split}
            & \widetilde{\cV}_1^{\bpi} (s_{2, 1} | s_{1, 1:K}, a_{1, K}) \\
            = & \bbE \brce{
                \sum_{t=2}^{\infty} \sum_{l=1}^K \eta_{1, K}^{t, l} R_{t, l} \middle| 
                S_{2, 1} = s_{2, 1}, A_{t, l} \sim \pi_{t, l} (s_{1, 1:K}, S_{2, 1}, \dots, S_{t, l}) \text{ for } t \ge 2
            }\\
            = & \gamma_K \bbE \brce{
                \sum_{t=1}^{\infty} \sum_{l=1}^K \eta_{1, 1}^{t, l} R_{t, l} \middle| 
                S_{1, 1} = s_{2, 1}, A_{t, l} \sim \pi_{t + 1, l} (s_{1, 1:K}, S_{1, 1}, \dots, S_{t, l}) \text{ for } t \ge 1
            } \\
            = & \gamma_K \bbE \brce{
                \sum_{t=1}^{\infty} \sum_{l=1}^K \eta_{1, 1}^{t, l} R_{t, l} \middle| 
                S_{1, 1} = s_{2, 1}, A_{t, l} \sim \widetilde{\pi}_{t, l, s_{0, 1:K}} (S_{1, 1}, \dots, S_{t, l}) \text{ for } t \ge 1
            }.
        \end{split}
    \end{equation*}
    The second equality follows from the change of index.
    Since $\widetilde{\pi}_{t, l, s_{0, 1:K}} (S_{1, 1}, \dots, S_{t, l}) \in \Pi$,
    we have 
    \begin{equation} \label{equ:bound.U21}
        \widetilde{\cV}_1^{\bpi} (s_{2, 1} | s_{1, 1:K}, a_{1, K})
        \le \gamma_K \cV_1^* (s_{2, 1})
    \end{equation}
    for all $\bpi \in \Pi$, $s_{2, 1} \in \cS_1$, $s_{1, 1:K} \in \cS_1 \times \dots \times \cS_K$, and $a_{1, K} \in \cA_K$.
    At the last decision time $k = K$ in the first bag $d = 1$,
    \begin{equation*}
        \begin{split}
            \cV_K^{\bpi} (s_{1, 1:K})
            = & \bbE^{\bpi} \brce{R_{1, K} + \sum_{t=2}^{\infty} \sum_{l=1}^K \eta_{2, 1}^{t, l} R_{t, l} \middle| S_{1, 1:K} = s_{1, 1:K} } \\
            = & \sum_{a_{1, K} \in \cA_K} \bbP \brce{\pi_{1, K} (s_{1, 1:K}) = a_{1, K}} \bbE^{\bpi} \brce{R_{1, K} + \sum_{t=2}^{\infty} \sum_{l=1}^K \eta_{2, 1}^{t, l} R_{t, l} \middle| S_{1, 1:K} = s_{1, 1:K}, A_{1, K} = a_{1, K} } \\
            = & \sum_{a_{1, K} \in \cA_K} \bbP \brce{\pi_{1, K} (s_{1, 1:K}) = a_{1, K}} \bbE \brce{R_{1, K} + \widetilde{\cV}_1^{\bpi} (S_{2, 1} | s_{1, 1:K}, a_{1, K}) \middle| S_{1, 1:K} = s_{1, 1:K}, A_{1, K} = a_{1, K}}.\\
        \end{split}
    \end{equation*}
    Notice that $R_{1, K}$ and $S_{2, 1}$ are independent of the policy $\bpi$ given $S_{1, 1:K}$ and $A_{1, K}$.
    The second equality follows from the tower property of conditional expectation, and the third equality follows from the definition of $\widetilde{\cV}_1^{\bpi} (s_{2, 1} | s_{1, 1:K}, a_{1, K})$.
    Following inequation~(\ref{equ:bound.U21}),
    \begin{equation} \label{equ:bound.VK.pi}
        \begin{split}
            \cV_K^{\bpi} (s_{1, 1:K})
            \le & \sum_{a_{1, K} \in \cA_K} \bbP \brce{\pi_{1, K} (s_{1, 1:K}) = a_{1, K}} \bbE \brce{R_{1, K} + \gamma_K \cV_1^* (S_{2, 1}) \middle| S_{1, 1:K} = s_{1, 1:K}, A_{1, K} = a_{1, K}}\\
            \le & \max_{a_{1, K} \in \cA_K} \bbE \brce{R_{1, K} + \gamma_K \cV_1^* (S_{2, 1}) \middle| S_{1, 1:K} = s_{1, 1:K}, A_{1, K} = a_{1, K}}.
        \end{split}
    \end{equation}
    Since the above inequation~(\ref{equ:bound.VK.pi}) holds for all $\bpi \in \Pi$, we have
    \begin{equation*}
        \cV_K^{*} (s_{1, 1:K}) = \sup_{\bpi \in \Pi} \cV_K^{\bpi} (s_{1, 1:K}) 
        \le \max_{a_{1, K} \in \cA_K} \bbE \brce{R_{1, K} + \gamma_K \cV_1^* (S_{2, 1}) \middle| S_{1, 1:K} = s_{1, 1:K}, A_{1, K} = a_{1, K}}.
    \end{equation*}
    
    Similarly, for $k = K - 1, \dots, 1$, define the cumulative reward starting from the $(k + 1)$th decision time in bag 1 as
    \begin{equation*}
        \begin{split}
            \widetilde{\cV}_1^{\bpi} (s_{1, k + 1} | s_{1, 1:k}, a_{1, k})
            = \bbE^{\bpi} \brce{
                \sum_{l=k + 1}^K \eta_{1, k}^{1, l} R_{1, l} +
                \sum_{t=2}^{\infty} \sum_{l=1}^K \eta_{1, k}^{t, l} R_{t, l} \middle| 
            S_{1, 1:k} = s_{1, 1:k}, A_{1, k} = a_{1, k}, S_{1, k + 1} = s_{1, k + 1} }.
        \end{split}
    \end{equation*}
    Then, according to the Markov property and the definition of $\cV_{k + 1}^* (s_{1, 1:(k + 1)})$,
    \begin{equation} \label{equ:bound.V1k}
        \begin{split}
            \widetilde{\cV}_1^{\bpi} (s_{1, k + 1} | s_{1, 1:k}, a_{1, k})
            = & \bbE \Bigg\{
            \sum_{l=k + 1}^K \eta_{1, k}^{1, l} R_{1, l} +
            \sum_{t=2}^{\infty} \sum_{l=1}^K \eta_{1, k}^{t, l} R_{t, l} \Bigg| 
            S_{1, k + 1} = s_{1, k + 1}, \\
            & \qquad \qquad \qquad A_{t, l} \sim \pi_{t, l} (s_{1, 1:k}, S_{1, k + 1}, \dots, S_{1, l}) \text{ for } t \ge 1, l \ge k + 1 \text{ or } t \ge 2
            \Bigg\} \\
            \le & \gamma_k \cV_{k + 1}^* (s_{1, 1:(k + 1)}).
        \end{split}
    \end{equation}
    The tower property of conditional expectation and inequation~(\ref{equ:bound.V1k}) yield
    \begin{equation} \label{equ:bound.Vk.pi}
        \begin{split}
            \cV_k^{\bpi} (s_{1, 1:k})
            = & \bbE^{\bpi} \brce{\sum_{l=k}^K \eta_{1, k}^{1, l} R_{1, l} + \sum_{t=2}^{\infty} \sum_{l=1}^K \eta_{1, k}^{t, l} R_{t, l} \middle| S_{1, 1:k} = s_{1, 1:k} } \\
            = & \sum_{a_{1, k} \in \cA_k} \bbP \brce{\pi_{1, k} (s_{1, 1:k}) = a_{1, k}} \bbE \brce{R_{1, k} + \widetilde{\cV}_{k + 1}^{\bpi} (S_{1, k + 1} | s_{1, 1:k}, a_{1, k}) \middle| S_{1, 1:k} = s_{1, 1:k}, A_{1, k} = a_{1, k} } \\
            \le & \sum_{a_{1, k} \in \cA_k} \bbP \brce{\pi_{1, k} (s_{1, 1:k}) = a_{1, k}} \bbE \brce{R_{1, k} + \gamma_{k} \cV_{k + 1}^* (s_{1, 1:k}, S_{1, k + 1}) \middle| S_{1, 1:k} = s_{1, 1:k}, A_{1, k} = a_{1, k}}\\
            \le & \max_{a_{1, k} \in \cA_k} \bbE \brce{R_{1, k} + \gamma_{k} \cV_{k + 1}^* (S_{1, 1:(k + 1)}) \middle| S_{1, 1:k} = s_{1, 1:k}, A_{1, k} = a_{1, k}}.
        \end{split}
    \end{equation}
    Since the above inequation~(\ref{equ:bound.Vk.pi}) holds for all $\bpi \in \Pi$, we have
    \begin{equation*}
        \cV_k^{*} (s_{1, 1:k}) = \sup_{\bpi \in \Pi} \cV_k^{\bpi} (s_{1, 1:k}) 
        \le \max_{a_{1, k} \in \cA_k} \bbE \brce{R_{1, k} + \gamma_{k} \cV_{k + 1}^* (S_{1, 1:(k + 1)}) \middle| S_{1, 1:k} = s_{1, 1:k}, A_{1, k} = a_{1, k}}.
    \end{equation*}

    For Q-functions at the last decision time $k = K$ in the first bag $d = 1$,
    \begin{equation*}
        \begin{split}
            \cQ_K^{\bpi} (s_{1, 1:K}, a_{1, K})
            = & \bbE^{\bpi} \brce{R_{1, K} + \sum_{t=2}^{\infty} \sum_{l=1}^K \eta_{1, k}^{t, l} R_{t, l} \middle| S_{1, 1:K} = s_{1, 1:K}, A_{1, K} = a_{1, K} } \\
            = & \bbE \brce{R_{1, K} + \widetilde{\cV}_1^{\bpi} (S_{2, 1} | s_{1, 1:K}, a_{1, K}) \big| S_{1, 1:K} = s_{1, 1:K}, A_{1, K} = a_{1, K}} \\
            \le & \bbE \brce{R_{1, K} + \gamma_K \cV_1^* (S_{2, 1}) \middle| S_{1, 1:K} = s_{1, 1:K}, A_{1, K} = a_{1, K}}
        \end{split}
    \end{equation*}
    for all $\bpi \in \Pi$.
    Therefore,
    \begin{equation*}
        \cQ_K^{*} (s_{1, 1:K}, a_{1, K}) = \sup_{\bpi \in \Pi} \cQ_K^{\bpi} (s_{1, 1:K}, a_{1, K}) 
        \le \bbE \brce{R_{1, K} + \gamma_K  \cV_1^* (S_{2, 1}) \middle| S_{1, 1:K} = s_{1, 1:K}, A_{1, K} = a_{1, K}}.
    \end{equation*}
    Similarly, for $k = K - 1, \dots, 1$,
    \begin{equation*}
        \begin{split}
            \cQ_k^{\bpi} (s_{1, 1:k}, a_{1, k})
            = & \bbE^{\bpi} \brce{\sum_{l=k}^K \eta_{1, k}^{1, l} R_{1, l} + \sum_{t=2}^{\infty} \sum_{l=1}^K \eta_{1, k}^{t, l} R_{t, l} \middle| S_{1, 1:k} = s_{1, 1:k}, A_{1, k} = a_{1, k} } \\
            = & \bbE \brce{R_{1, k} + \widetilde{\cV}_{k + 1}^{\bpi} (S_{1, k + 1} | s_{1, 1:k}, a_{1, k}) \middle| S_{1, 1:k} = s_{1, 1:k}, A_{1, k} = a_{1, k} } \\
            \le & \bbE \brce{R_{1, k} + \gamma_k \cV_{k + 1}^* (s_{1, 1:k}, S_{1, k + 1}) \middle| S_{1, 1:k} = s_{1, 1:k}, A_{1, k} = a_{1, k}}
        \end{split}
    \end{equation*}
    for all $\bpi \in \Pi$.
    Therefore,
    \begin{equation*}
        \cQ_k^{*} (s_{1, 1:k}, a_{1, k}) = \sup_{\bpi \in \Pi} \cQ_k^{\bpi} (s_{1, 1:k}, a_{1, k}) 
        \le \bbE \brce{R_{1, k} + \gamma_k \cV_{k + 1}^* (s_{1, 1:k}, S_{1, k + 1}) \middle| S_{1, 1:k} = s_{1, 1:k}, A_{1, k} = a_{1, k}}.
    \end{equation*}
\end{proof}

Next, we will show that there exists a Markov policy across bags, but not necessarily within bags, that is optimal with respect to the value function $\cV_1^*$.

\begin{lem} \label{lem:optimal.nonmarkov.policy}
    Define a policy $\bpi^{\prime} = \{ \pi_{1:K}^{\prime} \}$, where 
    \begin{equation*}
        \begin{split}
            \pi_K^{\prime} (s_{1, 1:K}) = \argmax_{a_{1, K} \in \cA_K}
            \bbE \{ R_{1, K} + \gamma_K \cV_1^{*} (S_{2, 1}) |
            S_{1, 1:K} = s_{1, 1:K}, A_{1, K} = a_{1, K} \}
        \end{split}
    \end{equation*}
    at time $k = K$,
    and
    \begin{equation*}
        \begin{split}
            \pi_k^{\prime} (s_{1, 1:k}) = \argmax_{a_{1, k} \in \cA_k}
            \bbE \{ R_{1, k} + \gamma_k \cV_{k + 1}^{*} (S_{1, 1:(k + 1)}) |
            S_{1, 1:k} = s_{1, 1:k}, A_{1, k} = a_{1, k} \}
        \end{split}
    \end{equation*}
    at time $k = K - 1, \dots, 1$.
    Then the policy $\bpi^{\prime}$ is optimal, i.e., $\cV_1^{\bpi^{\prime}} (s_{1, 1}) = \cV_1^{*} (s_{1, 1})$ for all $s_{1, 1} \in \cS_1$.
    In addition, 
    \begin{gather}
        \cV_K^{*} (s_{1, 1:K}) 
        = \cV_K^{\bpi^{\prime}} (s_{1, 1:K}) 
        = \max_{a_{1, K} \in \cA_K}
        \bbE \{ R_{1, K} + \gamma_K \cV_1^{*} (S_{2, 1}) |
        S_{1, 1:K} = s_{1, 1:K}, A_{1, K} = a_{1, K} \}, 
        \label{equ:V.optimal.pi.prime.K} \\
        \cV_k^{*} (s_{1, 1:k}) 
        = \cV_k^{\bpi^{\prime}} (s_{1, 1:k}) 
        = \max_{a_{1, k} \in \cA_k}
        \bbE \{ R_{1, k} + \gamma_k \cV_{k + 1}^{*} (S_{1, 1:(k + 1)}) |
        S_{1, 1:k} = s_{1, 1:k}, A_{1, k} = a_{1, k} \}
        \label{equ:V.optimal.pi.prime.k}
    \end{gather}
    for all $k \in \{1:K\}$ and all $s_{1, 1:k} \in \cS_1 \times \dots \times \cS_k$.
    Similarly, for the same $\bpi^{\prime}$ we have
    \begin{gather}
        \cQ_K^{*} (s_{1, 1:K}, a_{1, K}) 
        = \cQ_K^{\bpi^{\prime}} (s_{1, 1:K}, a_{1, K}) 
        \label{equ:Q.optimal.pi.prime.K} \\
        \cQ_k^{*} (s_{1, 1:k}, a_{1, k}) 
        = \cQ_k^{\bpi^{\prime}} (s_{1, 1:k}, a_{1, k}) 
        \label{equ:Q.optimal.pi.prime.k}
    \end{gather}
    for all $k \in \{1:K\}$, all $s_{1, 1:k} \in \cS_1 \times \dots \times \cS_k$, and all $a_{1, k} \in \cA_k$.
\end{lem}

\begin{proof}
    According to inequation~(\ref{equ:optimal.V.bound.K}) in Lemma~\ref{lem:upper.bound} and the definition of $\pi_K^{\prime} (s_{1, K})$,
    \begin{equation} \label{equ:final.bound.VK}
        \begin{split}
            \cV_K^{*} (s_{1, 1:K}) 
            \le & \bbE \brce{R_{1, K} + \gamma_K \cV_1^* (S_{2, 1}) \middle| S_{1, 1:K} = s_{1, 1:K}, A_{1, K} \sim \pi_K^{\prime} (s_{1, 1:K})} \\
            = & \bbE^{\bpi^{\prime}} \brce{R_{1, K} + \gamma_K \cV_1^* (S_{2, 1}) \middle| S_{1, 1:K} = s_{1, 1:K}}.
        \end{split}
    \end{equation}
    Similarly, by inequation~(\ref{equ:optimal.V.bound.k}) in Lemma~\ref{lem:upper.bound} and the definition of $\pi_k^{\prime} (s_{1, k})$,
    \begin{equation} \label{equ:final.bound.Vk}
        \begin{split}
            \cV_k^{*} (s_{1, 1:k}) 
            \le & \bbE \brce{R_{1, k} + \gamma_k \cV_{k + 1}^* (S_{1, 1:(k + 1)}) \middle| S_{1, 1:k} = s_{1, 1:k}, A_{1, k} \sim \pi_k^{\prime} (s_{1, 1:k})} \\
            = & \bbE^{\bpi^{\prime}} \brce{R_{1, k} + \gamma_k \cV_{k + 1}^* (S_{1, 1:(k + 1)}) \middle| S_{1, 1:k} = s_{1, 1:k}}.
        \end{split}
    \end{equation}
    Combining inequalities~(\ref{equ:final.bound.VK}) and~(\ref{equ:final.bound.Vk}) and applying the same arguments recursively, we have
    \begin{equation*} \label{equ:optimal.pi.prime.1}
    \begin{split}
        \cV_1^{*} (s_{1, 1}) 
        \le & \bbE^{\bpi^{\prime}} \brce{ R_{1, 1} + \gamma_1 \cV_{2}^* (S_{1, 1:2}) \middle| S_{1, 1} = s_{1, 1}} \\
        \le & \bbE^{\bpi^{\prime}} \brce{ R_{1, 1} + \gamma_1 \bbE^{\bpi^{\prime}} \brce{R_{1, 2} + \gamma_2 \cV_{3}^* (S_{1, 1:3}) \middle| S_{1, 1:2}} \middle| S_{1, 1} = s_{1, 1}} \\
        =   & \bbE^{\bpi^{\prime}} \brce{ R_{1, 1} + \gamma_1 R_{1, 2} + \gamma_1 \gamma_2 \cV_{3}^* (S_{1, 1:3}) \middle| S_{1, 1} = s_{1, 1}} \\
        \le & \dots \\
        \le & \bbE^{\bpi^{\prime}} \brce{ R_{1, 1} + \gamma_1 R_{1, 2} + \gamma_1 \gamma_2 R_{1, 3} + \dots \middle| S_{1, 1} = s_{1, 1}} \\
        =  & \cV_1^{\bpi^{\prime}} (s_{1, 1}).
    \end{split}
    \end{equation*}
    Since $\bpi^{\prime} \in \Pi$, we have $\cV_1^{\bpi^{\prime}} (s_{1, 1}) \le \cV_1^{*} (s_{1, 1})$.
    Therefore, $\cV_1^{\bpi^{\prime}} (s_{1, 1}) = \cV_1^{*} (s_{1, 1})$ and the policy $\bpi^{\prime}$ is optimal.
    Furthermore, with the same arguments, we have that
    \begin{equation*}
    \begin{split}
        \cV_k^{*} (s_{1, 1:k}) 
        \le \cV_k^{\bpi^{\prime}} (s_{1, 1:k})
    \end{split}
    \end{equation*}
    for $k = 2, \dots, K$.
    Since $\cV_k^{\bpi^{\prime}} (s_{1, 1:k}) \le \cV_k^{*} (s_{1, 1:k})$, we also have $\cV_k^{\bpi^{\prime}} (s_{1, 1:k}) = \cV_k^{*} (s_{1, 1:k})$.
    Equations~(\ref{equ:V.optimal.pi.prime.K}) and~(\ref{equ:V.optimal.pi.prime.k}) follow from the definition of $\bpi^{\prime}$.

    For Q-functions, inequation~(\ref{equ:optimal.Q.bound.k}), equations~(\ref{equ:V.optimal.pi.prime.K}) and~(\ref{equ:V.optimal.pi.prime.k}), and the definition of $\cV_{k + 1}^{\bpi^{\prime}}$ yield
    \begin{equation*}
    \begin{split}
        \cQ_k^* (s_{1, 1:k}, a_{1, k}) 
        \le & \bbE \brce{R_{1, k} + \gamma_k \cV_{k + 1}^* (S_{1, 1:(k + 1)}) \middle| S_{1, 1:k} = s_{1, 1:k}, A_{1, k} = a_{1, k} } \\
        = & \bbE \brce{R_{1, k} + \gamma_k \cV_{k + 1}^{\bpi^{\prime}} (S_{1, 1:(k + 1)}) \middle| S_{1, 1:k} = s_{1, 1:k}, A_{1, k} = a_{1, k} } \\
        = & \bbE^{\bpi^{\prime}} \brce{\sum_{l=k}^K \eta_{1, k}^{1, l} R_{1, l} + \sum_{t=2}^{\infty} \sum_{l=1}^K \eta_{1, k}^{t, l} R_{t, l} \middle| S_{1, 1:k} = s_{1, 1:k}, A_{1, k} = a_{1, k} } 
        = \cQ_k^{\bpi^{\prime}} (s_{1, 1:k}, a_{1, k}).
    \end{split}
    \end{equation*}
    for $k \in \{1:K\} - 1$ and for all $s_{1, 1:k}$ and $a_{1, k}$. 
    Similarly, inequation~(\ref{equ:optimal.Q.bound.K}) and equation~(\ref{equ:V.optimal.pi.prime.k}) yield
    \begin{equation*}
    \begin{split}
        \cQ_K^* (s_{1, 1:K}, a_{1, K}) 
        \le & \bbE \brce{R_{1, K} + \gamma_K \cV_{1}^* (S_{2, 1}) \middle| S_{1, 1:K} = s_{1, 1:K}, A_{1, K} = a_{1, K} } \\
        = & \bbE \brce{R_{1, K} + \gamma_K \cV_{1}^{\bpi^{\prime}} (S_{2, 1}) \middle| S_{1, 1:K} = s_{1, 1:K}, A_{1, K} = a_{1, K} }
    \end{split}
    \end{equation*}
    at the decision time $k = K$.
    Notice that 
    \begin{equation*}
    \begin{split}
        \cV_{1}^{\bpi^{\prime}} (s_{2, 1})
        = & \bbE^{\bpi^{\prime}} \brce{\sum_{t=1}^{\infty} \sum_{l=1}^K \eta_{1, 1}^{t, l} R_{t, l} \middle| S_{1, 1} = s_{2, 1} } \\
        = & \bbE \brce{\sum_{t=2}^{\infty} \sum_{l=1}^K \eta_{2, 1}^{t, l} R_{t, l} \middle| S_{2, 1} = s_{2, 1}, A_{t, l} \sim \pi_{t, l} (S_{t, 1:l}) \text{ for } t \ge 2 } \\
        = & \bbE \brce{\sum_{t=2}^{\infty} \sum_{l=1}^K \eta_{2, 1}^{t, l} R_{t, l} \middle| S_{1, 1:K} = s_{1, 1:K}, A_{1, K} = a_{1, K}, S_{2, 1} = s_{2, 1}, A_{t, l} \sim \pi_{t, l} (S_{t, 1:l}) \text{ for } t \ge 2 } \\
        = & \bbE^{\bpi^{\prime}} \brce{\sum_{t=2}^{\infty} \sum_{l=1}^K \eta_{2, 1}^{t, l} R_{t, l} \middle| S_{1, 1:K} = s_{1, 1:K}, A_{1, K} = a_{1, K}, S_{2, 1} = s_{2, 1} },
    \end{split}
    \end{equation*}
    where the second line follows from the change of index, and the third line follows from the Markov property.
    Therefore,
    \begin{equation*}
    \begin{split}
        & \cQ_K^* (s_{1, 1:K}, a_{1, K}) \\
        \le & \bbE \brce{R_{1, K} + \gamma_K \bbE^{\bpi^{\prime}} \brce{\sum_{t=2}^{\infty} \sum_{l=1}^K \eta_{2, 1}^{t, l} R_{t, l} \middle| S_{1, 1:K} = s_{1, 1:K}, A_{1, K} = a_{1, K}, S_{2, 1} } \middle| S_{1, 1:K} = s_{1, 1:K}, A_{1, K} = a_{1, K} } \\
        = & \bbE^{\bpi^{\prime}} \brce{R_{1, K} + \gamma_K \sum_{t=2}^{\infty} \sum_{l=1}^K \eta_{2, 1}^{t, l} R_{t, l} \middle| S_{1, 1:K} = s_{1, 1:K}, A_{1, K} = a_{1, K} } 
        = \cQ_K^{\bpi^{\prime}} (s_{1, 1:K}, a_{1, K})
    \end{split}
    \end{equation*}
    for all $s_{1, 1:K}$ and $a_{1, K}$.
    The results follow by noticing that $\cQ_k^* (s_{1, 1:k}, a_{1, k}) \ge \cQ_k^{\bpi^{\prime}} (s_{1, 1:k}, a_{1, k})$ for all $k$.
\end{proof}

Now we are ready to prove Lemma~\ref{lem:stationary.policy}.

\begin{proof}
    According to the Markov property, $R_{1, K}$ and $S_{2, 1}$ are independent of $S_{1, 1:(K - 1)}$ and $A_{1, 1:(K - 1)}$ given $S_{1, K}$ and $A_{1, K}$.
    Therefore,
    \begin{equation} \label{equ:markov.value.K}
        \bbE \{ R_{1, K} + \gamma_K \cV_1^{*} (S_{2, 1}) | S_{1, 1:K} = s_{1, 1:K}, A_{1, K} = a_{1, K} \}
    = \bbE \{ R_{1, K} + \gamma_K \cV_1^{*} (S_{2, 1}) | S_{1, K} = s_{1, K}, A_{1, K} = a_{1, K} \}
    \end{equation}
    is a function of only $s_{1, K}$ and $a_{1, K}$.
    Define
    \begin{equation*}
        \begin{split}
            \bar{\pi}_K (s_{1, K}) = \argmax_{a_{1, K} \in \cA_K}
            \bbE \{ R_{1, K} + \gamma_K \cV_1^{*} (S_{2, 1}) |
            S_{1, K} = s_{1, K}, A_{1, K} = a_{1, K} \}.
        \end{split}
    \end{equation*}
    Notice that $\bar{\pi}_K (s_{1, K}) = \pi_K^{\prime} (s_{1, 1:K})$ for all $s_{1, 1:K} \in \cS_1 \times \dots, \cS_K$ by (\ref{equ:markov.value.K}).
    Then, according to equation~(\ref{equ:V.optimal.pi.prime.K}) in Lemma~\ref{lem:optimal.nonmarkov.policy},
    \begin{equation} \label{equ:Markov.value.function.K}
        \begin{split}
            \cV_K^{*} (s_{1, 1:K})
            = & \bbE \{ R_{1, K} + \gamma_K \cV_1^{*} (S_{2, 1}) | S_{1, 1:K} = s_{1, 1:K}, A_{1, K} \sim \pi_K^{\prime} (s_{1, 1:K}) \} \\
            = & \bbE \{ R_{1, K} + \gamma_K \cV_1^{*} (S_{2, 1}) | S_{1, 1:K} = s_{1, 1:K}, A_{1, K} \sim \bar{\pi}_{K} (s_{1, K}) \} \\
            = & \bbE \{ R_{1, K} + \gamma_K \cV_1^{*} (S_{2, 1}) | S_{1, K} = s_{1, K}, A_{1, K} \sim \bar{\pi}_{K} (s_{1, K}) \}
            =: \bar{\cV}_K^{\bar{\pi}_K} (s_{1, K})
        \end{split}
    \end{equation}
    for any $s_{1, 1:K} \in \cS_1 \times \dots, \cS_K$,
    where $\bar{\cV}_K^{\bar{\pi}_K} (s_{1, K})$ is some function that only depends on $s_{1, K}$.
    The second line follows from the definition of $\bar{\pi}_K$, and the third line follows from the Markov property.
    Here, we have not shown that $\bar{\cV}_K^{\bar{\pi}_K} (s_{1, K})$ is a value function of any policy.

    Similarly, for $k = K - 1$, define
    \begin{equation*}
        \begin{split}
            \bar{\pi}_k (s_{1, k}) 
            = \argmax_{a_{1, k} \in \cA_k}
            \bbE \{ \bar{\cV}_K^{\bar{\pi}_K} (S_{1, K})  |
            S_{1, k} = s_{1, k}, A_{1, k} = a_{1, k} \}.
        \end{split}
    \end{equation*}
    Notice that
    \begin{equation*}
        \begin{split}
            \bbE \{ \bar{\cV}_K^{\bar{\pi}_K} (S_{1, K}) | S_{1, k} = s_{1, k}, A_{1, k} = a_{1, k} \}
            = & \bbE \{ \bar{\cV}_K^{\bar{\pi}_K} (S_{1, K}) | S_{1, 1:k} = s_{1, 1:k}, A_{1, k} = a_{1, k} \} \\
            = & \bbE \{ \cV_K^{*} (s_{1, 1:k}, S_{1, K}) | S_{1, 1:k} = s_{1, 1:k}, A_{1, k} = a_{1, k} \},
        \end{split}
    \end{equation*}
    where the first equality follows from the Markov property for $S_{1, K}$, and the second equality is based on equation~(\ref{equ:Markov.value.function.K}).
    Therefore,
    \begin{equation*}
        \begin{split}
            \bar{\pi}_k (s_{1, k}) 
            = \argmax_{a_{1, k} \in \cA_k}
            \bbE \{ \cV_K^{*} (s_{1, 1:k}, S_{1, K}) |
            S_{1, 1:k} = s_{1, 1:k}, A_{1, k} = a_{1, k} \}
            = \pi^{\prime}_k (s_{1:k})
        \end{split}
    \end{equation*}
    for any $s_{1, 1:k} \in \cS_1 \times \dots, \cS_k$.
    Then, according to equation~(\ref{equ:V.optimal.pi.prime.k}) in Lemma~\ref{lem:optimal.nonmarkov.policy},
    \begin{equation*} \label{equ:Markov.value.function.k}
        \begin{split}
            \cV_k^{*} (s_{1, 1:k})
            = & \bbE \{ \cV_K^{*} (s_{1, 1:k}, S_{1, K}) | S_{1, 1:k} = s_{1, 1:k}, A_{1, k} \sim \pi_k^{\prime} (s_{1, 1:k}) \} \\
            = & \bbE \{ \bar{\cV}_K^{\bar{\pi}_K} (S_{1, K}) | S_{1, k} = s_{1, k}, A_{1, k} \sim \bar{\pi}_{k} (s_{1, k}) \}
            =: \bar{\cV}_k^{\bar{\pi}_k} (s_{1, k})
        \end{split}
    \end{equation*}
    for all $s_{1, 1:k} \in \cS_1 \times \dots, \cS_k$,
    where $\bar{\cV}_k^{\bar{\pi}_k} (s_{1, k})$ is some function that only depends on $s_{1, k}$.

    With the same arguments, we can use induction to define $\bar{\cV}_k^{\bar{\pi}_k} (s_{1, k})$ and 
    \begin{equation*}
        \begin{split}
            \bar{\pi}_k (s_{1, k}) 
            = \argmax_{a_{1, k} \in \cA_k}
            \bbE \{ \bar{\cV}_{k + 1}^{\bar{\pi}_{k + 1}} (S_{1, k + 1})  |
            S_{1, k} = s_{1, k}, A_{1, k} = a_{1, k} \}
        \end{split}
    \end{equation*}
    for all $k \in \{1:K\}$,
    and prove that $\bar{\cV}_k^{\bar{\pi}_k} (s_{1, k}) = \cV_k^{*} (s_{1, 1:k})$ for all $s_{1, 1:k} \in \cS_1 \times \dots, \cS_k$.

    Now we define the policy
    $\bar{\bpi} := \{ \bar{\pi}_{1:K} \}$.
    By definition, $\bar{\pi}_k (s_{1, k}) = \pi^{\prime}_k (s_{1:k})$ for all $k \in \{1:K\}$.
    Hence, the value function of $\bar{\bpi}$ is 
    $\cV_k^{\bar{\bpi}} (s_{1, 1:k}) = \cV_k^{\bpi^{\prime}} (s_{1, 1:k}) = \cV_k^{*} (s_{1, 1:k})$
    by Lemma~\ref{lem:optimal.nonmarkov.policy}, and we can conclude that $\bar{\cV}_k^{\bar{\pi}_k} (s_{1, k}) = \cV_k^{\bar{\bpi}} (s_{1, 1:k})$ is indeed the value function of $\bar{\bpi}$ and is optimal.
    Specifically, at time $k = 1$, we have $\cV_1^{\bar{\bpi}} (s_{1, 1}) = \cV_1^{*} (s_{1, 1})$, indicating that $\bar{\bpi}$ is an optimal policy.
    Similarly, $\cQ_k^{\bar{\bpi}} (s_{1, 1:k}, a_{1, k}) = \cQ_k^{\bpi^{\prime}} (s_{1, 1:k}, a_{1, k}) = \cQ_k^{*} (s_{1, 1:k}, a_{1, k})$ for all $k$ by Lemma~\ref{lem:optimal.nonmarkov.policy}.
\end{proof}

\subsection{Proof of Lemma~\ref{lem:markov.value.function}}

\begin{proof}
    Following the notations in Appendix~\ref{sec:proof.stationary.policy}, we have
    \begin{equation*}
    \begin{split}
        \cV_k^{\widetilde{\bpi}} (s_{1, 1:k}) 
        = & \bbE^{\bpi} \brce{\sum_{l=k}^K \eta_{1, k}^{1, l} R_{1, l} + \sum_{t=2}^{\infty} \sum_{l=1}^K \eta_{1, k}^{t, l} R_{t, l} \middle| S_{1, 1:k} = s_{1, 1:k} } \\
        = & \bbE \brce{\sum_{l=k}^K \eta_{1, k}^{1, l} R_{1, l} + \sum_{t=2}^{\infty} \sum_{l=1}^K \eta_{1, k}^{t, l} R_{t, l} \middle| S_{1, 1:k} = s_{1, 1:k}, A_{t, l} \sim \widetilde{\pi}_{l} (S_{t, l}) \text{ for } t = 1, l \ge k \text{ or } t \ge 2 } \\
        = & \bbE \brce{\sum_{l=k}^K \eta_{1, k}^{1, l} R_{1, l} + \sum_{t=2}^{\infty} \sum_{l=1}^K \eta_{1, k}^{t, l} R_{t, l} \middle| S_{1, k} = s_{1, k}, A_{t, l} \sim \widetilde{\pi}_{l} (S_{t, l}) \text{ for } t = 1, l \ge k \text{ or } t \ge 2 } \\
        = & \bbE^{\bpi} \brce{\sum_{l=k}^K \eta_{1, k}^{1, l} R_{1, l} + \sum_{t=2}^{\infty} \sum_{l=1}^K \eta_{1, k}^{t, l} R_{t, l} \middle| S_{1, k} = s_{1, k} }.
    \end{split}
    \end{equation*}
    The first and second equalities are based on the definition of $\cV_k^{\widetilde{\bpi}} (s_{1, 1:k})$ and the definition of $\bbE^{\bpi}$.
    The third equality follows from the Markov property, i.e. $R_{t, l}$ for $t = 1, l \ge k$ or $t \ge 2$ is independent of $S_{1, 1:(k - 1)}$ and $A_{1, 1:(k - 1)}$ given $S_{1, k}$ and $A_{1, k}$.
    Therefore, the value function of $\widetilde{\bpi}$ can be denoted as $\cV_k^{\widetilde{\bpi}} (s_{1, k})$.

    According to Lemma~\ref{lem:stationary.policy}, there exists a Markov and stationary policy $\bar{\bpi}$ that is optimal.
    Since $\bar{\bpi} \in \widetilde{\Pi}$,
    \[ 
    \cV_k^{*} (s_{1, 1:k})
    = \cV_k^{\bar{\bpi}} (s_{1, 1:k}) 
    \le \sup_{\widetilde{\bpi} \in \widetilde{\Pi}} \cV_k^{\widetilde{\bpi}} (s_{1, k}) 
    = \cV_k^{*} (s_{1, k}) 
    \]
    for all $k \in \{1:K\}$ and all $s_{1, 1:k} \in \cS_1 \times \dots, \cS_k$.
    On the other hand, $\cV_k^{*} (s_{1, k}) \le \cV_k^{*} (s_{1, 1:k})$ since $\widetilde{\Pi} \subseteq \Pi$.
    Therefore, $\cV_k^{*} (s_{1, k}) = \cV_k^{*} (s_{1, 1:k})$ for all $k \in \{1:K\}$ and all $s_{1, 1:k} \in \cS_1 \times \dots \times \cS_k$.

    The value function of $\widetilde{\bpi}$ at time $k$ in bag $d \ge 1$ satisfies
    \begin{equation*}
    \begin{split}
        & \bbE^{\widetilde{\bpi}} \brce{\sum_{l=k}^K \eta_{d, k}^{d, l} R_{d, l} + \sum_{t=d + 1}^{\infty} \sum_{l=1}^K \eta_{d, k}^{t, l} R_{t, l} \middle| S_{1, 1} = s_{1, 1}, \dots, S_{d, k} = s_{d, k} } \\
        = & \bbE \brce{\sum_{l=k}^K \eta_{d, k}^{d, l} R_{d, l} + \sum_{t=d + 1}^{\infty} \sum_{l=1}^K \eta_{d, k}^{t, l} R_{t, l} \middle| S_{d, k} = s_{d, k}, A_{t, l} \sim \widetilde{\pi}_{l} (S_{t, l}) \text{ for } (t, l) \ge (d, k) } \\
        = & \bbE \brce{\sum_{l=k}^K \eta_{1, k}^{1, l} R_{1, l} + \sum_{t=2}^{\infty} \sum_{l=1}^K \eta_{1, k}^{t, l} R_{t, l} \middle| S_{1, k} = s_{d, k}, A_{t, l} \sim \widetilde{\pi}_{l} (S_{t, l}) \text{ for } (t, l) \ge (1, k) } \\
        = & \cV_k^{\widetilde{\bpi}} (s_{d, k}),
    \end{split}
    \end{equation*}
    where the second line comes from the Markov property and the third line follows from the change of index.
    Let $\widetilde{\pi}_{t, l, s_{1:(d - 1), 1:K}} (S_{d, k}, \dots, S_{t, l})$ denote the policy $\pi_{t + d - 1, l} (s_{1, 1:K}, \dots, s_{d - 1, 1:K}, S_{d, k}, \dots, S_{t, l})$.
    Then, the optimal value function satisfies
    \begin{equation*}
        \sup_{\bpi \in \Pi} \bbE^{\bpi} \brce{\sum_{l=k}^K \eta_{d, k}^{d, l} R_{d, l} + \sum_{t=d + 1}^{\infty} \sum_{l=1}^K \eta_{d, k}^{t, l} R_{t, l} \middle| S_{1, 1} = s_{1, 1}, \dots, S_{d, k} = s_{d, k} } 
        = \cV_k^{*} (s_{d, 1:k})
        = \cV_k^{*} (s_{d, k}).
    \end{equation*}

    Similar arguments can be used to prove the results about Q-functions.
\end{proof}

\subsection{Proof of Theorem~\ref{thm:optimal.policy}}

\begin{proof}
    We begin by showing that 
    \begin{equation} \label{equ:Bellman.V.Q}
        \cV^*_k (s_{1, k}) = \max_{a_{1, k} \in \cA_k} \cQ_k^{*} (s_{1, k}, a_{1, k}).
    \end{equation}
    Lemma~\ref{lem:stationary.policy} shows that there is a Markov and stationary policy $\bar{\bpi}$, and that
    $\cV_k^{\bar{\bpi}} (s_{1, 1:k}) = \cV_k^{*} (s_{1, 1:k})$ and
    $\cQ_k^{\bar{\bpi}} (s_{1, 1:k}, a_{1, k}) = \cQ_k^{*} (s_{1, 1:k}, a_{1, k})$ 
    for all $k$.
    Since $\bar{\bpi} \in \widetilde{\Pi}$, we have 
    $\cV_k^{\bar{\bpi}} (s_{1, k}) = \cV_k^{*} (s_{1, k})$ and
    $\cQ_k^{\bar{\bpi}} (s_{1, k}, a_{1, k}) = \cQ_k^{*} (s_{1, k}, a_{1, k})$ 
    according to Lemma~\ref{lem:markov.value.function}.
    Now due to the definition that $\cV_k^{*}$ is the optimal value function over all possible policies in $\widetilde{\Pi}$,
    \begin{equation*}
    \begin{split}
        \cV_k^{*} (s_{1, k}) 
        = & \sup_{\widetilde{\bpi} \in \widetilde{\Pi}} \bbE^{\widetilde{\bpi}} \brce{ \sum_{l=k}^K \eta_{1, k}^{1, l} R_{1, l} + \sum_{t=2}^{\infty} \sum_{l=1}^K \eta_{1, k}^{t, l} R_{t, l} \middle| S_{1, k} = s_{1, k} } \\
        \ge & \bbE^{\bar{\bpi}} \brce{ \sum_{l=k}^K \eta_{1, k}^{1, l} R_{1, l} + \sum_{t=2}^{\infty} \sum_{l=1}^K \eta_{1, k}^{t, l} R_{t, l} \middle| S_{1, k} = s_{1, k}, A_{1, k} = a_{1, k} }
        = \cQ_k^{\bar{\bpi}} (s_{1, k}, a_{1, k})
        = \cQ_k^{*} (s_{1, k}, a_{1, k})
    \end{split}
    \end{equation*}
    for any $a_{1, k} \in \cA_k$ and any $k$.
    Hence, we have $\cV_k^{*} (s_{1, k}) \ge \max_{a_{1, k} \in \cA_k} \cQ_k^{*} (s_{1, k}, a_{1, k})$.
    The conclusion follows by showing that
    \begin{equation*}
    \begin{split}
        \cV_k^{*} (s_{1, k}) 
        = \cV_k^{\bar{\bpi}} (s_{1, k}) 
        = \cQ_k^{\bar{\bpi}} (s_{1, k}, \bar{\pi}_k (s_{1, k})) 
        \le \max_{a_{1, k} \in \cA_k} \cQ_k^{\bar{\bpi}} (s_{1, k}, a_{1, k})
        = \max_{a_{1, k} \in \cA_k} \cQ_k^{*} (s_{1, k}, a_{1, k}).
    \end{split}
    \end{equation*}

    We first show the sufficiency, i.e. the optimal Q-function $\cQ^*_k$ satisfies equations~(\ref{equ:Bellman.optimality.equations}).
    For all $s_{1, K}$ and $a_{1, K}$,
    \begin{equation*}
    \begin{split}
        \cQ_K^{*} (s_{1, K}, a_{1, K})
        = & \sup_{\widetilde{\bpi} \in \widetilde{\Pi}} \cQ_K^{\widetilde{\bpi}} (s_{1, K}, a_{1, K}) \\
        = & \sup_{\widetilde{\bpi} \in \widetilde{\Pi}} \bbE \brce{R_{1, K} + \gamma_K \cV^{\widetilde{\bpi}}_1 (S_{2, 1}) \middle| S_{1, K} = s_{1, K}, A_{1, K} = a_{1, K} } \\
        = & \bbE \brce{R_{1, K} \middle| S_{1, K} = s_{1, K}, A_{1, K} = a_{1, K} }
        + \gamma_K \sup_{\widetilde{\bpi} \in \widetilde{\Pi}} \bbE \brce{\cV^{\widetilde{\bpi}}_1 (S_{2, 1}) \middle| S_{1, K} = s_{1, K}, A_{1, K} = a_{1, K} }.
    \end{split}
    \end{equation*}
    The second equality follows from the Bellman equation, which can be proved using standard arguments in stationary MDP.
    The third equality follows from the fact that $R_{1, K}$ is independent of the policy $\widetilde{\bpi}$ given $S_{1, K}$ and $A_{1, K}$.
    For the optimal policy $\bar{\bpi} \in \widetilde{\Pi}$, we have $\cV^{\bar{\bpi}}_1 (s_{2, 1}) = \cV^{*}_1 (s_{2, 1})$ for all $s_{2, 1} \in \cS_1$.
    Notice that 
    \begin{equation*}
    \begin{split}
        \sup_{\widetilde{\bpi} \in \widetilde{\Pi}} \bbE \brce{\cV^{\widetilde{\bpi}}_1 (S_{2, 1}) \middle| S_{1, K} = s_{1, K}, A_{1, K} = a_{1, K} }
        \le & \bbE \brce{\sup_{\widetilde{\bpi} \in \widetilde{\Pi}} \cV^{\widetilde{\bpi}}_1 (S_{2, 1}) \middle| S_{1, K} = s_{1, K}, A_{1, K} = a_{1, K} } \\
        = & \bbE \brce{\cV^{*}_1 (S_{2, 1}) \middle| S_{1, K} = s_{1, K}, A_{1, K} = a_{1, K} }
    \end{split}
    \end{equation*}
    due to the convexity of the $\sup$ function.
    On the other hand,
    \begin{equation*}
    \begin{split}
        \sup_{\widetilde{\bpi} \in \widetilde{\Pi}} \bbE \brce{\cV^{\widetilde{\bpi}}_1 (S_{2, 1}) \middle| S_{1, K} = s_{1, K}, A_{1, K} = a_{1, K} }
        \ge & \bbE \brce{\cV^{\bar{\bpi}}_1 (S_{2, 1}) \middle| S_{1, K} = s_{1, K}, A_{1, K} = a_{1, K} } \\
        = & \bbE \brce{\cV^{*}_1 (S_{2, 1}) \middle| S_{1, K} = s_{1, K}, A_{1, K} = a_{1, K} }.
    \end{split}
    \end{equation*}
    Therefore, we have that 
    \begin{equation*}
    \begin{split}
        \sup_{\widetilde{\bpi} \in \widetilde{\Pi}} \bbE \brce{\cV^{\widetilde{\bpi}}_1 (S_{2, 1}) \middle| S_{1, K} = s_{1, K}, A_{1, K} = a_{1, K} }
        = \bbE \brce{\cV^{*}_1 (S_{2, 1}) \middle| S_{1, K} = s_{1, K}, A_{1, K} = a_{1, K} }.
    \end{split}
    \end{equation*}
    Consequently,
    \begin{equation*}
    \begin{split}
        \cQ_K^{*} (s_{1, K}, a_{1, K})
        = & \bbE \brce{R_{1, K} \middle| S_{1, K} = s_{1, K}, A_{1, K} = a_{1, K} }
        + \gamma_K \bbE \brce{\cV^{*}_1 (S_{2, 1}) \middle| S_{1, K} = s_{1, K}, A_{1, K} = a_{1, K} } \\
        = & \bbE \brce{R_{1, K} + \gamma_K \max_{a_{2, 1} \in \cA_1} \cQ_1^{*} (S_{2, 1}, a_{2, 1}) \middle| S_{1, K} = s_{1, K}, A_{1, K} = a_{1, K} }.
    \end{split}
    \end{equation*}
    The last equality follows from equation~(\ref{equ:Bellman.V.Q}).
    Similarly, 
    \begin{equation*}
    \begin{split}
        \cQ_k^{*} (s_{1, k}, a_{1, k})
        = & \bbE \brce{R_{1, k} + \gamma_k \max_{a_{1, k + 1} \in \cA_{k + 1}} \cQ_{k + 1}^{*} (S_{1, k + 1}, a_{1, k + 1}) \middle| S_{1, k} = s_{1, k}, A_{1, k} = a_{1, k} }
    \end{split}
    \end{equation*}
    for $k \in \{1:K\} - 1$.

    To show the sufficiency, we need to prove that any functions $Q_k: \cS_k \times \cA_k \mapsto \bbR$ satisfying  the Bellman optimality equations are the optimal Q-functions, i.e. $Q_k = Q_k^*$.
    For any policy $\widetilde{\bpi} \in \widetilde{\Pi}$,
    \begin{equation*}
    \begin{split}
        & \cQ_1 (s_{1, 1}, a_{1, 1}) \\
        = & \bbE \brce{ R_{1, 1} + \gamma_1 \max_{a_{1, 2} \in \cA_{2}} \cQ_{2} (S_{1, 2}, a_{1, 2}) \middle| S_{1, 1} = s_{1, 1}, A_{1, 1} = a_{1, 1} } \\
        \ge & \bbE \brce{ R_{1, 1} + \gamma_1 \cQ_{2} (S_{1, 2}, \widetilde{\pi}_2 (S_{1, 2})) \middle| S_{1, 1} = s_{1, 1}, A_{1, 1} = a_{1, 1} } \\
        = & \bbE \brce{ R_{1, 1} + \gamma_1 \bbE \brce{ R_{1, 2} + \gamma_2 \max_{a_{1, 3} \in \cA_{3}} \cQ_{3} (S_{1, 3}, a_{1, 3}) \middle| S_{1, 2}, A_{1, 2} = \widetilde{\pi}_2 (S_{1, 2}), S_{1, 1}, A_{1, 1} } \middle| S_{1, 1} = s_{1, 1}, A_{1, 1} = a_{1, 1} } \\
        = & \bbE^{\widetilde{\bpi}} \brce{R_{1, 1} + \gamma_1 R_{1, 2} + \gamma_1 \gamma_2 \max_{a_{1, 3} \in \cA_{3}} \cQ_{3} (S_{1, 3}, a_{1, 3}) \middle| S_{1, 1} = s_{1, 1}, A_{1, 1} = a_{1, 1} }.
    \end{split}
    \end{equation*}
    In the last equality, $\bbE^{\widetilde{\bpi}}$ represents that $A_{1, 2}$ follows $\widetilde{\pi}_2$.
    Using the same arguments recursively, we have
    \begin{equation*}
    \begin{split}
        \cQ_1 (s_{1, 1}, a_{1, 1}) 
        \ge & \bbE^{\widetilde{\bpi}} \brce{ R_{1, 1} + \gamma_1 R_{1, 2} + \gamma_1 \gamma_2 \cQ_{3} (S_{1, 3}, \widetilde{\pi}_3 (S_{1, 3})) \middle| S_{1, 1} = s_{1, 1}, A_{1, 1} = a_{1, 1} } \\
        \ge & \dots \\
        \ge & \bbE^{\widetilde{\bpi}} \brce{ R_{1, 1} + \gamma_1 R_{1, 2} + \gamma_1 \gamma_2 R_{1, 3} + \dots \middle| S_{1, 1} = s_{1, 1}, A_{1, 1} = a_{1, 1} } \\
        = & \cQ_1^{\widetilde{\bpi}} (s_{1, 1}, a_{1, 1})
    \end{split}
    \end{equation*}
    for all $\widetilde{\bpi} \in \widetilde{\Pi}$.
    Therefore, 
    \begin{equation*}
    \begin{split}
        \cQ_1 (s_{1, 1}, a_{1, 1}) 
        \ge \sup_{\widetilde{\bpi} \in \widetilde{\Pi}} \cQ_1^{\widetilde{\bpi}} (s_{1, 1}, a_{1, 1})
        = \cQ^*_1 (s_{1, 1}, a_{1, 1})
    \end{split}
    \end{equation*}
    On the other hand, take the policy $\bpi^* = \{\pi_{1:K}^*\}$, where
    $\pi_k^* (s_{1, k}) = \argmax_{a_{1, k} \in \cA_k} \cQ_k (s_{1, k}, a_{1, k})$.
    By the definition of $\bpi^*$,
    \begin{equation*}
    \begin{split}
        \cQ_1 (s_{1, 1}, a_{1, 1}) 
        = & \bbE \brce{ R_{1, 1} + \gamma_1 \max_{a_{1, 2} \in \cA_{2}} \cQ_{2} (S_{1, 2}, a_{1, 2}) \middle| S_{1, 1} = s_{1, 1}, A_{1, 1} = a_{1, 1} } \\
        = & \bbE \brce{ R_{1, 1} + \gamma_1 \cQ_{2} (S_{1, 2}, \pi^*_2 (S_{1, 2})) \middle| S_{1, 1} = s_{1, 1}, A_{1, 1} = a_{1, 1} } \\
        = & \bbE^{\bpi^*} \brce{ R_{1, 1} + \gamma_1 R_{1, 2} + \gamma_1 \gamma_2 \max_{a_{1, 3} \in \cA_{3}} \cQ_{3} (S_{1, 3}, a_{1, 3}) \middle| S_{1, 1} = s_{1, 1}, A_{1, 1} = a_{1, 1} } \\
        = & \dots \\
        = & \cQ_1^{\bpi^*} (s_{1, 1}, a_{1, 1}).
    \end{split}
    \end{equation*}
    Hence, we have $\cQ_1 (s_{1, 1}, a_{1, 1}) = \cQ_1^{\bpi^*} (s_{1, 1}, a_{1, 1})$ for all $s_{1, 1}$ and $a_{1, 1}$, and can conclude that $\cQ_1$ is the Q-function of $\bpi^*$.
    Since $\bpi^* \in \widetilde{\Pi}$,
    \begin{equation*}
    \begin{split}
        \cQ_1 (s_{1, 1}, a_{1, 1}) 
        = \cQ_1^{\bpi^*} (s_{1, 1}, a_{1, 1})
        \le \sup_{\widetilde{\bpi} \in \widetilde{\Pi}} \cQ_1^{\widetilde{\bpi}} (s_{1, 1}, a_{1, 1})
        = \cQ^*_1 (s_{1, 1}, a_{1, 1}).
    \end{split}
    \end{equation*}
    Therefore, we can conclude that $\cQ_1 (s_{1, 1}, a_{1, 1}) = \cQ^*_1 (s_{1, 1}, a_{1, 1})$ for all $s_{1, 1}$ and $a_{1, 1}$, and $\bpi^*$ is an optimal policy.
    The equivalence between $\cQ_k (s_{1, k}, a_{1, k}) = \cQ^*_k (s_{1, k}, a_{1, k})$ for $k = 2, \dots, K$ can be proved similarly.
\end{proof}

\subsection{Proof of Lemma~\ref{lem:markov.transition}} \label{sec:proof.markov.transition}

\begin{proof}
    We prove the results for a general problem, where there is a reward $R_{d, k}$ at each decision time $k$.
    In the problems with bagged decision times, we take $R_{d, 1:(K - 1)} = 0$ and $R_{d, K} = R_d$.
    Since the rewards $R_{d, 1:(K - 1)}$ are defined to be constants, they are essentially independent of any other variables in the DAG and do not affect the sufficiency condition.
    For notation simplicity, the indices $(d - 1, K)$ and $(d, 0)$, $(d + 1, 1)$ and $(d, K + 1)$, or $(K + 1)$ and $(1)$ are used interchangeably.

    \textbf{Step 1.} 
    Show that the reward function is Markovian, i.e. $R_{d, k} \indep S_{t, l}, A_{t, l} | S_{d, k}, A_{d, k}$ for any $(t, l) < (d, k)$.

    By Assumption~\ref{asp:state.D.BaSS},
    \begin{equation*}
        R_{d, k}, R_{d, k + 1}, \dots \indep \cH_{d, k} | S_{d, k}, A_{d, k}.
    \end{equation*}
    Since $\cH_{d, k - 1} \subseteq \cH_{d, k}$, we have $R_{d, k} \indep \cH_{d, k - 1} | S_{d, k}, A_{d, k}$.
    In addition, since $A_{d, k - 1} \in \cH_{d, k - 1}$ and $S_{d, k - 1} = F_t (\cH_{d, k - 1})$, we can conclude that $R_{d, k} \indep S_{d, k - 1}, A_{d, k - 1} | S_{d, k}, A_{d, k}$.
    Similarly, for any $(t, l) < (d, k)$, we have $R_{d, k} \indep S_{t, l}, A_{t, l} | S_{d, k}, A_{d, k}$ and thus the reward function is Markovian.

    \textbf{Step 2.} 
    Show that the state transition is Markovian, i.e. $S_{d, k + 1} \indep S_{t, l}, A_{t, l} | S_{d, k}, A_{d, k}$ when $k < K$ and $S_{d + 1, 1} \indep S_{t, l}, A_{t, l} | S_{d, K}, A_{d, K}$ when $k = K$ for any $(t, l) < (d, k)$.

    We show the results for $k < K$, and the results for $k = K$ can be shown analogously.
    Using similar arguments as in step 1, we have that $R_{m, n} \indep S_{d, k}, A_{d, k} | S_{d, k + 1}, A_{d, k + 1}$ for all $(m, n) \ge (d, k + 1)$.
    Therefore, we must have (1) $R_{m, n} \nindep S_{d, k + 1}, A_{d, k + 1} | S_{d, k}, A_{d, k}$, (2) $R_{m, n} \indep S_{d, k + 1}, A_{d, k + 1}$ and $R_{m, n} \indep S_{d, k}, A_{d, k}$, or (3) $S_{d, k + 1}, A_{d, k + 1}$ equals $S_{d, k}, A_{d, k}$ almost surely.
    Notice that for three vectors $X, Y, Z$, if $X \indep Y | Z$ and $X \indep Z | Y$, then $\bbP (X | Y, Z) = \bbP (X | Y) = \bbP (X | Z)$.
    This holds true only when $Y = Z$ almost surely or $X \indep Y$ and $X \indep Z$.
    
    When (1) $R_{m, n} \nindep S_{d, k + 1}, A_{d, k + 1} | S_{d, k}, A_{d, k}$, we have that $R_{m, n}$ and $S_{d, k + 1}$ are not d-separated by $S_{d, k}, A_{d, k}$.
    We will prove $S_{d, k + 1} \indep S_{t, l} | S_{d, k}, A_{d, k}$ by contradiction.
    If $S_{d, k + 1} \nindep S_{t, l} | S_{d, k}, A_{d, k}$ for some $(t, l)$, then $S_{d, k + 1}$ and $S_{t, l}$ are not d-separated by $S_{d, k}, A_{d, k}$.
    Then the pair $S_{d, k}, A_{d, k}$ does not contain any middle node on a chain or a fork and contain all colliders or its descendants on at least one path from $S_{t, l} \to S_{d, k + 1}$ and from $S_{d, k + 1} \to R_{m, n}$.
    Notice that $R_{m, n} \indep S_{t, l} | S_{d, k + 1}, A_{d, k + 1}$ since $(t, l) < (d, k + 1)$.
    Therefore, $S_{d, k + 1}, A_{d, k + 1}$ cannot contain any collider on any path $S_{t, l} \to S_{d, k + 1} \to R_{m, n}$.
    When conditioning on $S_{d, k}, A_{d, k}$, the path $S_{t, l} \to S_{d, k + 1} \to R_{m, n}$ is not d-separated.
    However, Assumption~\ref{asp:state.D.BaSS} for $S_{d, k}$ requires that $R_{m, n} \indep \cH_{d, k} | S_{d, k}, A_{d, k}$ and thus $R_{m, n} \indep S_{t, l} | S_{d, k}, A_{d, k}$.
    This completes the proof by contradiction.
    Similarly, we can show that $S_{d, k + 1} \indep A_{t, l} | S_{d, k}, A_{d, k}$ since $A_{t, l} \in \cH_{d, k}$.

    When (2) $R_{m, n} \indep S_{d, k + 1}, A_{d, k + 1}$ and $R_{m, n} \indep S_{d, k}, A_{d, k}$ for all $(m, n) \ge (d, k + 1)$, the assumption that $R_d$ is a descendant of at least one variable $V_{d, k} \in \bB_d$ is violated.
    Otherwise, if $V_{d, k}$ is an ancestor of $R_d$ for some $k$, we must have $V_{d, k} \in S_{d, k + 1}$ by Assumption~\ref{asp:state.D.BaSS} for $S_{d, k + 1}$ since $V_{d, k} \in \cH_{d, k}$, and $R_{d} \indep S_{d, k + 1}$ does not hold.

    When (3) $S_{d, k + 1}, A_{d, k + 1}$ equals $S_{d, k}, A_{d, k}$ almost surely, we have immediately that $S_{d, k + 1} \indep S_{t, l}, A_{t, l} | S_{d, k}, A_{d, k}$.

    \textbf{Step 3.} 
    Show that $S_{d, k + 1} \indep \widetilde{\bB}_{d - 1} | S_{d, k}, A_{d, k}$ for $k \le K - 1$ and $R_d, S_{d + 1, 1} \indep \widetilde{\bB}_{d - 1} | S_{d, K}, A_{d, K}$.

    The statements $S_{d, k + 1} \indep \widetilde{\bB}_{d - 1} | S_{d, k}, A_{d, k}$ for $k \le K$ (with the index $(d + 1, K) = (d, K + 1)$) can be proved using the same arguments as Step 2 with $S_{t, l}, A_{t, l}$ substituted by $\widetilde{\bB}_{d - 1}$, since $S_{t, l}, A_{t, l}, \widetilde{\bB}_{d - 1} \in \cH_{d, k}$.
    The statement $R_d \indep \widetilde{\bB}_{d - 1} | S_{d, K}, A_{d, K}$ holds true by Assumption~\ref{asp:state.D.BaSS} for $S_{d, K}$, which requires $R_{d} \indep \cH_{d, k} | S_{d, K}, A_{d, K}$ and $\widetilde{\bB}_{d - 1} \in \cH_{d, K}$.
\end{proof}

\subsection{Proof of Lemma~\ref{lem:rl.framework}} \label{sec:proof.rl.framework}

\begin{proof}
    \textbf{Step 1.} 
    The state space and action space is stationary across bags.

    The state space $\cU_k$ is stationary across bags according to the condition (1) in Assumption~\ref{asp:markovian.state}, and the action space $\cA_k$ are stationary across bags by Definition~\ref{dfn:bagged}.

    \textbf{Step 2.} 
    The state transition is Markovian, stationary across bags and Markovian, non-stationary within bags.
    
    The condition (2) in Assumption~\ref{asp:markovian.state} assumes that the state transition is Markovian within and across bags.
    We first show that the state transition is stationary across bags from $U_{d, k}, A_{d, k}$ to $U_{d, k + 1}$ for $k \le K - 1$.
    Definition~\ref{dfn:bagged} requires that
    $\bbP (\bB_{d + 1} | \widetilde{\bB}_{d} = \bb, A_{d + 1, 1:K} = a_{1:K}) = \bbP (\bB_{d} | \widetilde{\bB}_{d - 1} = \bb, A_{d, 1:K} = a_{1:K})$.
    Integrating over variables $\bB^{\prime}_d := \bB_d \backslash (U_{d, k} \cup A_{d, 1:K})$ that are not in the state $U_{d, k}$, we have that the distribution 
    \[ \bbP (U_{d, k} | \widetilde{\bB}_{d - 1} = \bb, A_{d, 1:K} = a_{1:K}) = \int \bbP (U_{d, k}, \bB^{\prime}_d = \bb^{\prime}_d | \widetilde{\bB}_{d - 1} = \bb, A_{d, 1:K} = a_{1:K}) d \bb^{\prime}_d \] 
    is stationary across bags.
    Similarly, $\bbP (U_{d, k}, U_{d, k + 1} | \widetilde{\bB}_{d - 1} = \bb, A_{d, 1:K} = a_{1:K})$ is also stationary across bags.
    According to the condition (4) in Assumption~\ref{asp:markovian.state}, $U_{d, k + 1} \indep \widetilde{\bB}_{d - 1} | U_{d, k}, A_{d, k}$ for $k < K$.
    Besides, $U_{d, k + 1} (a_{d, 1:k})$ is only a potential outcome of $a_{d, 1:k}$.
    Then we have
    \[ \bbP (U_{d, k + 1} | U_{d, k}, A_{d, 1:k}) = \bbP (U_{d, k + 1} | U_{d, k}, \widetilde{\bB}_{d - 1}, A_{d, 1:K}) = \frac{\bbP (U_{d, k}, U_{d, k + 1} | \widetilde{\bB}_{d - 1}, A_{d, 1:K})} {\bbP (U_{d, k} | \widetilde{\bB}_{d - 1}, A_{d, 1:K})} \]
    is stationary across bags.
    The condition (2) in Assumption~\ref{asp:markovian.state} assumes that the state transition is Markovian within a bag, which implies that $U_{d, k + 1} \indep A_{d, 1:(k - 1)} | U_{d, k}, A_{d, k}$.
    The Markovian state transition within a bag implies that $U_{d, k + 1} \indep A_{d, 1:(k - 1)} | U_{d, k}, A_{d, k}$
    Therefore, $\bbP (U_{d, k + 1} | U_{d, k}, A_{d, k}) = \bbP (U_{d, k + 1} | U_{d, k}, A_{d, 1:k})$ is stationary across bags.

    For the state transition from $U_{d - 1, K}, A_{d - 1, K}$ to $U_{d, 1}$, the condition (4) requires that $U_{d, 1} \indep \widetilde{\bB}_{d - 2} | U_{d - 1, K}, A_{d - 1, K}$, and the condition (2) requires that $U_{d, 1} \indep A_{d - 1, 1:(K - 1)} | U_{d - 1, K}, A_{d - 1, K}$.
    Hence we have $\bbP (U_{d, 1} | U_{d - 1, K}, A_{d - 1, K}) = \bbP (U_{d, 1} | U_{d - 1, K}, A_{d - 1, K}, \widetilde{\bB}_{d - 2}, A_{d - 1, 1:(K - 1)})$.
    Notice that 
    \begin{equation*}
        \bbP (U_{d, 1} | U_{d - 1, K}, \widetilde{\bB}_{d - 2}, A_{d - 1, 1:K})
        = \frac{\bbP (U_{d, 1}, U_{d - 1, K} | \widetilde{\bB}_{d - 2}, A_{d - 1, 1:K})}{\bbP (U_{d - 1, K} | \widetilde{\bB}_{d - 2}, A_{d - 1, 1:K})},
    \end{equation*}
    where the numerator is stationary across bags if we integrate $\bbP (\bB_{d - 1} | \widetilde{\bB}_{d - 2} = \bb, A_{d - 1, 1:K} = a_{1:K})$ over the variables other than $U_{d - 1, K}$.
    For the denominator, since $U_{d, 1} \subseteq \bB_d \cup \widetilde{\bB}_{d - 1}$ and $U_{d - 1, K} \subseteq \bB_{d - 1} \cup \widetilde{\bB}_{d - 2}$,
    \begin{equation*}
        \begin{split}
            \bbP (U_{d, 1}, U_{d - 1, K} | \widetilde{\bB}_{d - 2}, A_{d - 1, 1:K})
            = & \int \bbP (U_{d, 1}, U_{d - 1, K}, \widetilde{\bB}_{d - 1} = \bb_{d - 1} | \widetilde{\bB}_{d - 2}, A_{d - 1, 1:K}) d \bb_{d - 1} \\
            = & \int \bbP (U_{d, 1} | \widetilde{\bB}_{d - 1} = \bb_{d - 1}) \bbP (U_{d - 1, K}, \widetilde{\bB}_{d - 1} = \bb_{d - 1} | \widetilde{\bB}_{d - 2}, A_{d - 1, 1:K}) d \bb_{d - 1} \\
        \end{split}
    \end{equation*}
    is also stationary across bags.
    Therefore, $\bbP (U_{d, 1} | U_{d - 1, K}, A_{d - 1, K})$ is stationary across bags.

    \textbf{Step 3.} 
    The reward function is Markovian, stationary across bags and Markovian, non-stationary within bags.
    
    When $R_{d, k} = 0$ is defined as a constant for $k \le K - 1$, we have $R_{d, k} \indep U_{d, k}, A_{d, k}$.
    Besides, the reward function $r_{k} (u_{d, k}, a_{d, k}) = 0$ is stationary across bags.
    
    For the last reward $R_{d, K} = R_d$, the condition (3) in Assumption~\ref{asp:markovian.state} requires that $R_{d} \indep U_{t, l}, A_{t, l} | U_{d, K}, A_{d, K}$ for any $(t, l) < (d, K)$, which guarantees that the reward function is Markovian. 
    Similar as Step 2, the distribution $\bbP (R_d | U_{d, K}, A_{d, K}) = \bbP (R_d | U_{d, K}, A_{d, K}, \widetilde{\bB}_{d - 1}) = \frac{\bbP (R_d, U_{d, K}, A_{d, K} | \widetilde{\bB}_{d - 1})}{\bbP (U_{d, K}, A_{d, K} | \widetilde{\bB}_{d - 1})}$  is stationary across bags.

    \textbf{Step 4.} 
    With the reward defined as $R_{d, 1:(K - 1)} = 0$, $R_{d, K} = R_d$ and the discount factor defined as $\gamma_{1:(K - 1)} = 1$, $\gamma_K = \bar{\gamma}$ $\gamma_{k} = 1$, we can recover the value function $\bbE^{\widetilde{\bpi}} \brce{ \sum_{d=1}^{\infty} \bar{\gamma}^{d - 1} R_d | U_{1, 1} = u_{1, 1} }$.

    We aim to maximize the cumulative rewards given the initial state $u_{1, 1}$, discounted every bag.
    With the discount factors defined as $\gamma_{k} = 1$ for $k < K$ and $\gamma_K = \bar{\gamma}$ for $k = K$, we have
    $\bar{\gamma} = \prod\nolimits_{k = 1}^K \gamma_k$.
    Notice that when $R_{d, 1:(K - 1)} = 0$,
    \begin{equation*}
        \begin{split}
            \bbE^{\widetilde{\bpi}} \brce{ \sum_{d=1}^{\infty} \bar{\gamma}^{d - 1} R_d | U_{1, 1} = u_{1, 1} } 
            = \bbE^{\widetilde{\bpi}} \brce{ 
            \sum_{l=k}^K \eta_{d, k}^{d, l} R_{d, l} + 
            \sum_{t=d + 1}^{\infty} \sum_{l=1}^K \eta_{d, k}^{t, l} R_{t, l} | U_{1, 1} = u_{1, 1} }
        \end{split}
    \end{equation*}
    for any $\widetilde{\bpi} \in \widetilde{\Pi}$,
    where $\eta_{d, k}^{t, l} = \allowbreak \prod_{(m, n): (d, k) \le (m, n) < (t, l)} \gamma_n$.
    Therefore, the decision process in Figure~\ref{fig:dag} is a $K$-Periodic MDP with the desired value function.
\end{proof}

However, this process is not a stationary MDP on the bag level.
The reason is that the outcome of $A_{d, k}$ is revealed before the next decision time $k + 1$ and used for choosing the next actions $A_{d, k + 1}$ in the same bag.

In Figure~\ref{fig:dag}, consider the bag state $Z_d := [M_{d-1, 1:K}, N_{d-1, 1:K}, E_{d-1}, R_{d-1}, C_{d, 1:K}]$ and bag action $A_d := [A_{d, 1:k}]$.
The transition probability of $Z_{d + 1}$ given all the history is 
\begin{equation*}
    \begin{split}
        \bbP (Z_{d + 1} | Z_{1:d}, A_{1:d})
        = & \bbP (M_{d, 1:K}, N_{d, 1:K}, E_{d}, R_{d}, C_{d + 1, 1:K} | M_{d - 1, 1:K}, N_{d - 1, 1:K}, E_{d - 1}, R_{d - 1}, C_{d, 1:K}, A_{d, 1:K}) \\
        = & \bbP (M_{d, 1:K}, N_{d, 1:K}, E_{d}, R_{d},  C_{d + 1, 1:K} | E_{d - 1}, R_{d - 1}, C_{d, 1:K}, A_{d, 1:K}) \\
    \end{split}
\end{equation*}
since $Z_d$ is independent of $M_{d - 1, 1:K}$, $N_{d - 1, 1:K}$, and the history $Z_{1:(d-1)}, A_{1:(d-1)}$ given $E_{d - 1}, R_{d - 1}, C_{d, 1:K}, A_{d, 1:K}$.

For a policy $\bpi$ based on the state $S_{d, k}$ in (\ref{equ:state.definition}), the transition probability can be expanded as
\begin{equation} \label{equ:markov.bag}
    \begin{split}
        & \bbP (Z_{d + 1} | Z_{1:d}, A_{1:d}) \\
        = & \brce{\bbP (M_{d, 1:K}, N_{d, 1:K}, E_{d}, R_{d}, C_{d + 1, 1:K}, A_{d, 1:K} | E_{d - 1}, R_{d - 1}, C_{d, 1:K})}
        \brce{\bbP (A_{d, 1:K} | E_{d - 1}, R_{d - 1}, C_{d, 1:K})}^{-1} \\
        = 
        & \Bigg\{ \bbP_{\bpi} (A_{d, 1} | E_{d - 1}, R_{d - 1}, C_{d, 1}) 
        \times \bbP (M_{d, 1} | E_{d - 1}, R_{d - 1}, C_{d, 1}, A_{d, 1}) 
        \times \bbP (N_{d, 1} | E_{d - 1}, A_{d, 1}) \times \dots \\
        & \times \bbP_{\bpi} (A_{d, K} | M_{d, 1:(K - 1)}, N_{d, 1:(K - 1)}, E_{d - 1}, R_{d - 1}, C_{d, K}) \\
        & \times \bbP (M_{d, K} | E_{d - 1}, R_{d - 1}, C_{d, K}, A_{d, K}) 
        \times \bbP (N_{d, K} | E_{d - 1}, A_{d, K}) \\
        & \times \bbP (E_{d} | E_{d - 1}, N_{d, 1:K})
        \times \bbP (R_{d} | E_{d}, M_{d, 1:K}, R_{d - 1})
        \times \prod_{k=1}^K \bbP (C_{d + 1, k}) \Bigg\} \\
        & \brce{\bbP_{\bpi} (A_{d, 1} | E_{d - 1}, R_{d - 1}, C_{d, 1}) \times \dots 
        \times \bbP_{\bpi} (A_{d, K} | A_{d, 1:(K - 1)}, E_{d - 1}, R_{d - 1}, C_{d, 1:K}) }^{-1}.
    \end{split}
\end{equation}


A stationary MDP on the bag level requires that the joint transition probability $\bbP (Z_{d + 1} | Z_d, A_d)$ of the bag state $Z_{d + 1}$ is stationary over time.
However, on the right-hand side of (\ref{equ:markov.bag}), $\bbP_{\bpi} (A_{d, k} | M_{d, 1:(k - 1)}, N_{d, 1:(k - 1)}, E_{d - 1}, R_{d - 1}, C_{d, k})$ in the numerator is different from 
\begin{equation*}
    \begin{split}
        & \bbP_{\bpi} (A_{d, k} | A_{d, 1:(k - 1)}, E_{d - 1}, R_{d - 1}, C_{d, 1:k}) \\
        = & \int \bbP_{\bpi} (A_{d, k} | M_{d, 1:(k - 1)}, N_{d, 1:(k - 1)}, E_{d - 1}, R_{d - 1}, C_{d, k}) \cdot \\
        & \qquad \prod_{t = 1}^{k - 1} \bbP (M_{d, t} | E_{d - 1}, R_{d - 1}, C_{d, t}, A_{d, t}) \bbP (N_{d, t} | E_{d - 1}, A_{d, t}) d M_{d, 1} \dots d M_{d, k - 1} d N_{d, 1} \dots d N_{d, k - 1}
    \end{split}
\end{equation*}
in the denominator when $k > 1$.
Therefore, the joint transition probability $\bbP (Z_{d + 1} | Z_{d}, A_{d})$ actually depends on the policy $\bpi$.
This is a problem of conditioning on descents of $Z_{d}$.

\subsection{Proof of Theorem~\ref{thm:best.state}} \label{sec:proof.best.state}

\begin{proof}
    Since $\{S_{d, k}\}_{\cI}$ is the minimal D-BaSS and the joint distribution of $\bB_d$ is stationary across bags, the state space $\cS_k$ is bag-invariant. 
    Lemmas~\ref{lem:markov.transition} and~\ref{lem:rl.framework} thus guarantees that $\{S_{d, k}\}_{\cI}$ satisfies the definition for a $K$-periodic MDP, which further ensures that $\{U_{d, k}\}_{\cI}$ exists.

    \textbf{Step 1.} 
    For any $\{U_{d, k}\}_{\cI}$ satisfying the precedence condition and the definition for a $K$-periodic MDP, we have 
    $\cV_1^{\cU*} (u_{1, 1}) \le \cV_1^{\cS*} (u_{1, 1})$ for all $u_{1, 1} \in \cU_1$.

    Since $\{U_{d, k}\}_{\cI}$ satisfies the definition for a $K$-periodic MDP, we have that $R_{d, k}, U_{d, k + 1} \indep U_{d, k - 1}, \dots, U_{1, 1}, A_{d, k - 1}, \dots, A_{1, 1} | U_{d, k}, A_{d, k}$.
    
    We will prove the result based on value iteration.
    Consider the two policy classes 
    \begin{align*}
        \Pi^{\cS} & := \{ \bpi^{\cS} = \{\pi^{\cS}_{1:K}\}: \pi^{\cS}_k: \cS_k \mapsto \Delta(\cA_k), \forall k \}, \\
        \Pi^{\cU} & := \{ \bpi^{\cU} = \{\pi^{\cU}_{1:K}\}: \pi^{\cU}_k: \cU_k \mapsto \Delta(\cA_k), \forall k \}.
    \end{align*}
    We start from defining $\bpi^{\cS (0)}$ and $\bpi^{\cU (0)}$ as the same deterministic policies.
    Specifically, let $\pi_k^{\cS (0)} (s_{d, k}) = 0$ for all $k$ and all $s_{d, k} \in \cS_k$ and let $\pi_k^{\cU (0)} (u_{d, k}) = 0$ for all $k$ and all $u_{d, k} \in \cU_k$.
    Then we have 
    \[ \cQ_k^{\cU (0)} (u_{d, k}, a_{d, k}) := \cQ_k^{\bpi^{\cU (0)}} (u_{d, k}, a_{d, k}) = \cQ_k^{\bpi^{\cS (0)}} (u_{d, k}, a_{d, k}) =: \cQ_k^{\cS (0)} (u_{d, k}, a_{d, k}) \]
    for all $k$ and all $u_{d, k}, a_{d, k}$ since all the actions are the same in the whole trajectory for the two policies.
    
    Given the Q-functions in the 0th iteration,
    the first value function at $i = 1$ is defined as
    \begin{align*}
        \cV_k^{\cU (1)} (u_{d, k}) &= \max_{a_{d, k} \in \cA_k} \cQ_k^{\cU (0)} (u_{d, k}, a_{d, k}), \\
        \cV_k^{\cS (1)} (s_{d, k}) &= \max_{a_{d, k} \in \cA_k} \cQ_k^{\cS (0)} (s_{d, k}, a_{d, k}),
    \end{align*}
    for all $k$.
    Then we have
    \begin{equation*}
        \begin{split}
            & \cV_k^{\cU (1)} (u_{d, k}) 
            = \max_{a_{d, k} \in \cA_k} \cQ_k^{\cU (0)} (u_{d, k}, a_{d, k})
            = \max_{a_{d, k} \in \cA_k} \cQ_k^{\cS (0)} (u_{d, k}, a_{d, k}) \\
            = & \max_{a_{d, k} \in \cA_k} \bbE^{\bpi^{\cS (0)}} \brce{ 
            \sum_{l=k}^K \eta_{d, k}^{d, l} R_{d, l} + 
            \sum_{t=d + 1}^{\infty} \sum_{l=1}^K \eta_{d, k}^{t, l} R_{t, l} 
            \middle| U_{d, k} = u_{d, k}, A_{d, k} = a_{d, k} } \\
            = & \max_{a_{d, k} \in \cA_k} \int
            \bbE^{\bpi^{\cS (0)}} \brce{ 
            \sum_{l=k}^K \eta_{d, k}^{d, l} R_{d, l} + 
            \sum_{t=d + 1}^{\infty} \sum_{l=1}^K \eta_{d, k}^{t, l} R_{t, l} 
            \middle| S_{d, k} = s_{d, k}, U_{d, k} = u_{d, k}, A_{d, k} = a_{d, k} } \cdot \\
            & \kern 27em
            p(s_{d, k} | U_{d, k} = u_{d, k}, A_{d, k} = a_{d, k}) 
            d s_{d, k}.
        \end{split}
    \end{equation*}
    Common descendants of $S_{d, k}$ and $A_{d, k}$ like the reward $R_{d, k}$ are colliders and have not been observed yet.
    Since $a_{d, k}$ is a function of only $u_{d, k}$, $U_{d, k}$ is the only parent of $A_{d, k}$ in a causal DAG.
    Therefore, we have that $S_{d, k} \indep A_{d, k} | U_{d, k}$.
    Besides, by the definition of $S_{d, k}$, we have $R_{t, l} \indep \cH_{d, k} | S_{d, k}, A_{d, k}$ for any $(t, l) \ge (d, k)$.
    Since $U_{d, k}$ satisfies the precedence condition of $\cH_{d, k}$, we have $R_{t, l} \indep U_{d, k} | S_{d, k}, A_{d, k}$.
    Therefore, for all $k$,
    \begin{equation} \label{equ:non.dbass.itr1}
        \begin{split}
            & \cV_k^{\cU (1)} (u_{d, k}) \\
            = & \max_{a_{d, k} \in \cA_k} \int
            \bbE^{\bpi^{\cS (0)}} \brce{ 
            \sum_{l=k}^K \eta_{d, k}^{d, l} R_{d, l} + 
            \sum_{t=d + 1}^{\infty} \sum_{l=1}^K \eta_{d, k}^{t, l} R_{t, l} 
            \middle| S_{d, k} = s_{d, k}, A_{d, k} = a_{d, k} }
            p(s_{d, k} | U_{d, k} = u_{d, k}) 
            d s_{d, k} \\
            \le & \int \max_{a_{d, k} \in \cA_k} \cQ_k^{\cS (0)} (s_{d, k}, a_{d, k})
            p(s_{d, k} | U_{d, k} = u_{d, k}) 
            d s_{d, k} \\
            = & \bbE \brce{\max_{a_{d, k} \in \cA_k} \cQ_k^{\cS (0)} (S_{d, k}, a_{d, k}) \middle | U_{d, k} = u_{d, k} } 
            = \bbE \brce{\cV_k^{\cS (1)} (S_{d, k}) \middle | U_{d, k} = u_{d, k} }
            = \cV_k^{\cS (1)} (u_{d, k}),
        \end{split}
    \end{equation}
    where the inequality follows from Jensen's inequality and the last line follows from the definition of $\cQ_k^{\cS (0)}$ and $\cV_k^{\cS (1)}$.
    Remember that Jensen's inequality states $\varphi (\bbE X) \le \bbE \varphi (X)$ for a convex function $\varphi$ and a random vector $X$.
    For the function $\varphi (\bX) = \max_{a \in \cA_k} \{X_a\}$ of the vector $\bX = \{X_a\}_{a \in \cA_k}$, the inequality strictly holds if and only if $P (X_a = \max_{a \in \cA_k} \{X_a\}) < 1$ for all $a$.
    Otherwise, if there exists some $a^*$ s.t. $P (X_{a^*} = \max_{a \in \cA_k} \{X_a\}) = 1$, then $P (X_{a^*} \ge X_a) = 1$ for all $a$ and thus $\bbE X_{a^*} \ge \bbE X_a$.
    In this case, $\max_{a \in \cA_k} \{\bbE X_a\} = \bbE X_{a^*} = \bbE \max_{a \in \cA_k} \{X_a\}$.
    Therefore, the inequality holds strictly in (\ref{equ:non.dbass.itr1}) if and only if $P \{ \cQ_k^{\cS (0)} (S_{d, k}, a) = \max_{a' \in \cA_k} \cQ_k^{\cS (0)} (S_{d, k}, a') | U_{d, k} = u_{d, k} \} < 1$ for all $a \in \cA_k$.
    
    In the second iteration, the value functions are updated as $\cV_K^{\cU (2)} (u_{d, k}) = \max_{a \in \cA_{k}} \cQ_K^{\cU (2)} (u_{d, k}, a)$ and $\cV_k^{\cS (2)} (s_{d, k}) = \max_{a \in \cA_{k}} \cQ_k^{\cS (2)} (s_{d, k}, a)$ with
    \begin{align*}
        \cQ_K^{\cU (2)} (u_{d, K}, a_{d, K}) &= \bbE \brce{R_{d, K} + \gamma_K \cV_{1}^{\cU (1)} (U_{d + 1, 1}) \middle| U_{d, K} = u_{d, K}, A_{d, K} = a_{d, K} }, \\
        \cQ_k^{\cU (2)} (u_{d, k}, a_{d, k}) &= \bbE \brce{R_{d, k} + \gamma_k \cV_{k + 1}^{\cU (2)} (U_{d, k + 1}) \middle| U_{d, k} = u_{d, k}, A_{d, k} = a_{d, k} }, \quad k \in \{1:(K - 1)\} \\
        \cQ_K^{\cS (2)} (s_{d, K}, a_{d, K}) &= \bbE \brce{R_{d, K} + \gamma_K \cV_{1}^{\cS (1)} (S_{d + 1, 1}) \middle| S_{d, K} = s_{d, K}, A_{d, K} = a_{d, K} }, \\
        \cQ_k^{\cS (2)} (s_{d, k}, a_{d, K}) &= \bbE \brce{R_{d, k} + \gamma_k \cV_{k + 1}^{\cS (2)} (S_{d, k + 1}) \middle| S_{d, k} = s_{d, k}, A_{d, k} = a_{d, k} }, \quad k \in \{1:(K - 1)\}.
    \end{align*}
    Notice that for $k = K$,
    \begin{equation*}
        \begin{split}
            & \cV_K^{\cU (2)} (u_{d, K}) 
            = \max_{a_{d, K} \in \cA_{K}} \bbE \brce{R_{d, K} + \gamma_K \cV_{1}^{\cU (1)} (U_{d + 1, 1}) \middle| U_{d, K} = u_{d, K}, A_{d, K} = a_{d, K} } \\
            \le & \max_{a_{d, K} \in \cA_{K}} \bbE \brce{R_{d, K} + \gamma_K \cV_{1}^{\cS (1)} (U_{d + 1, 1}) \middle| U_{d, K} = u_{d, K}, A_{d, K} = a_{d, K} } \\
            = & \max_{a_{d, K} \in \cA_{K}} \int \bbE \brce{R_{d, K} + \gamma_K \cV_{1}^{\cS (1)} (U_{d + 1, 1}) \middle| S_{d, K} = s_{d, K}, U_{d, K} = u_{d, K}, A_{d, K} = a_{d, K} } \cdot \\
            & \kern 27em
            p(s_{d, K} | U_{d, K} = u_{d, K}, A_{d, K} = a_{d, K}) d s_{d, K},
        \end{split}
    \end{equation*}
    where the inequality follows from (\ref{equ:non.dbass.itr1}).
    We have shown that $R_{d, K} \indep U_{d, K} | S_{d, K}, A_{d, K}$ and $S_{d, K} \indep A_{d, K} | U_{d, K}$.
    Using similar arguments as in Step 2 of Section~\ref{sec:proof.markov.transition}, we can show  that $U_{d + 1, 1} \indep U_{d, K} | S_{d, K}, A_{d, K}$.
    Therefore,
    \begin{equation*}
        \begin{split}
            \cV_K^{\cU (2)} (u_{d, K}) 
            \le & \max_{a_{d, K} \in \cA_{K}} \int \bbE \brce{R_{d, K} + \gamma_K \cV_{1}^{\cS (1)} (U_{d + 1, 1}) \middle| S_{d, K} = s_{d, K}, A_{d, K} = a_{d, K} } p(s_{d, K} | U_{d, K} = u_{d, K}) d s_{d, K} \\
            \le & \int \max_{a_{d, K} \in \cA_{K}} \bbE \brce{R_{d, K} + \gamma_K \cV_{1}^{\cS (1)} (U_{d + 1, 1}) \middle| S_{d, K} = s_{d, K}, A_{d, K} = a_{d, K} } p(s_{d, K} | U_{d, K} = u_{d, K}) d s_{d, K},
        \end{split}
    \end{equation*}
    where the second inequality follows from Jensen's inequality.
    Notice that 
    \begin{equation*}
        \begin{split}
            \bbE \brce{R_{d, K} + \gamma_K \cV_{1}^{\cS (1)} (U_{d + 1, 1}) \middle| S_{d, K}, A_{d, K} } 
            = & \bbE \brce{ R_{d, K} + \gamma_K \bbE \brce{ \cV_{1}^{\cS (1)} (S_{d + 1, 1}) \middle| U_{d + 1, 1}, S_{d, K}, A_{d, K} } \middle| S_{d, K}, A_{d, K} } \\
            = & \bbE \brce{ R_{d, K} + \gamma_K \cV_{1}^{\cS (1)} (S_{d + 1, 1}) \middle| S_{d, K}, A_{d, K} }.
        \end{split}
    \end{equation*}
    Therefore, 
    \begin{equation*}
        \begin{split}
            \cV_K^{\cU (2)} (u_{d, K}) 
            \le & \bbE \brce{ \max_{a_{d, K} \in \cA_{K}} \bbE \brce{R_{d, K} + \gamma_K \cV_{1}^{\cS (1)} (S_{d + 1, 1}) \middle| S_{d, K}, A_{d, K} = a_{d, K} } \middle| U_{d, K} = u_{d, K}} \\
            = & \bbE \brce{ \max_{a_{d, K} \in \cA_{K}} \cQ_K^{\cS (2)} (S_{d, K}, a_{d, K} ) \middle| U_{d, K} = u_{d, K}} 
            = \bbE \brce{ \cV_K^{\cS (2)} (S_{d, k}) \middle| U_{d, K} = u_{d, K}}
            = \cV_K^{\cS (2)} (u_{d, K}).
        \end{split}
    \end{equation*}
    Similar as before, the inequality strictly holds if (1) the inequality strictly holds in the last iteration, i.e. $P \{ \cV_{1}^{\cU (1)} (U_{d + 1, 1}) < \cV_{1}^{\cS (1)} (U_{d + 1, 1}) | U_{d, K} = u_{d, K}, A_{d, K} = a_{d, K} \} > 0$ for $a_{d, K} = \argmax_{a \in \cA_K} \cQ_K^{\cU (1)} (u_{d, K}, a)$, or 
    (2) $P \{ \cQ_K^{\cS (2)} (S_{d, K}, a) = \max_{a' \in \cA_K} \cQ_K^{\cS (2)} (S_{d, K}, a') | U_{d, K} = u_{d, K} \} < 1$ for all $a \in \cA_K$.

    For $k = K - 1, \dots, 1$, we also have $\cV_k^{\cU (2)} (u_{d, k}) \le \cV_k^{\cS (2)} (u_{d, k})$.
    The inequality strictly holds if (1) the inequality strictly holds in the last iteration, i.e. $P \{ \cV_{k + 1}^{\cU (2)} (U_{d, k + 1}) < \cV_{k + 1}^{\cS (2)} (U_{d, k + 1}) | U_{d, k} = u_{d, k}, A_{d, k} = a_{d, k} \} > 0$ for $a_{d, k} = \argmax_{a \in \cA_k} \cQ_k^{\cU (2)} (u_{d, k}, a)$, or 
    (2) $P \{ \cQ_k^{\cS (2)} (S_{d, k}, a) = \max_{a' \in \cA_k} \cQ_k^{\cS (2)} (S_{d, k}, a') | U_{d, k} = u_{d, k} \} < 1$ for all $a \in \cA_k$.
    
    Using induction, we can show that 
    \begin{equation*}
        \begin{split}
            \cV_k^{\cU (i)} (u_{d, k}) \le \cV_k^{\cS (i)} (u_{d, k})
        \end{split}
    \end{equation*}
    for all $k$ and all iterations $i = 1, 2, \dots$.
    The convergence of value iteration proved in \citet{hu2014near} can be easily generalized to $K$-periodic MDP with time-dependent discount factor.
    The optimality of value iteration guarantees that
    \begin{align*}
        \lim_{i \to \infty} \cV_k^{\cU (i)} (u_{d, k}) & = \cV_1^{\cU*} (u_{d, k}) \\
        \lim_{i \to \infty} \cV_k^{\cS (i)} (u_{d, k}) & = \cV_1^{\cS*} (u_{d, k}).
    \end{align*}
    Therefore, we have $\cV_1^{\cU*} (u_{1, 1}) \le \cV_1^{\cS*} (u_{1, 1})$.

    We first show the sufficiency of the condition for the inequality to strictly hold. 
    That is, $\cV_k^{\cU*} (u_{d, k}) < \cV_k^{\cS*} (u_{d, k})$ if
    $P \{ \cQ_k^{\cS *} (S_{d, k}, a) = \max_{a' \in \cA_k} \cQ_k^{\cS *} (S_{d, k}, a') | u_{d, k} \} < 1$ for all $a \in \cA_k$.
    To see this, we will first show that if $P \{ \cQ_{k}^{\cS *} (S_{d, k}, a) = \max_{a' \in \cA_k} \cQ_k^{\cS *} (S_{d, k}, a') | u_{d, k} \} < 1$ for some $k$, then there exists an iteration $i$ and time $k'$ s.t. 
    \begin{equation} \label{different.maximizer.all.iters}
        \begin{split}
            P \brce{ \cQ_{k'}^{\cS (i)} (S_{d, k'}, a) = \max_{a' \in \cA_{k'}} \cQ_{k'}^{\cS (i)} (S_{d, k'}, a') \middle| U_{d, k'} = u_{d, k'} } < 1 \text{ for all } a \in \cA_{k'}.
        \end{split}
    \end{equation}
    In fact, if there does not exist such an iteration $i$ and time $k'$ in (\ref{different.maximizer.all.iters}),
    then for all $i$ and $k'$, there exists $a_{k'}^{(i)} \in \cA_k'$ s.t. 
    \[ P \brce{ \cQ_{k'}^{\cS (i)} (S_{d, k'}, a_{k'}^{(i)}) = \max_{a' \in \cA_{k'}} \cQ_{k'}^{\cS (i)} (S_{d, k'}, a') \middle| U_{d, k'} = u_{d, k'} } = 1. \]
    In this case, 
    \[ \cV_{k'}^{\cS (i)} (S_{d, k'}) = \max_{a' \in \cA_{k'}} \cQ_{k'}^{\cS (i)} (S_{d, k'}, a') = \cQ_{k'}^{\cS (i)} (S_{d, k'}, a_{k'}^{(i)}) \]
    almost surely given $U_{d, k'} = u_{d, k'}$.
    Taking $i \to \infty$ and we have $\cV_{k'}^{\cS *} (S_{d, k'}) = \cQ_{k'}^{\cS *} (S_{d, k'}, a_{k'})$ for all $k'$ and some $a_{k'}$ almost surely.
    However, $P \{ \cQ_k^{\cS *} (S_{d, k}, a) = \max_{a' \in \cA_k} \cQ_k^{\cS *} (S_{d, k}, a') | u_{d, k} \} < 1$ is saying that $P \{ \cQ_k^{\cS *} (S_{d, k}, a) = \cV_{k}^{\cS *} (S_{d, k}) | u_{d, k} \} < 1$.
    Then when an inequality holds for some iteration $i$ in (\ref{different.maximizer.all.iters}),
    then the inequality $\cV_k^{\cU (i)} (u_{d, k}) < \cV_k^{\cS (i)} (u_{d, k})$ strictly holds for this iteration $i$.
    Consequently, the limitation $\lim_{i \to \infty} \cV_1^{\cU (i)} (u_{d, k}) < \lim_{i \to \infty} \cV_k^{\cS (i)} (u_{d, k})$ and thus $\cV_k^{\cU*} (u_{k, k}) < \cV_k^{\cS*} (u_{k, k})$ strictly holds.

    Then we show the necessity of the condition.
    That is, if $\cV_k^{\cU*} (u_{k, k}) < \cV_k^{\cS*} (u_{k, k})$, then we have $P \{ \cQ_k^{\cS *} (S_{d, k}, a) = \max_{a' \in \cA_k} \cQ_k^{\cS *} (S_{d, k}, a') | u_{d, k} \} < 1$ for all $a \in \cA_k$.
    If $P \{ \cQ_k^{\cS *} (S_{d, k}, a_k^* (u_{d, k})) = \max_{a' \in \cA_k} \cQ_k^{\cS *} (S_{d, k}, a') | u_{d, k} \} = 1$ for some $a_k^* (u_{d, k}) \in \cA_k$,
    then
    \begin{equation*}
        \begin{split}
            \cV_k^{\cS *} (u_{d, k})
            = & \bbE \brce{ \max_{a' \in \cA_k} \cQ_k^{\cS *} (S_{d, k}, a') \middle| U_{d, k} = u_{d, k}} \\
            = & \bbE \brce{ \cQ_k^{\cS *} (S_{d, k}, a_k^* (u_{d, k})) \middle| U_{d, k} = u_{d, k}} \\
            = & \bbE \brce{ \bbE \brce{R_{d, k} (a_k^* (u_{d, k})) + \gamma_k \cV_{k + 1}^{\cS *} (S_{d, k + 1} (a_k^* (u_{d, k}))) \middle| S_{d, k} } \middle| U_{d, k} = u_{d, k}} \\
            = & \bbE \brce{ R_{d, k} (a_k^* (u_{d, k})) + \gamma_k \cV_{k + 1}^{\cS *} (S_{d, k + 1} (a_k^* (u_{d, k}))) \middle| U_{d, k} = u_{d, k}} \\
            \le & \max_{a_{k} \in \cA_{k}} \bbE \brce{ R_{d, k} + \gamma_k \cV_{k + 1}^{\cS *} (S_{d, k + 1} ) \middle| U_{d, k} = u_{d, k}, A_{d, k} = a_{k}} 
            = \cV_k^{\cU *} (u_{d, k})
        \end{split}
    \end{equation*}
    for $k < K$. Similarly, $\cV_K^{\cS *} (u_{d, K}) \le \cV_K^{\cU *} (u_{d, K})$ for $k = K$.
    We have proved that $\cV_k^{\cU*} (u_{d, k}) \le \cV_k^{\cS*} (u_{d, k})$, so we have $\cV_k^{\cS *} (u_{d, k}) = \cV_k^{\cU *} (u_{d, k})$.
    However, this violates the condition that $\cV_k^{\cU*} (u_{k, k}) < \cV_k^{\cS*} (u_{k, k})$.
    Therefore, we must have $P \{ \cQ_k^{\cS *} (S_{d, k}, a) = \max_{a' \in \cA_k} \cQ_k^{\cS *} (S_{d, k}, a') | u_{d, k} \} < 1$ for all $a \in \cA_k$.

    \textbf{Step 2.} 
    Show that if $\{U_{d, k}\}_{\cI}$ is a D-BaSS, then $\cV_1^{\cU*} (u_{1, 1}) = \cV_1^{\cS*} (u_{1, 1})$ for all $u_{1, 1} \in \cU_1$.

    By definition, there exists a function $f_{d, k} (\cdot)$ s.t. $f_{d, k} (U_{d, k}) = S_{d, k}$ for all $(d, k) \in \cI$.
    Then we can define a policy $\bpi^{\cU} := \{\pi^{\cU}_{1:K}\}$, where $\pi^{\cU}_k (u_{d, k}) = \pi^{\cS*}_k (f_{d, k} (u_{d, k}))$ for all $u_{d, k} \in \cU_k$.
    By Bellman equation, 
    \begin{equation*}
        \begin{split}
            \cV_1^{\bpi^{\cU}} (u_{1, 1})
            = & \bbE^{\bpi^{\cU}} \brce{R_{1, 1} + \gamma_1 \cV_2^{\bpi^{\cU}} (U_{1, 2}) \middle| U_{1, 1} = u_{1, 1}} \\
            = & \sum_{a_{1, 1} \in \cA_1} \pi^{\cU}_1 (a_{1, 1}; u_{1, 1}) \bbE \brce{R_{1, 1} + \gamma_1 \cV_2^{\bpi^{\cU}} (U_{1, 2}) \middle| U_{1, 1} = u_{1, 1}, A_{1, 1} = a_{1, 1}}.
        \end{split}
    \end{equation*}
    Using similar arguments as in Step 1 and Fubini's theorem,
    \begin{equation*}
        \begin{split}
            \cV_1^{\bpi^{\cU}} (u_{1, 1})
            = & \sum_{a_{1, 1} \in \cA_1} \pi^{\cU}_1 (a_{1, 1}; u_{1, 1}) \int \bbE \brce{R_{1, 1} + \gamma_1 \cV_2^{\bpi^{\cU}} (U_{1, 2}) \middle| s_{1, 1}, u_{1, 1}, a_{1, 1} } p(s_{1, 1} | U_{1, 1} = u_{1, 1}, A_{1, 1} = a_{1, 1}) d s_{1, 1} \\
            = & \int \sum_{a_{1, 1} \in \cA_1} \pi^{\cU}_1 (a_{1, 1}; u_{1, 1}) \bbE \brce{R_{1, 1} + \gamma_1 \cV_2^{\bpi^{\cU}} (U_{1, 2}) \middle| s_{1, 1}, a_{1, 1} } p(s_{1, 1} | U_{1, 1} = u_{1, 1}) d s_{1, 1}.
        \end{split}
    \end{equation*}
    According to the definition of $\bpi^{\cU}$,
    \begin{equation*}
        \begin{split}
            \cV_1^{\bpi^{\cU}} (u_{1, 1})
            = & \bbE \brce{ \sum_{a_{1, 1} \in \cA_1} \pi^{\cS^*}_1 (a_{1, 1}; f_{1, 1} (u_{1, 1})) \bbE \brce{R_{1, 1} + \gamma_1 \cV_2^{\bpi^{\cU}} (U_{1, 2}) \middle| S_{1, 1}, a_{1, 1} } \middle| U_{1, 1} = u_{1, 1}} \\
            = & \bbE \brce{ \bbE^{\bpi^{\cS*}} \brce{R_{1, 1} + \gamma_1 \cV_2^{\bpi^{\cU}} (U_{1, 2}) \middle| S_{1, 1} } \middle| U_{1, 1} = u_{1, 1}} \\
            = & \bbE^{\bpi^{\cS*}} \brce{R_{1, 1} + \gamma_1 \cV_2^{\bpi^{\cU}} (U_{1, 2}) \middle| U_{1, 1} = u_{1, 1}}.
        \end{split}
    \end{equation*}
    Using induction we can show that 
    \begin{equation*}
        \begin{split}
            \cV_1^{\bpi^{\cU}} (u_{1, 1})
            = \bbE^{\bpi^{\cS*}} \brce{R_{1, 1} + \gamma_1 R_{1, 2} + \dots \middle| U_{1, 1} = u_{1, 1}}
            = \cV_1^{\bpi^{\cS^*}} (u_{1, 1})
            = \cV_1^{\cS*} (u_{1, 1}).
        \end{split}
    \end{equation*}
    Hence, we have $\cV_1^{\cS*} (u_{1, 1}) = \cV_1^{\bpi^{\cU}} (u_{1, 1}) \le \cV_1^{\cU*} (u_{1, 1})$.
    On the other hand, Step 1 has shown that $\cV_1^{\cU*} (u_{1, 1}) \le \cV_1^{\cS*} (u_{1, 1})$.
    Therefore, $\cV_1^{\cU*} (u_{1, 1}) = \cV_1^{\cS*} (u_{1, 1})$.
\end{proof}

\subsection{Proof of Corollary~\ref{lem:state.construction.HS}} \label{sec:proof.example.dag}

\begin{proof}
    The state $S_{d, k}$ is a D-BaSS since $S_{d, k} \in \cH_{d, k}$, and $R_{d:\infty} \indep \cH_{d, k} | S_{d, k}, A_{d, k}$.
    This can be seen from the conditional distribution of $R_d$ given all the observed data before time $k$ in bag $d$
    \begin{equation} \label{equ:Rd.conditional.distribution}
        \bbP (R_d | S_{d, k}, A_{d, k}, \cH_{d, k})
        = \bbP (R_d | E_{d - 1}, R_{d - 1}, M_{d, 1:(k-1)}, N_{d, 1:(k - 1)}, C_{d, k}, A_{d, k}).
    \end{equation}
    The state $S^{\prime}_{d, k}$ is not a D-BaSS since $N_{d, 1:(k - 1)} \in \cH_{d, k}$, but $R_{d, k} \nindep N_{d, 1:(k - 1)} | S^{\prime}_{d, k}, A_{d, k}$.
    Similarly, $S^{\prime \prime}_{d, k}$ is also not a D-BaSS.
    
    Assume the minimal D-BaSS is $\{B_{d, k}\}_{\cI}$ for some vector $B_{d, k}$, where $B_{d, k} \in \cB_k$.
    Theorem~\ref{thm:best.state} implies that $\cV_k^{\cS^{\prime}*} (s^{\prime}_{d, k}) < \cV_k^{\cB*} (s^{\prime}_{d, k})$ if and only if $P \{ \cQ_k (B_{d, k}, a) = \max_{a' \in \cA_k} \cQ_k (B_{d, k}, a') | s^{\prime}_{d, k} \} < 1$ for all $a \in \cA_k$.    
    On the other hand, Theorem~\ref{thm:best.state} also implies that $\cV_k^{\cS*} (s^{\prime}_{1, 1}) = \cV_k^{\cB*} (s^{\prime}_{d, k})$ since $\{S_{d, k}\}_{\cI}$ is a D-BaSS.
    Besides, since there exists a function $f_{d, k} (\cdot)$ s.t. $f_{d, k} (S_{d, k}) = B_{d, k}$, by defining $\pi^{\cS}_k (s_{d, k}) = \pi^{\cB*}_k (f_{d, k} (s_{d, k}))$, we have that $\pi^{\cS}_k$ is the optimal policy according to Theorem~\ref{thm:best.state}.
    Thus we can denote $\pi^{\cS}_k$ as $\pi^{\cS*}_k$.
    Besides, $P \{ \cQ_k (S_{d, k}, a) = \max_{a' \in \cA_k} \cQ_k (S_{d, k}, a') | s^{\prime}_{d, k} \} < 1$ for all $a \in \cA_k$.

    The results for $\cS^{\prime}_k$ and $\cS^{\prime \prime}_k$ can be proved similarly, since we have $R_{d:\infty} \indep S^{\prime \prime}_{d, k} | S^{\prime}_{d, k}, A_{d, k}$ and $S^{\prime \prime}_{d, k + 1} \indep S^{\prime \prime}_{d, k} | S^{\prime}_{d, k}, A_{d, k}$.
    This is equivalent to assuming that $N_{d, k}$ is not observed.
\end{proof}

\subsection{Misspecified DAG Assumptions} \label{sec:misspecified.assumption}

Figure~\ref{fig:dag} assumes that $C_{d, 1:K}$ are exogenous variables and that $E_d$ is conditionally independent of $R_{d - 1}$ given $E_{d - 1}, A_{d, 1:K}$.
However, Algorithm~\ref{alg:brlsvi} is robust to these assumptions even if they are violated.
The policy is updated using a model-free method, stationary RLSVI, which only depend on the state space through the model of the Q-functions in (\ref{equ:q.function}), but does not depend on the dynamics of variables.

\begin{lem}
    The state vector $S_{d, k}$ remain unchanged when $C_{d, 1:K}$ is directly dependent on $R_{d - 1}, E_{d - 1}$ and when $E_d$ is directly dependent on $R_{d - 1}$.
\end{lem}

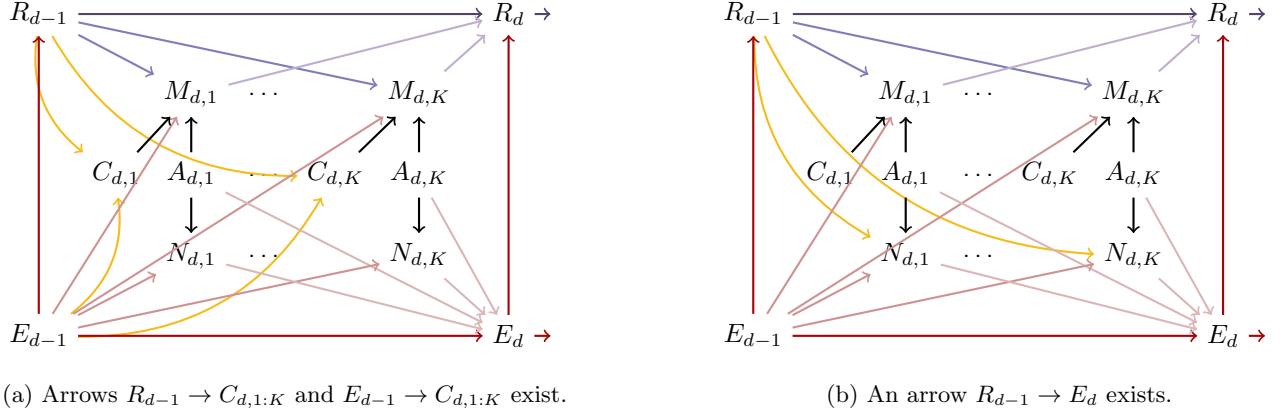
\begin{figure}[t]
    \centering
    \begin{subfigure}{0.46\textwidth}
        \centering
        \begin{tikzpicture}[->, thick, main/.style={font=\sffamily}]
        \matrix [column sep=0.08cm, row sep=0.5cm] {
            \node[main] (R0) {$R_{d-1}$};
            & & & & & & & & & 
            & \node[main] (RK) {$R_{d}$}; 
            & & & \node[main] (R6) {}; \\
            & & \node[main] (M1) {$M_{d, 1}$};
            & & \node[main] (M2) {\dots};
            & & \node[main] (MK) {$M_{d, K}$};
            & \\
            & \node[main] (C1) {$C_{d, 1}$}; & \node[main] (A1) {$A_{d, 1}$};
            & & \node[main] (A2) {\dots};
            & \node[main] (CK) {$C_{d, K}$}; & \node[main] (AK) {$A_{d, K}$};
            & \\
            & & \node[main] (N1) {$N_{d, 1}$};
            & & \node[main] (N2) {\dots};
            & & \node[main] (NK) {$N_{d, K}$};
            & \\
            \node[main] (E0) {$E_{d-1}$};
            & & & & & & & & & 
            & \node[main] (EK) {$E_{d}$}; 
            & & & \node[main] (E6) {}; \\
        };
        \path[every node/.style={font=\sffamily}]
            (R0) edge [bend right, color8] node [right] {} (C1)
            (R0) edge [bend right, color8] node [right] {} (CK)
            (E0) edge [bend right, color8] node [right] {} (C1)
            (E0) edge [bend right, color8] node [right] {} (CK)
            (R0) edge [color1] node [right] {} (RK)
            (RK) edge [color1] node [right] {} (R6)
            (E0) edge [color6] node [right] {} (EK)
            (EK) edge [color6] node [right] {} (E6)
            (E0) edge [color6] node [right] {} (R0)
            (EK) edge [color6] node [right] {} (RK)
            (R0) edge [color2] node [right] {} (M1)
            (R0) edge [color2] node [right] {} (MK)
            (M1) edge [color3] node [right] {} (RK)
            (MK) edge [color3] node [right] {} (RK)
            (E0) edge [color4] node [right] {} (M1)
            (E0) edge [color4] node [right] {} (MK)
            (A1) edge [color5] node [right] {} (EK)
            (AK) edge [color5] node [right] {} (EK)
            (E0) edge [color4] node [right] {} (N1)
            (E0) edge [color4] node [right] {} (NK)
            (N1) edge [color5] node [right] {} (EK)
            (NK) edge [color5] node [right] {} (EK)
            (A1) edge node [right] {} (M1)
            (AK) edge node [right] {} (MK)
            (C1) edge node [right] {} (M1)
            (CK) edge node [right] {} (MK)
            (A1) edge node [right] {} (N1)
            (AK) edge node [right] {} (NK);
        \end{tikzpicture}
        \caption{Arrows $R_{d - 1} \to C_{d, 1:K}$ and $E_{d - 1} \to C_{d, 1:K}$ exist.}
        \label{subfig:dag.RC}
    \end{subfigure}
    \hfill
    \begin{subfigure}{0.46\textwidth}
        \centering
        \begin{tikzpicture}[->, thick, main/.style={font=\sffamily}]
        \matrix [column sep=0.08cm, row sep=0.5cm] {
            \node[main] (R0) {$R_{d-1}$};
            & & & & & & & & & 
            & \node[main] (RK) {$R_{d}$}; 
            & & & \node[main] (R6) {}; \\
            & & \node[main] (M1) {$M_{d, 1}$};
            & & \node[main] (M2) {\dots};
            & & \node[main] (MK) {$M_{d, K}$};
            & \\
            & \node[main] (C1) {$C_{d, 1}$}; & \node[main] (A1) {$A_{d, 1}$};
            & & \node[main] (A2) {\dots};
            & \node[main] (CK) {$C_{d, K}$}; & \node[main] (AK) {$A_{d, K}$};
            & \\
            & & \node[main] (N1) {$N_{d, 1}$};
            & & \node[main] (N2) {\dots};
            & & \node[main] (NK) {$N_{d, K}$};
            & \\
            \node[main] (E0) {$E_{d-1}$};
            & & & & & & & & & 
            & \node[main] (EK) {$E_{d}$}; 
            & & & \node[main] (E6) {}; \\
        };
        \path[every node/.style={font=\sffamily}]
            (R0) edge [bend right, color8] node [right] {} (N1)
            (R0) edge [bend right, color8] node [right] {} (NK)
            (R0) edge [color1] node [right] {} (RK)
            (RK) edge [color1] node [right] {} (R6)
            (E0) edge [color6] node [right] {} (EK)
            (EK) edge [color6] node [right] {} (E6)
            (E0) edge [color6] node [right] {} (R0)
            (EK) edge [color6] node [right] {} (RK)
            (R0) edge [color2] node [right] {} (M1)
            (R0) edge [color2] node [right] {} (MK)
            (M1) edge [color3] node [right] {} (RK)
            (MK) edge [color3] node [right] {} (RK)
            (E0) edge [color4] node [right] {} (M1)
            (E0) edge [color4] node [right] {} (MK)
            (A1) edge [color5] node [right] {} (EK)
            (AK) edge [color5] node [right] {} (EK)
            (E0) edge [color4] node [right] {} (N1)
            (E0) edge [color4] node [right] {} (NK)
            (N1) edge [color5] node [right] {} (EK)
            (NK) edge [color5] node [right] {} (EK)
            (A1) edge node [right] {} (M1)
            (AK) edge node [right] {} (MK)
            (C1) edge node [right] {} (M1)
            (CK) edge node [right] {} (MK)
            (A1) edge node [right] {} (N1)
            (AK) edge node [right] {} (NK);
        \end{tikzpicture}
        \caption{Arrows $R_{d - 1} \to N_{d, 1:K}$ exist.}
        \label{subfig:dag.RE}
    \end{subfigure}
    \caption{Causal DAG when Figure~\ref{fig:dag} is misspecified. The arrows pointing to the actions $A_{d, 1:K}$ are omitted.}
    \label{fig:dag.misspecified}
\end{figure}

\begin{proof}
    If arrows $R_{d - 1} \to C_{d, 1:K}$ and $E_{d - 1} \to C_{d, 1:K}$ exist as in Figure~\ref{subfig:dag.RC}, $C_{d, k}$ becomes a collider of $E_{d - 1}$ and $R_{d - 1}$. 
    However, the conditional independence in (\ref{equ:Rd.conditional.distribution}) still holds.
    Similarly, if arrows $R_{d - 1} \to N_{d, 1:K}$ exist as in Figure~\ref{subfig:dag.RE}, $N_{d, k}$ becomes a collider of $E_{d - 1}$, $R_{d - 1}$ and $A_{d, k}$, which does not affect the conditional independence in (\ref{equ:Rd.conditional.distribution}).
\end{proof}

\section{IMPLEMENTATION DETAILS} \label{sec:algorithm.detail}

Algorithm~\ref{alg:brlsvi} focuses on the case when each Q-function is learned separately.
In Algorithm~\ref{alg:brlsvi2}, we present the implementation details from Section~\ref{sec:experiment} where the main effect is learned using pooled data across all decision times.
For $t \in \{1:(d - 1)\}$, let
\begin{equation} \label{equ:brlsvi2.variables}
\begin{split}
    \bX_{t, k} &= \phi (\widetilde{S}^{\prime}_{t, k}, A_{t, k}), \\
    Y_{t, k} &= 
    \begin{cases}
        \max_{a \in \cA_{k + 1}} \phi (\widetilde{S}^{\prime}_{t, k + 1}, a)^T \widetilde{\bbeta}_{d - 1} 
        & \text{if } k < K, \\
        R_{t} + \bar{\gamma} \max_{a \in \cA_1} \phi (\widetilde{S}^{\prime}_{t + 1, 1}, a)^T \widetilde{\bbeta}_{d - 1}
        & \text{if } k = K,
    \end{cases}
\end{split}
\end{equation}
and construct
\begin{align*}
    \Xb_{d} &= [\bX_{1:(d - 1), 1}, \dots, \bX_{1:(d - 1), K}]^T, \\
    \Yb_{d} &= [Y_{1:(d - 1), 1}, \dots, Y_{1:(d - 1), K}]^T.
\end{align*}
Fit a Bayesian linear regression for the optimal Q-functions.
Then the posterior of $\bbeta_{d}$ is given by a normal distribution $N(\bmu_{d}, \Sigma_{d})$, where
\begin{equation} \label{equ:update.beta.normal2}
    \begin{split}
        \bSigma_{d} &= \prth{\Xb_{d}^T \Xb_{d} / \sigma^2 + \lambda_{d} \Ib}^{-1}, \\
        \bmu_{d} &= \bSigma_{d} (\Xb_{d}^T \Yb_{d} / \sigma^2),
    \end{split}
\end{equation}
and $\Ib$ is the identity matrix.


\begin{algorithm}[tb]
    \caption{Bagged RLSVI with Pooling Across Decision Times}
    \label{alg:brlsvi2}
    \textbf{Input}: Hyperparameters $L, \lambda_{d}, \sigma^2$.
    
    \begin{algorithmic}[1] 
    \STATE Warm-up: Take actions $A_{d, k} \sim \text{Bernoulli} (0.5)$ in bag $d \in \{1:L\}$ for $k \in \{1:K\}$.
    \FOR{$d \ge L + 1$}
        \FOR{$k = K, \dots, 1$}
            \STATE Construct $\bX_{1:(d - 1), k}$ and $Y_{1:(d - 1), k}$ with $\widetilde{\bbeta}_{d - 1}$ using (\ref{equ:brlsvi2.variables}).
        \ENDFOR
        \STATE Obtain $\bmu_{d}, \Sigma_{d}$ using (\ref{equ:update.beta.normal2}). Draw $\widetilde{\bbeta}_{d} \sim N(\bmu_{d}, \Sigma_{d})$. 
        \FOR{$k = 1, \dots, K$}
            \STATE Observe $H_{d, k}$ and construct $\widetilde{S}^{\prime}_{d, k}$. 
            \STATE Take $A_{d, k} = \argmax_{a \in \cA_k} \phi (\widetilde{S}^{\prime}_{d, k}, a)^T \widetilde{\bbeta}_{d}$.
        \ENDFOR
        \STATE Observe $M_{d, K}$, $E_d$, and $R_d$.
    \ENDFOR
    \end{algorithmic}
\end{algorithm}

In standard RL algorithms, we need to wait until the first state vector $S_{d, 1}$ on day $d$ is observed before updating the parameter $\widetilde{\bbeta}_{d}$.
However, in our problem, since $S_{d, 1}$ contains only three variables $E_{d - 1}, R_{d - 1}, C_{d, 1}$ and $C_{d, 1}$ is exogenous, we can randomly sample $\widetilde{C}_{d, 1}$ from $C_{1:(d - 1), 1:K}$ and update the policy at the end of day $d - 1$.

The discount factor is taken as $\bar{\gamma}= 0.99$ for BRLSVI, SRLSVI, and finite-horizon RLSVI.
The tuning parameter $L$ is set to 7 for all algorithms to mimic a real clinical trial scenario, where a user receives randomized notifications over the first week.

The hyperparameters $\lambda_d$ and $\sigma^2$ are set to be identical across all testbed variants for each algorithm.
Specifically, since $\lambda_d \times \sigma^2$ is equivalent to the tuning parameter of an $L_2$ penalty, we fix the value of $\tau_d = \lambda_d \times \sigma^2$.
The values of $\sigma^2$ and $\tau_d$ are set as $\sigma^2 = 0.005$, $\tau_d = 5d$ for Algorithm~\ref{alg:brlsvi2}, $\sigma^2 = 1$, $\tau_d = 10$ for Algorithm~\ref{alg:srlsvi}, $\sigma^2 = 0.005$, $\tau_d = 2d$ for Algorithm~\ref{alg:rlsvi}, and $\sigma^2 = 0.2$, $\tau_d = 2$ for Algorithm~\ref{alg:ts}.
Notice that $\tau_d$ is a function of bag $d$ for Algorithms~\ref{alg:brlsvi2} and~\ref{alg:rlsvi}.
This parameter is set to be adaptive since $\Xb_{d}^T \Xb_{d}$ is increasing as the number of bags increases, while $\lambda_{d} \Ib$ remains constant.
This becomes problematic when the actions in the history are highly imbalanced, meaning that either most actions are zeros or most are ones, as this can lead $\Xb_{d}^T \Xb_{d}$ to be ill-conditioned.
This is consistent with the implementation of regularized linear regression in the \texttt{statsmodels} package \citep{statsmodels}, which minimizes the squared error divided by the sample size.
Experiments have shown that a constant $\tau_d$ is better for Algorithm~\ref{alg:srlsvi}.
Although these hyperparameters may not be optimal for each specific testbed variant, our objective is to demonstrate that there exists a set of hyperparameters that is near-optimal across all testbed variants.
This approach ensures that, regardless of the specific environment which might represent a user in the new trial, such a set of hyperparameters will be effective.

When running on a cloud server, BRLSVI takes 129 seconds on average for one CPU core to run an iteration for 42 users sequentially, while SRLSVI takes 153 seconds, finite-horizon RLSVI takes 159 seconds, RAND takes 80 seconds, and TS takes 96 seconds.

\section{ADDITIONAL SIMULATION RESULTS} \label{sec:additional.simulation}

In this section, we compare the states 
\begin{equation} \label{equ:state.comparison}
    \begin{split}
        S^{\prime}_{d, k} &= [E_{d - 1}, R_{d - 1}, M_{d, 1:(k-1)}, A_{d, 1:(k-1)}, C_{d, k}], \\
        S^{\prime \prime}_{d, k} &= [E_{d - 1}, R_{d - 1}, A_{d, 1:(k-1)}], \\
        S^{\prime \prime \prime}_{d, k} &= [E_{d - 1}, R_{d - 1}, C_{d, 1:(k-1)}, A_{d, 1:(k-1)}, C_{d, k}], 
    \end{split}
\end{equation}
in the experiment setting in Section~\ref{sec:experiment}.
According to Theorem~\ref{thm:best.state}, a D-BaSS will yield a higher optimal value function than another state if no single action is dominant for all values of the D-BaSS given the other state.
Therefore, to demonstrate the advantage of using mediators $M_{d, 1:(k-1)}$ in $S^{\prime}_{d, k}$ compared to using $C_{d, 1:(k-1)}$ in $S^{\prime \prime \prime}_{d, k}$, we create another testbed variant where there is an interaction effect between $M_{d, j}$ and $A_{d, k}$ on $M_{d, k}$ for $j < k$ (see Appendix~\ref{sec:interaction.effect.MA}).
The basis functions are defined as
\begin{align*}
    \begin{split}
        \phi (S^{\prime}_{d, k}, A_{d, k})^T
        := & [
        1, k,
        E_{d - 1}, k E_{d - 1}, 
        R_{d - 1}, k R_{d - 1},
        \widetilde{M}_{d, 1:(K - 1)},
        \widetilde{A}_{d, 1:(K - 1)},
        C_{d, k}
        ] 
        \\
        & \frown \bbone (k = 1) A_{d, k} [1, E_{d - 1}, R_{d - 1}, C_{d, 1}]
        \frown \bbone (k = 2) A_{d, k} [1, E_{d - 1}, R_{d - 1}, C_{d, 2}, M_{d, 1}, A_{d, 1}]
        \frown \dots \\
        & \frown \bbone (k = K) A_{d, k} [1, E_{d - 1}, R_{d - 1}, C_{d, K}, M_{d, 1:(K-1)}, A_{d, 1:(K-1)}],
    \end{split}
    \\
    \begin{split}
        \phi (S^{\prime \prime}_{d, k}, A_{d, k})^T
        := & [
        1, k,
        E_{d - 1}, k E_{d - 1}, 
        R_{d - 1}, k R_{d - 1},
        \widetilde{A}_{d, 1:(K - 1)}
        ] 
        \\
        & \frown \bbone (k = 1) A_{d, k} [1, E_{d - 1}, R_{d - 1}]
        \frown \bbone (k = 2) A_{d, k} [1, E_{d - 1}, R_{d - 1}, A_{d, 1}]
        \frown \dots \\
        & \frown \bbone (k = K) A_{d, k} [1, E_{d - 1}, R_{d - 1}, A_{d, 1:(K-1)}],
    \end{split}
    \\
    \begin{split}
        \phi (S^{\prime \prime \prime}_{d, k}, A_{d, k})^T
        := & [
        1, k,
        E_{d - 1}, k E_{d - 1}, 
        R_{d - 1}, k R_{d - 1},
        \widetilde{C}_{d, 1:(K - 1)},
        \widetilde{A}_{d, 1:(K - 1)},
        C_{d, k}
        ] 
        \\
        & \frown \bbone (k = 1) A_{d, k} [1, E_{d - 1}, R_{d - 1}, C_{d, 1}]
        \frown \bbone (k = 2) A_{d, k} [1, E_{d - 1}, R_{d - 1}, C_{d, 2}, C_{d, 1}, A_{d, 1}]
        \frown \dots \\
        & \frown \bbone (k = K) A_{d, k} [1, E_{d - 1}, R_{d - 1}, C_{d, K}, C_{d, 1:(K-1)}, A_{d, 1:(K-1)}],
    \end{split}
\end{align*}
where $\widetilde{M}_{d, j} = \bbone (j < k) M_{d, j}$, $\widetilde{A}_{d, j} = \bbone (j < k) A_{d, j}$, and $\widetilde{C}_{d, j} = \bbone (j < k) C_{d, j}$.
The tuning parameters are set at $\sigma^2 = 0.005$, $\tau_d = 5d$.

\begin{figure}[t]
    \centering
    \includegraphics[width=0.35\textwidth]{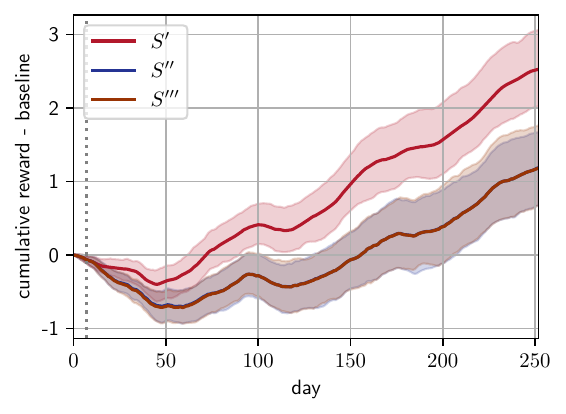}
    \caption{
    The average cumulative rewards of BRLSVI learned with states $S^{\prime}$, $S^{\prime \prime}$, or $S^{\prime \prime \prime}$ in (\ref{equ:state.comparison}), each subtracting the average cumulative rewards of the zero policy.
    The vertical dotted line represents the end of the warm-up period.
    }
    \label{fig:simulation.compare.states}
\end{figure}

The average cumulative reward for BRLSVI based on $S^{\prime}$, $S^{\prime \prime}$, or $S^{\prime \prime \prime}$ subtracting the average cumulative reward of the zero policy is plotted in Figure~\ref{fig:simulation.compare.states}.
We can see that $S^{\prime}_{d, k}$ yields a higher cumulative reward than $S^{\prime \prime \prime}_{d, k}$, since $M_{d, 1:(k-1)}$ is able to explain part of the noise on the causal path from $C_{d, 1:(k-1)}$ to the future rewards.
If the mediators $M_{d, 1:(k-1)}$ are not observed, we can only choose the action $A_{d, k}$ based on the conditional distribution of $S^{\prime}_{d, k}$ given $S^{\prime \prime \prime}_{d, k}$, rather than the observed values of $S^{\prime}_{d, k}$.
On the other hand, $S^{\prime \prime}_{d, k}$ and $S^{\prime \prime \prime}_{d, k}$ have similar cumulative rewards.
Although $S^{\prime \prime}_{d, k}$ contains less information than $S^{\prime \prime \prime}_{d, k}$, it has a smaller dimension.
Therefore, even though the true optimal value function based on $S^{\prime \prime}_{d, k}$ is smaller than that of $S^{\prime \prime \prime}_{d, k}$, learning a policy with $S^{\prime \prime}_{d, k}$ has a higher sample efficiency.

\section{BASELINE ALGORITHMS} \label{sec:baseline}

In this section, we specify the implementation details of the baseline algorithms discussed in Section~\ref{sec:simulation.results}.

\subsection{Stationary RLSVI} \label{sec:srlsvi}

The $K$ actions within a bag can be treated as a $K$-dimensional action, so that algorithms designed for a stationary MDP can be utilized at the bag level.
However, since we need to choose the $K$ actions simultaneously, variables at each decision time like $C_{d, k}, M_{d, k - 1}, A_{d, k - 1}$ cannot be used for decision making.
The variables available at the beginning of a bag are $E_{d - 1}$ and $R_{d - 1}$.
Therefore, the bag state is defined as $S_d = [E_{d - 1}, R_{d - 1}] \in \cS$.
The bag action is defined as $A_d = [A_{1:K}]$, which has $2^K$ different values when $\cA_k = \{0, 1\}$ for all $k$.
Denote the action space as $\cA := \cA_1 \times \dots \times \cA_K$.
For a policy $\pi: \cA \mapsto \cS$, the Q-function is
\begin{equation*}
    \cQ^{\pi} (s_{d}, a_{d}) = \bbE^{\bpi} \brce{\sum_{t=d}^{\infty} \bar{\gamma}^{t - 1} R_t \middle| S_{d} = s_{d}, A_{d} = a_{d}}
\end{equation*}
for $s_d \in \cS, a_d \in \cA$.
The Bellman optimality equation is 
\begin{equation*}
    \begin{split}
        \cQ^* (s_{d}, a_{d}) 
        & = \bbE \{ R_{d} + \bar{\gamma} \max_{a_{d + 1} \in \cA}
        \cQ^* (S_{d + 1}, a_{d + 1}) | 
        S_{d} = s_{d}, A_{d} = a_{d} \}.
    \end{split}
\end{equation*}
The optimal Q-function is modeled as 
\begin{equation*}
    \cQ (s_d, a_d) = \phi (s_d, a_d)^T \bbeta.
\end{equation*}
The basis function is 
\begin{equation} \label{equ:srlsvi.basis}
    \phi (S_{d}, A_{d})^T
    := [1, E_{d - 1}, R_{d - 1}] \frown \bbone(A_d = a_1) [1, E_{d - 1}, R_{d - 1}] \frown \dots \frown \bbone(A_d = a_{2^K}) [1, E_{d - 1}, R_{d - 1}],
\end{equation}
where $a_1, \dots, a_{2^K}$ traverse through all possible actions in $\cA$.
Based on the Bellman optimality equation, Algorithms~\ref{alg:srlsvi} updates the estimate of $\bbeta$ sequentially.

\begin{algorithm}[!htbp]
    \caption{Stationary RLSVI.}
    \label{alg:srlsvi}
    \textbf{Input}: Hyperparameters $L, \lambda_{d}, \sigma^2$.
    
    \begin{algorithmic}[1] 
    \STATE Warm-up: Randomly take actions $A_{d, k} \sim \text{Bernoulli} (0.5)$ in bag $d \in \{1:L\}$ for $k \in \{1:K\}$.
    \FOR{$d \ge L + 1$}
        \FOR{$k = K, \dots, 1$}
            \STATE For $t \in \{1:(d - 1)\}$, let
                \begin{equation*}
                \begin{split}
                    \bX_{t} =& \phi (S_{t}, A_{t}), \\
                    Y_{t} =& R_t + \bar{\gamma} \max_{a_{t + 1} \in \cA} \phi (S_{t + 1}, a_{t + 1})^T \widetilde{\bbeta}_{d - 1},
                \end{split}
                \end{equation*}
                and define $\Xb_{d} = [\bX_{1:(d - 1)}]^T$, $\Yb_{d} = [Y_{1:(d - 1)}]^T$.
            \STATE Fit a Bayesian linear regression for the optimal Q-functions.
            Obtain 
                \begin{equation*}
                    \begin{split}
                        \bSigma_{d} &= \prth{\Xb_{d}^T \Xb_{d} / \sigma^2 + \lambda_{d} \Ib}^{-1}, \\
                        \bmu_{d} &= \bSigma_{d} (\Xb_{d}^T \Yb_{d} / \sigma^2),
                    \end{split}
                \end{equation*}
                where $\Ib$ is the identity matrix.
            \STATE Randomly draw $\widetilde{\bbeta}_{d} \sim N(\bmu_{d}, \bSigma_{d})$. 
        \ENDFOR
        \FOR{$k = 1, \dots, K$}
            \STATE Construct the basis function as in (\ref{equ:srlsvi.basis}). 
            \STATE Take $A_{d} = \argmax_{a_d \in \cA} \phi (S_{d}, a)^T \widetilde{\bbeta}_{d}$.
        \ENDFOR
        \STATE Observe $E_d, R_d$.
    \ENDFOR
    \end{algorithmic}
\end{algorithm}

\subsection{RLSVI with a Horizon $K$} \label{sec:rlsvi}

When each bag is viewed as an episode, algorithms for finite horizon MDP can be leveraged with a reward observed at the end of the episode.
The states can still be defined as $S_{d, k} \in \cS_k$.
For a policy $\bpi = \{\pi_{1:K}: \pi_k: \cS_k \mapsto \cA_k \}$, the Q-function at time $k$ is
\begin{equation*}
    \cQ_k^{\bpi} (s_{d, k}, a_{d, k}) = \bbE^{\bpi} \brce{R_d | S_{d, k} = s_{d, k}, A_{d, k} = a_{d, k}}.
\end{equation*}
The Bellman optimality equation is 
\begin{equation*}
    \begin{split}
        \cQ^*_K (s_{d, K}, a_{d, K}) 
        & = \bbE \{ R_{d, K} |
        S_{d, K} = s_{d, K}, A_{d, K} = a_{d, K} \}, \\
        \cQ^*_k (s_{d, k}, a_{d, k}) 
        & = \bbE \{ R_{d, k} + \max_{a_{d, k + 1} \in \cA_{k + 1}}
        \cQ^*_{k + 1} (S_{d, k + 1}, a_{d, k + 1}) | 
        S_{d, k} = s_{d, k}, A_{d, k} = a_{d, k} \},
    \end{split}
\end{equation*}
for $k \in \{1:(K - 1)\}$.
The optimal Q-function is modeled as 
\begin{equation*}
    \cQ_k (s_{d, k}, a_{d, k}) = \phi_k (s_{d, k}, a_{d, k})^T \bbeta_k,
\end{equation*}
where the basis function is 
\begin{equation} \label{equ:rlsvi.basis}
    \phi_k (S_{d, k}, A_{d, k})^T
    := [1, E_{d - 1}, R_{d - 1}, M_{1:(k - 1)}, A_{1:(k - 1)}, C_{d, k}, A_{d, k}, A_{d, k} E_{d - 1}, A_{d, k} R_{d - 1}, A_{d, k} C_{d, k}].
\end{equation}
Based on the Bellman optimality equation, Algorithm~\ref{alg:rlsvi} updates the estimate of $\bbeta$ sequentially.

\begin{algorithm}[!htbp]
    \caption{RLSVI with a horizon $K$.}
    \label{alg:rlsvi}
    \textbf{Input}: Hyperparameters $L, \lambda_{d}, \sigma^2$.
    
    \begin{algorithmic}[1] 
    \STATE Warm-up: Randomly take actions $A_{d, k} \sim \text{Bernoulli} (0.5)$ in bag $d \in \{1:L\}$ for $k \in \{1:K\}$.
    \FOR{$d \ge L + 1$}
        \FOR{$k = K, \dots, 1$}
            \STATE For $t \in \{1:(d - 1)\}$, let
                \begin{equation*}
                \begin{split}
                    \bX_{t, k} =& \phi_k (S_{t, k}, A_{t, k}), \\
                    Y_{t, k} =& 
                    \begin{cases}
                        \max_{a_{t, k + 1} \in \cA_{k + 1}} \phi_k (S_{t, k + 1}, a_{t, k + 1})^T \widetilde{\bbeta}_{d - 1, k + 1} 
                        & \text{if } k < K, \\
                        R_{t}
                        & \text{if } k = K,
                    \end{cases}
                \end{split}
                \end{equation*}
                and define $\Xb_{d, k} = [\bX_{1:(d - 1), k}]^T$, $\Yb_{d, k} = [Y_{1:(d - 1), k}]^T$.
            \STATE Fit a Bayesian linear regression for the optimal Q-functions.
            Obtain 
                \begin{equation*}
                    \begin{split}
                        \bSigma_{d, k} &= \prth{\Xb_{d, k}^T \Xb_{d, k} / \sigma^2 + \lambda_{d} \Ib}^{-1}, \\
                        \bmu_{d, k} &= \bSigma_{d, k} (\Xb_{d, k}^T \Yb_{d, k} / \sigma^2),
                    \end{split}
                \end{equation*}
                where $\Ib$ is the identity matrix.
            \STATE Randomly draw $\widetilde{\bbeta}_{d, k} \sim N(\bmu_{d, k}, \bSigma_{d, k})$. 
        \ENDFOR
        \FOR{$k = 1, \dots, K$}
            \STATE Construct the basis function as in (\ref{equ:rlsvi.basis}). 
            \STATE Take $A_{d, k} = \argmax_{a_{d, k} \in \cA_k} \phi_k (S_{d, k}, a)^T \widetilde{\bbeta}_{d, k}$.
        \ENDFOR
        \STATE Observe $M_{d, K}, E_d, R_d$.
    \ENDFOR
    \end{algorithmic}
\end{algorithm}

\subsection{Random Policy}
All the actions are taken randomly with probability 0.5. See Algorithms~\ref{alg:rand} for detail.

\begin{algorithm}[!htbp]
    \caption{Random policy (RAND)}
    \label{alg:rand}
    \begin{algorithmic}[1] 
    \FOR{day $d \in \{1:D\}$}
        \FOR{$k \in \{1:K\}$}
            \STATE Observe the current context $C_{d, k}$.
            \STATE Take actions $A_{d, k} \sim \text{Bernoulli} (0.5)$.
        \ENDFOR
        \STATE Observe $R_d$.
    \ENDFOR
    \end{algorithmic}
\end{algorithm}

\subsection{Maximizing Proximal Outcomes with TS}

We compare with the TS algorithm \citep{russo2018tutorial} that directly maximizes the proximal outcome $M_{d, k}$, as implemented in previous HeartSteps studies \citep{liao2020personalized}. 
The state at the $k$th decision time on day $d$ is $S_{d, k} = [E_{d - 1}, R_{d - 1}, C_{d, k}]$.
In TS, the Q-function is modeled as $\cQ(s, a) = \phi (s, a)^T \bbeta$, with the basis function defined as
\[ \phi (S_{d, k}, A_{d, k}) = [1, E_{d - 1}, R_{d - 1}, C_{d, k}, A_{d, k}, A_{d, k} E_{d - 1}, A_{d, k} R_{d - 1}, A_{d, k} C_{d, k}]. \] 
The policy is updated daily. 
Refer to Algorithms~\ref{alg:ts} for detailed implementation.

\begin{algorithm}[!htbp]
    \caption{Maximizing $M_{d, k}$ with TS}
    \label{alg:ts}
    \textbf{Input}: Hyperparameters $L, \lambda_{d}, \sigma^2$.
    
    \begin{algorithmic}[1] 
    \STATE Warm-up: Randomly take actions $A_{d, k} \sim \text{Bernoulli} (0.5)$ in bag $d \in \{1:L\}$ for $k \in \{1:K\}$.
    \FOR{$d \ge L + 1$}
        \STATE Fit a Bayesian linear regression for the optimal Q-functions. Obtain
        \begin{align*}
            \Xb_d =& [\phi (S_{1, 1}, A_{1, 1})^T, \dots, \phi (S_{1, K}, A_{1, K})^T, \dots, \phi (S_{d, 1}, A_{d, 1})^T, \dots, \phi (S_{d, K}, A_{d, K})^T]^T, \\
            \Yb_d =& [M_{1, 1}, \dots, M_{1, K}, \dots, M_{d, 1}, \dots, M_{d, K}]^T, \\
            \bSigma_d^{\prime} =& (\Xb_d^T \Xb_d / \sigma^2 + \lambda_d \bI)^{-1}, \\
            \bmu_d^{\prime} =& \bSigma_d^{\prime} (\Xb_d^T \Yb_d / \sigma^2).
        \end{align*}
        where $\Ib$ is the identity matrix.
        \STATE Randomly draw $\bbeta \sim N(\bmu_d^{\prime}, \bSigma^{\prime}_d)$.
        \FOR{$k \in \{1:K\}$}
            \STATE Observe the current state $S_{d, k} = [E_{d - 1}, R_{d - 1}, C_{d, k}]$.
            \STATE Obtain the posterior probability $P_{d, k} = P \{[1, E_{d - 1}, R_{d - 1}, C_{d, k}] \bbeta_{5:8} > 0; \bbeta \sim N(\bmu_d^{\prime}, \bSigma^{\prime}_d) \}$.
            \STATE Take $A_{d, k} \sim \text{Bernoulli} (P_{d, k})$ and observe $M_{d, k}$.
        \ENDFOR
    \ENDFOR
    \end{algorithmic}
\end{algorithm}

\section{DATA FOR CONSTRUCTING THE TESTBED} \label{sec:data}
The testbed is constructed using data from HeartSteps V2.
This section details the properties of the dataset and how we obtain the variables in the DAG shown in Figure~\ref{fig:dag}.
The HeartSteps V2 dataset has two limitations: (1) the reward commitment to being active contains too much missing data, and (2) 9 users dropped out before the study concluded.
To address these challenges, we exclude the data after dropout and reconstruct the rewards.

\subsection{Dropout} \label{sec:dropout}

According to (anonymous citation), a user must wear the Fitbit tracker for at least 8 hours a day to be included in the analysis. 
Thus, we define the dropout day as the last day the Fitbit tracker was worn for at least 8 hours between 6 a.m. and 11 p.m. 
After the dropout day, the user never wore the Fitbit for more than 8 hours on any given day.
The study is considered ongoing as long as the RL algorithm continues to update nightly, and the user has not exceeded 83 days, in compliance with IRB regulations.

In HeartSteps V2, 9 out of 42 users dropped out before the conclusion of the study, after days beyond the 83-day limit were excluded.  
Refer to Figure~\ref{fig:dropout.bar} for the bar plot of days before, during, and after the study.  
For the 42 users, the average duration of participation is 74 days including days after dropout, and 69 days when excluding days after dropout.  
We have omitted the days following dropout in the subsequent analysis.  
Below, we address some special cases.

\begin{itemize}
    \item Two users stayed in the study for more than 83 days. 
    User 10044 has 291 days of record and dropped out on day 63.
    User 10262 has 210 days of record and dropped out on day 192. 
    We have removed the data beyond 83 days for them.
    All the other users were in the study no more than 83 days.
    \item User 10086 wore the Fitbit for 88 days, but the RL algorithm started on the 48th day.
    \item User 10217 wore the Fitbit for 21 days. The warm-up period is 8 days. 
    Then the RL algorithm started on the 9th day, and the user stayed until the end of the study.
    \item The RL algorithm ran for 60 days for User 10105, but there were only 40 days left after data cleaning.
\end{itemize}

\begin{figure}[!htbp]
    \centering
    \includegraphics[width=0.7\textwidth]{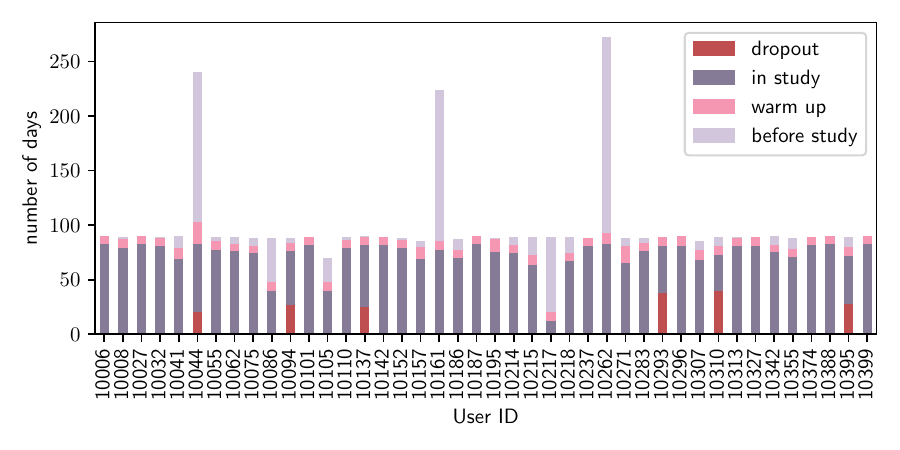}
    \caption{
    The number of days before, during, and after the study for each user.  
    ``Before study'' is the period when the trial starts recording activities for a user, but the user has not yet started wearing the Fitbit.  
    ``Warm up'' is the period during which a user wears the Fitbit for at least 8 hours per day for 7 days. The length of the warm-up period may extend beyond 8 days if the user wears the Fitbit for less than 8 hours on any given day within this timeframe.  
    ``In study'' is the period during which the RL algorithm is active for a user, and the user has not yet dropped out.  
    ``Drop out'' is the period during which the RL algorithm continues to run for a user, but the user no longer wears the Fitbit for more than 8 hours per day.
    }
    \label{fig:dropout.bar}
\end{figure}

%

\subsection{Obtaining the Variables}
In this section, we describe how we obtain the variables from HeartSteps V2. 
At the $k$th decision time on day $d$,  
\begin{itemize}
    \item The context $C_{d,k}$ is the logarithm of the 30-minute step count prior to the $k$th decision time, with a small perturbation of 0.5 added to avoid the logarithm of zero step counts. 
    Missing step counts are imputed with the average from the previous seven days. 
    \item The action $A_{d,k}$ indicates whether a notification is sent (1) or not (0).
    \item The proximal outcome $M_{d,k}$ is the logarithm of the 30-minute step count following the $k$th decision time, again with a 0.5 perturbation added to avoid logarithm issues with zero counts.
    \item The engagement $E_{d}$ is defined as $E_d = \sqrt{\frac{1 - \omega}{1 - \omega^7} \sum_{t = 0}^6 \omega^{t} V_{d - t}}$, where $V_{d}$ is the total app page view times on day $d$ and $\omega = 6/7$ \citep{trella2023reward,ghosh2024rebandit}. 
    The normalizing constant ensures the weights of each $E_{d - t}$ sum to one. 
    The square root transforms the right-skewed distribution of weighted average app page views into a nearly symmetric distribution. 
    The app page views include dashboard view, weekly goal setting, etc.
    \item The reward $R_d$ represents the commitment to being active. Since HeartSteps V2 lacks enough data on commitment to being active, we estimate the reward $R_d$ based on the causal DAG. 
    Details on how to find $R_d$ will be introduced in Section~\ref{sec:construct.reward}.
    \item The emission $O_{d}$ is derived from per minute Fitbit step count data. According to \citet{cuthbertson2022associations}, a bout is defined as at least 10 consecutive minutes at or above the specified cadence, allowing for up to 20\% of the time or less than 5 consecutive minutes below the cadence. 
    We examine the minimum suggested cadence, which is $\geq$ 40 steps/min.
    A bout is considered unprompted if its start time does not occur within 30 minutes after a notification.
    We calculate the number of unprompted bouts of PA using this criterion.
    \begin{itemize}
        \item Check a bout starting from a minute with no fewer than 40 steps.
        \item Look ahead at least 10 minutes.
        \item End the bout if more than 20\% of the minutes have fewer than 40 steps, or if there are more than 5 consecutive minutes with fewer than 40 steps.
        \item If the starting point does not lie within [t, t + 30min], where t is one of the notification times, increment the count of bouts for that day by one.
    \end{itemize}
\end{itemize}

Figure~\ref{fig:histogram.ERO} displays the histogram of $E_d$, $R_d$, and $O_d$.
Section~\ref{sec:missing.data} will address the missing step counts for $C_{d, k}$ and $M_{d, k}$.

\begin{figure}[!htbp]
    \centering
    \begin{subfigure}{0.32\textwidth}
        \centering
        \includegraphics[width=\textwidth]{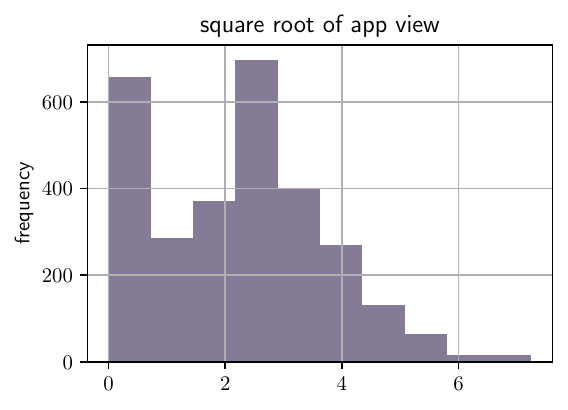}
        \caption{Histogram of $E_d$.}
    \end{subfigure}
    \hfill
    \begin{subfigure}{0.32\textwidth}
        \centering
        \includegraphics[width=\textwidth]{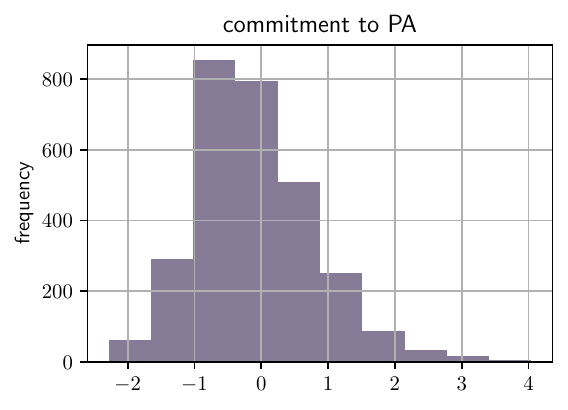}
        \caption{Histogram of $R_d$.}
    \end{subfigure}
    \hfill
    \begin{subfigure}{0.32\textwidth}
        \centering
        \includegraphics[width=\textwidth]{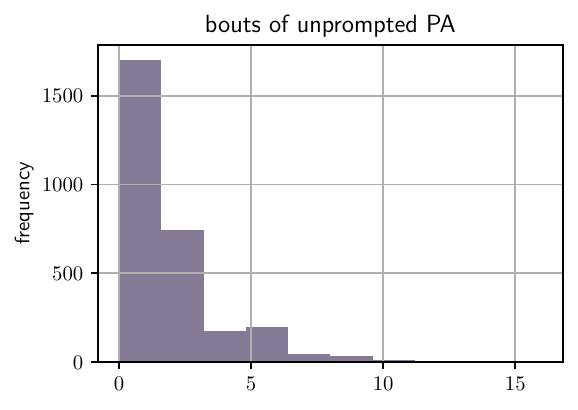}
        \caption{Histogram of $O_d$.}
    \end{subfigure}
    \caption{Histograms of $E_d$ and $R_d$ across 42 users in HeartSteps V2.}
    \label{fig:histogram.ERO}
\end{figure}

\subsection{Missing Data} \label{sec:missing.data}

The 30-minute step count is considered missing if the Fitbit is not worn at any minute during the 30-minute period.
Figures~\ref{fig:histogram.C} and~\ref{fig:histogram.M} show the histograms of $C_{d, k}$ and $M_{d, k}$ with missing data either removed or imputed.
Figures~\ref{fig:histogram.C.missing} and~\ref{fig:histogram.M.missing} show the histograms of decision times or days with missing data.
In subsequent analyses, the imputed log prior 30-minute step count $C_{d, k}$ will be used.
Then only the log post 30-minute step count $M_{d, k}$ will contain missing data.
As will be detailed later, complete or imputed data will be used in Sections~\ref{sec:construct.reward} and~\ref{sec:construct.testbed.vanilla}.

\begin{figure}[!htbp]
    \centering
    \begin{subfigure}{0.49\textwidth}
        \centering
        \includegraphics[width=0.7\textwidth]{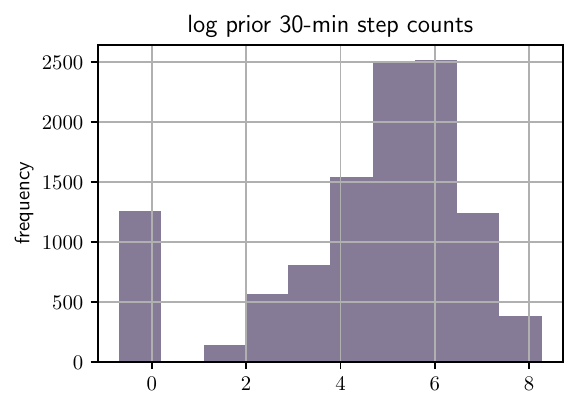}
        \caption{Histogram of $C_{d, k}$ with missing data removed.}
    \end{subfigure}
    \begin{subfigure}{0.49\textwidth}
        \centering
        \includegraphics[width=0.7\textwidth]{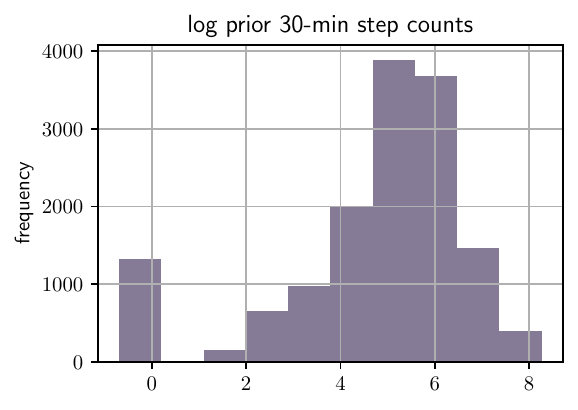}
        \caption{Histogram of $C_{d, k}$ with missing data imputed.}
        \label{fig:histogram.C.impute}
    \end{subfigure}
    \caption{
    Histograms of $C_{d, k}$ across 42 users and all decision times in HeartSteps V2. 
    In Figure~(\subref{fig:histogram.C.impute}), missing $C_{d, k}$ values are imputed using average step counts from the same decision time over the previous seven days, as these will be directly utilized as exogenous variables in the testbed.
    }
    \label{fig:histogram.C}
\end{figure}

\begin{figure}[!htbp]
    \centering
    \begin{subfigure}{0.49\textwidth}
        \centering
        \includegraphics[width=0.7\textwidth]{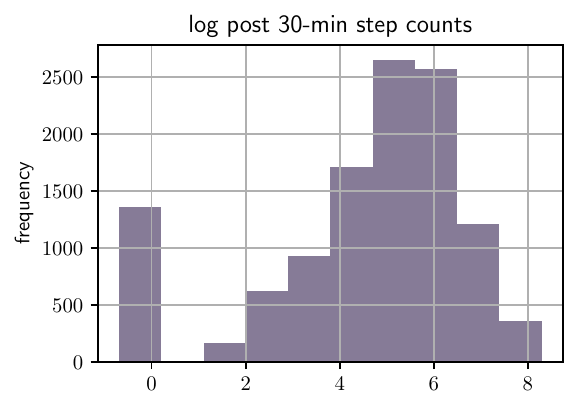}
        \caption{Histogram of $M_{d, k}$ with missing data removed.}
    \end{subfigure}
    \begin{subfigure}{0.49\textwidth}
        \centering
        \includegraphics[width=0.7\textwidth]{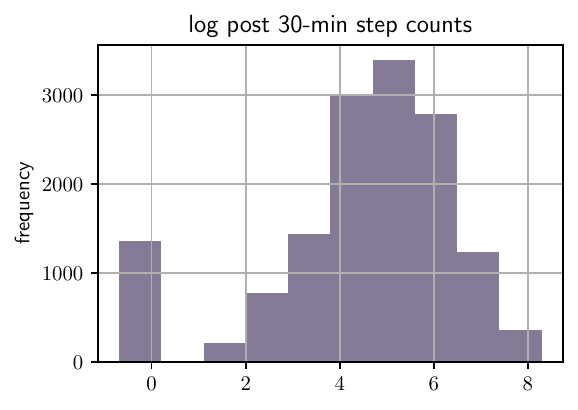}
        \caption{Histogram of $M_{d, k}$ with missing data imputed.}
        \label{fig:histogram.M.impute}
    \end{subfigure}
    \caption{
    Histograms of $M_{d, k}$ across 42 users and all decision times in HeartSteps V2.
    In Figure~(\subref{fig:histogram.M.impute}), missing $M_{d, k}$ values are imputed using the average step count from the same decision time of this user over the study period, as they will be utilized solely for constructing the reward and the testbed.
    }
    \label{fig:histogram.M}
\end{figure}

\begin{figure}[!htbp]
    \centering
    \begin{subfigure}{0.49\textwidth}
        \centering
        \includegraphics[width=0.7\textwidth]{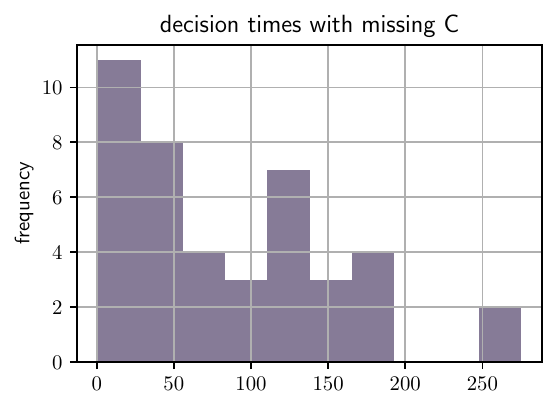}
        \caption{Histogram of the number of decision times with missing $C_{d, k}$.}
    \end{subfigure}
    \begin{subfigure}{0.49\textwidth}
        \centering
        \includegraphics[width=0.7\textwidth]{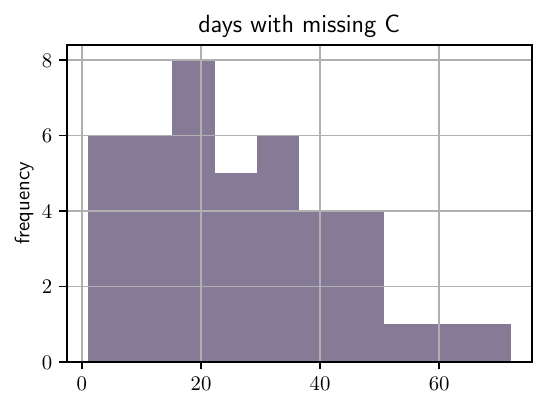}
        \caption{Histogram of the number of days with any one of $C_{d, 1:K}$ missing.}
    \end{subfigure}
    \caption{Number of decision times and days with missing $C_{d, k}$ for each user.}
    \label{fig:histogram.C.missing}
\end{figure}

\begin{figure}[!htbp]
    \centering
    \begin{subfigure}{0.49\textwidth}
        \centering
        \includegraphics[width=0.7\textwidth]{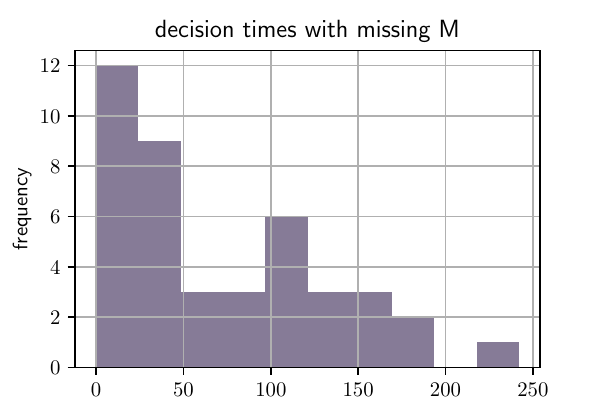}
        \caption{Histogram of the number of decision times with missing $M_{d, k}$.}
    \end{subfigure}
    \begin{subfigure}{0.49\textwidth}
        \centering
        \includegraphics[width=0.7\textwidth]{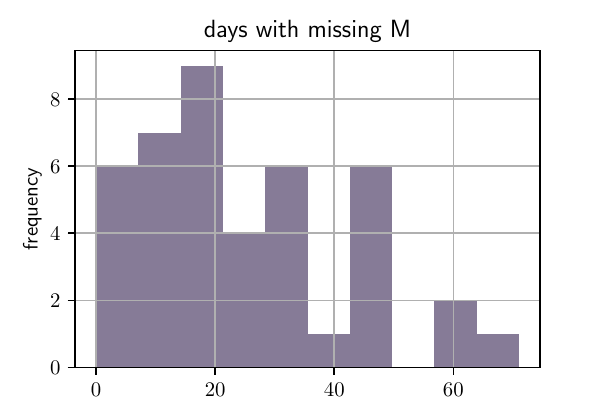}
        \caption{Histogram of the number of days with any one of $M_{d, 1:K}$ missing.}
    \end{subfigure}
    \caption{Number of decision times and days with missing $M_{d, k}$ for each user.}
    \label{fig:histogram.M.missing}
\end{figure}

\subsection{Constructing the Reward} \label{sec:construct.reward}

Since $R_{d}$ is a continuous variable, we fit a Linear Dynamical System (LDS) model  to estimate $R_{d}$ using the DAG in Figure~\ref{fig:dag.O}.

\begin{figure}[!ht]
    \centering
    \begin{tikzpicture}[->, thick, main/.style={font=\sffamily}]
    \matrix [column sep=0.10cm, row sep=0.6cm] {
        & \node[main] (O0) {$O_{d}$};
        & & & & & & & & & 
        & \node[main] (OK) {$O_{d+1}$}; \\
          \node[main] (R0) {$R_{d-1}$};
        & & & & & & & & & 
        & \node[main] (RK) {$R_{d}$}; 
        & & &  \node[main] (R6) {}; \\
        & & \node[main] (M1) {$M_{d, 1}$};
        & & \node[main] (M2) {\dots};
        & & \node[main] (MK) {$M_{d, K}$};
        & \\
        & \node[main] (C1) {$C_{d, 1}$}; & \node[main] (A1) {$A_{d, 1}$};
        & & \node[main] (A2) {\dots};
        & \node[main] (CK) {$C_{d, K}$}; & \node[main] (AK) {$A_{d, K}$};
        & \\
        \node[main] (E0) {$E_{d-1}$};
        & & & & & & & & & 
        & \node[main] (EK) {$E_{d}$}; 
        & & & \node[main] (E6) {}; \\
    };
    \path[every node/.style={font=\sffamily}]
        (R0) edge [color7] node [right] {} (O0)
        (RK) edge [color7] node [right] {} (OK)
        (R0) edge [color1] node [right] {} (RK)
        (RK) edge [color1] node [right] {} (R6)
        (E0) edge [color6] node [right] {} (EK)
        (EK) edge [color6] node [right] {} (E6)
        (E0) edge [color6] node [right] {} (R0)
        (EK) edge [color6] node [right] {} (RK)
        (R0) edge [color2] node [right] {} (M1)
        (R0) edge [color2] node [right] {} (MK)
        (M1) edge [color3] node [right] {} (RK)
        (MK) edge [color3] node [right] {} (RK)
        (E0) edge [color4] node [right] {} (M1)
        (E0) edge [color4] node [right] {} (MK)
        (A1) edge [color5] node [right] {} (EK)
        (AK) edge [color5] node [right] {} (EK)
        (A1) edge node [right] {} (M1)
        (AK) edge node [right] {} (MK)
        (C1) edge node [right] {} (M1)
        (CK) edge node [right] {} (MK);
    \end{tikzpicture}
    \caption{Causal DAG for day $d$. The arrows pointing to the actions $A_{d, 1:K}$ are omitted. 
    The emission $O_{d} \in \bbR$ is an indicator that reveals information about the reward $R_{d-1}$. 
    For example, $O_{d}$ can be the number of bouts of unprompted PA during day $d$. 
    The affective attitude toward PA, $R_d$, is a summary of day $d$, so its effect on unprompted PAs $O_{d+1}$ can only be observed the next day.
    $O_d$ is not impacted by actions in the current day, since the actions may only affect the step counts in a short time window.
    }
    \label{fig:dag.O}
\end{figure}

An LDS is defined by the equations
\begin{align}
    x_d &= b_1 x_{d-1} + b_2 u_d + w_d, \label{equ:lds.utox} \\
    y_d &= b_3 x_{d} + z + v_d, \label{equ:lds.xtoy} 
\end{align}
where $d$ denotes the time index, $x_d$ is the latent variable, $y_d$ is the emission variable, $u_d$ is the input variable, and $w_d, v_d$ are the noise terms.
In equation~(\ref{equ:lds.utox}), the intercept is fixed at zero.
Besides, the input $u_d$ directly influences the latent variable but not the emission variable.
In our problem, the latent variable is $x_d = R_{d}$, and the input variable is the 6-dimensional vector $u_d = [A_{d, 1}, \dots, A_{d, 5}, E_{d}]$.
Given that $R_{d}$ is a more stable construct than the emission $O_{d}$, we apply an Exponentially Weighted Moving Average (EWMA) to smooth $O_{d}$ before estimating $R_{d}$.
Specifically, we set $y_d = \bar{O}_{d + 1}$, where $\bar{O}_d = w O_d + (1 - w) \bar{O}_{d - 1}$ for $d > 1$, with $\bar{O}_1 = O_1$ and $w = 0.5$.
Note that this smoothed $\bar{O}_d$ is utilized solely for estimating $R_{d}$; the original $O_{d}$ is used when constructing the testbed in Section~\ref{sec:construct.testbed.vanilla}.

We fit the LDS model using the \texttt{SSM} package \citep{Linderman_SSM_Bayesian_Learning_2020}, and the estimated latent variable $x_d$ is utilized as the constructed reward.
Note that the Laplace EM algorithm employed by the \texttt{SSM} package is sensitive to the initial parameter values. 
Firstly, the parameters in an LDS are not unique, even with sufficient data. 
The LDS configuration specified possesses additional degrees of freedom. 
For example, the variables $(x_d, b_2, b_3)$ can be simultaneously scaled to $(2x_d, 2b_2, b_3/2)$, illustrating this non-uniqueness. 
Furthermore, the log-posterior probability density function is typically nonconvex with respect to the parameters. 
When all variables are Gaussian, for example, the log-posterior is quadratic but nonconvex. 
The Hessian matrix of the log-posterior is generally neither positive semi-definite nor negative semi-definite.

Considering user heterogeneity, we will fit an LDS model and construct rewards for each user individually.  
However, due to the limited sample size for each user, the model's stability might be compromised because it is sensitive to the initial parameter settings.  
Therefore, we initially fit a homogeneous LDS model using pooled data from all users to obtain stable parameter estimates.  
These estimates are then used to initialize the LDS models for individual users.  
A larger sample size in the Laplace EM algorithm enhances model reliability.

\paragraph{Initialize with the pooled data from all users.}
To initialize the homogeneous LDS model, we use the \texttt{initialize} option in the \texttt{SSM} package.  
The package automatically chooses the initial values $b_1^{\prime}$, $b_2^{\prime}$, $b_3^{\prime}$, $z^{\prime}$, $\Var(w_d^{\prime})$, $\Var(v_d^{\prime})$ for the parameters of the LDS model.  
The \texttt{initialize} option in the package consists of two steps:
\begin{itemize}
    \item Fit an autoregressive HMM to initialize the parameters in equation~(\ref{equ:lds.utox}). 
    \item Fit a PCA model to initialize the parameters in equation~(\ref{equ:lds.xtoy}). Since $R_d$ and $\bar{O}_d$ are both one-dimensional, the package essentially initializes with $b_3^{\prime} = 1$.
\end{itemize}
All parameters are initially set using the data, except for $b_3^{\prime}$.
In this step, data from days with any missing proximal outcome $M_{d, 1:K}$ have been excluded.  
All the variables $C_{d, k}, M_{d, k}, E_{d}, \bar{O}_{d}$ have been standardized.  
After fitting the homogeneous LDS model with the pooled data from all users, we obtain the parameters $b_1^{\prime \prime}$, $b_2^{\prime \prime}$, $b_3^{\prime \prime}$, $z^{\prime \prime}$, $\Var(w_d^{\prime \prime})$, $\Var(v_d^{\prime \prime})$.

Note that we have changed the source code in \texttt{lds.py} of the \texttt{SSM} package to remove the direct influence of the input vector $u_d$ on the emission variable $y_d$.
Besides, we fixed the bug in initializing $\Var(v_d^{\prime})$ when the latent variable $x_d$ and the emission variable $y_d$ have the same dimension.

\paragraph{Fit a separate model for each user.}
For each user, initialize the LDS with $b_1^{\prime \prime}$, $b_2^{\prime \prime}$, $b_3^{\prime \prime}$, $z^{\prime \prime}$, $\Var(w_d^{\prime \prime})$, $\Var(v_d^{\prime \prime})$, and derive the estimated $R_{d}$ for this user.  
In this step, any missing proximal outcome $M_{d, k}$ has been imputed using the average step counts at the same decision time $k$ for this user.  
The variables $C_{d, k}, M_{d, k}, E_{d}, \bar{O}_{d}$ are standardized as well.

\section{CONSTRUCTING THE VANILLA TESTBED} \label{sec:construct.testbed.vanilla}

\subsection{Working Model for the DAG} \label{sec:dag.model}
From the DAG, we make the following assumptions in model~(\ref{equ:testbed.model}):
\begin{equation} \label{equ:testbed.model}
\begin{split}
    C_{d, k} &= \theta^C_{0} + \epsilon^C_{d, k} \\
    M_{d, k} &= \theta^{M}_{0} + \theta^{M}_{1} E_{d-1} + \theta^{M}_{2} R_{d-1} + \theta^{M}_{3} C_{d, k} 
    + A_{d, k} (\theta^{M}_{4} + \theta^{M}_{5} E_{d-1} + \theta^{M}_{6} R_{d-1} + \theta^{M}_{7} C_{d, k}) + \epsilon^M_{d, k}, \\
    E_{d} &= \theta^{E}_{0} + \theta^{E}_{1} E_{d-1} + \textstyle{\sum_{k=1}^K} \theta^{E}_{k+1} A_{d,k} 
    + \textstyle{\sum_{k=1}^K} \theta^{E}_{K + 1 + k} A_{d,k} E_{d-1} + \epsilon^E_{d}, \\
    R_{d} &= \theta^{R}_{0} + \textstyle{\sum_{k=1}^K} \theta^{R}_{k} M_{d,k} + \theta^{R}_{K+1} E_{d} + \theta^{R}_{K+2} R_{d-1} + \epsilon^R_{d}, \\
    O_{d} &= \theta^{O}_{0} + \theta^{O}_{1} R_{d-1} + \epsilon^O_{d},
\end{split}
\end{equation}
where 
$\btheta^C = \theta^{C}_{0}$, 
$\btheta^M = \theta^{M}_{0:7}$, 
$\btheta^E = \theta^{E}_{0:11}$,
$\btheta^R = \theta^{R}_{0:7}$,
$\btheta^O = \theta^{O}_{0:1}$,
and let $\btheta := \{ \btheta^{C}, \btheta^M, \btheta^E, \btheta^R, \btheta^O \}$ represent all the parameters.
The noise terms $\epsilon^C_{d, k}, \epsilon^M_{d, k}, \epsilon^E_{d}, \epsilon^R_{d}, \epsilon^O_{d}$ have mean zero and variances $\sigma^{2C}, \sigma^{2M}, \sigma^{2E}, \sigma^{2R}, \sigma^{2O}$.

\subsection{Standardization} \label{sec:standardization.parameters}

In the subsequent analysis, all variables have been standardized to ensure numerical stability.
Refer to Table~\ref{tab:standardization.parameters} for the standardization parameters.
A variable $G$ is standardized using the formula $(G - \text{shift}_G) / \text{scale}_G$, where $G$ can be $C_{d, k}$, $M_{d, k}$, $E_{d}$, or $O_{d}$.
The standardization parameters are the same for all users.
After constructing the reward, $R_d$ inherently has approximately zero mean and unit variance, hence it does not require further standardization.

\begin{table}[!htbp]
    \caption{Standardization parameters.}
    \label{tab:standardization.parameters}
    \centering
    \begin{tabular}{ccc}
        \toprule
        variable   & shift & scale \\
        \midrule
        $C_{d, k}$ & 4.686 & 2.089 \\
        $M_{d, k}$ & 4.456 & 2.279 \\
        $E_{d}$    & 2.178 & 1.514 \\
        $O_{d}$    & 1.685 & 1.980 \\
        \bottomrule
    \end{tabular}
\end{table}

\subsection{Fitting the Testbed Model} \label{sec:model.fitting}

In the vanilla testbed, we operate under the assumption that model~(\ref{equ:testbed.model}) is valid.
We estimate the coefficients for model~(\ref{equ:testbed.model}) for each user by minimizing the following loss functions:
\begin{align*}
    \widehat{\btheta}^M &= \argmin_{\btheta^M \in \bbR^8} \frac{1}{D K} \sum_{d=1}^D \sum_{k=1}^K [M_{d, k} - \{\theta^{M}_{0} + \theta^{M}_{1} E_{d-1} + \theta^{M}_{2} R_{d-1} + \theta^{M}_{3} C_{d, k} \\
    & \qquad \qquad \qquad \qquad 
    + A_{d, k} \prth{\theta^{M}_{4} + \theta^{M}_{5} E_{d-1} + \theta^{M}_{6} R_{d-1} + \theta^{M}_{7} C_{d, k}} \}]^2 
    + \kappa^M \norm{\btheta^M}^2, \\
    \widehat{\btheta}^E &= \argmin_{\btheta^E \in \bbR^{K + 2}} \frac{1}{D} \sum_{d=1}^D [E_{d} - \{\theta^{E}_{0} + \theta^{E}_{1} E_{d-1} + \sum_{k=1}^K \theta^{E}_{k+1} A_{d,k} + \sum_{k=1}^K \theta^{E}_{K + 1 + k} A_{d,k} E_{d-1} \}]^2 \\
    & \qquad \qquad \qquad \qquad 
    + \kappa^{E}_{1} \norm{\btheta^E}^2 + \kappa^{E}_{2} (\btheta^E)^T \Lb^E (\btheta^E), \\
    \widehat{\btheta}^R &= \argmin_{\btheta^R \in \bbR^{K + 3}} \frac{1}{D} \sum_{d=1}^D [R_{d} - \{\theta^{R}_{0} + \sum_{k=1}^K \theta^{R}_{k} M_{d,k} + \theta^{R}_{K+1} E_{d-1} + \theta^{R}_{K+2} R_{d-1} \}]^2 
    + \kappa^{R}_{1} \norm{\btheta^R}^2 + \kappa^{R}_{2} (\btheta^R)^T \Lb^R (\btheta^R), \\
    \widehat{\btheta}^O &= \argmin_{\btheta^O \in \bbR^{2}} \frac{1}{D} \sum_{d=1}^D [O_{d} - \{ \theta^{O}_{0} + \theta^{O}_{1} R_{d-1} \}]^2 
    + \kappa^O \norm{\btheta^O}^2,
\end{align*}
where $\Lb^E, \Lb^R$ are the Laplacian matrices.
Due to the limited sample size under user availability
and the limited number of notifications at each decision time, $L_2$ penalties and Laplacian penalties are applied to the coefficients.
Let $q$ denote the dimension of the parameter $\btheta$ that we are estimating.
Define the adjacency matrix $\Omega = \{ \omega_{lm} \}_{1 \le l, m \le q}$, where $\omega_{lm} = 1$ indicates an edge between $\beta_l$ and $\beta_m$.
The degree matrix $D$ is defined as $\diag\{ d_{1:q} \}$, where $d_m = \sum_{l=1}^q \absbig{\omega_{lm}}$, and the Laplacian matrix $L$ is defined as $L := D - \Omega$ \citep{huang2011sparse}.
A Laplacian penalty is thus defined as $\btheta^T L \btheta$, which aims to minimize the discrepancies among the components of $\btheta$.
In the model for $E$, the adjacency matrix $\Omega^E$ is constructed s.t. only entries corresponding to $A_{d, k}, A_{d, k^{\prime}}$ or $A_{d, k} E_{d - 1}, A_{d, k^{\prime}} E_{d - 1}$ for $k \ne k^{\prime}$ are set to one.
For the model of $R$, the adjacency matrix $\Omega^R$ is constructed s.t. only entries corresponding to $M_{d, k}, M_{d, k^{\prime}}$ for $k \ne k^{\prime}$ are set to one.
All other entries in $\Omega^E$ and $\Omega^R$ set to zero.
The tuning parameters $\kappa^M, \kappa^{E}_{1}, \kappa^{E}_{2}, \kappa^{R}_{1}, \kappa^{R}_{2}, \kappa^O$ are selected through cross-validation.
Given that the context variable $C_{d, k}$ has been standardized and is treated as exogenous, $\btheta^C$ is fixed at $\widehat{\btheta}^C = 0$.

We employ the gradient-based method \texttt{BFGS} from the package \texttt{scipy.optimize.minimize} to solve the optimization problem.
Initial values for the parameters in the optimization are set to zero.
This penalized linear regression differs from the previous LDS model; it is deployed to establish a robust data generation model after the rewards have been observed or constructed.
The fitted coefficients are illustrated in Figure~\ref{fig:est.coefs}.
For each fitted model of $M, E, R, O$, we store the estimated coefficients $\widehat{\btheta}^C, \widehat{\btheta}^M, \widehat{\btheta}^E, \widehat{\btheta}^R, \widehat{\btheta}^O$ and the residuals $\br^C, \br^M, \br^E, \br^R, \br^O$.
Given that $\widehat{\btheta}^C = 0$, the residual vector $\br^C$ effectively represents the context $C_{d, k}$.
Refer to Section~\ref{sec:data.generation} for further discussion on the usage of the residual vector.

The missing data in $M_{d,k}$ are managed as follows.
When fitting the model for $M_{d,k}$, all decision times with missing $M_{d, k}$ are excluded.
We choose not to impute missing $M_{d, k}$ with the average $M_{d, k}$ at time $k$ since this average step count is a constant but not a variable dependent on $A_{d, k}$.
Additionally, we only utilize data when a user is available.
When fitting the model for $R_{d}$, missing $M_{d, k}$ is imputed with the average observed $M_{d, k}$ at decision time $k$ across all days.

According to domain knowledge, the direct effect of action $A_{d, k}$ will not decrease the proximal outcome $M_{d, k}$, as a user will not purposefully reduce PA after receiving a notification.
Therefore, we set the advantage function in $A_{d, k} \to M_{d, k}$ to be nonnegative, i.e.
\begin{equation} \label{equ:step.truncation}
    M_{d, k} = \theta^{M}_{0} + \theta^{M}_{1} E_{d-1} + \theta^{M}_{2} R_{d-1} + \theta^{M}_{3} C_{d, k} + A_{d, k} \max(0, \theta^{M}_{4} + \theta^{M}_{5} E_{d-1} + \theta^{M}_{6} R_{d-1} + \theta^{M}_{7} C_{d, k}) + \epsilon^M_{d, k}.
\end{equation}
Similarly, the proximal outcome $M_{d, k}$ will not decrease $R_d$, and previous engagement $E_{d - 1}$ and reward $R_{d - 1}$ will not decrease the proximal outcome $M_{d, k}$.
Therefore, we set the fitted coefficients $\theta^R_{1:K}, \theta^M_{1:2}$ to be nonnegative for all users.

\subsection{Generate Data With the Testbed} \label{sec:data.generation}

To evaluate the performance of our proposed method in an online setting, we use the testbed to generate data for simulation studies as follows.

Given the historical data, the mean of the next observation for $C_{d, k}, M_{d, k}, E_{d}, R_{d}$, or $O_{d}$ is generated by plugging in the estimated coefficients $\widehat{\btheta}_M$, $\widehat{\btheta}_E$, $\widehat{\btheta}_R$, $\widehat{\btheta}_O$.  
The residuals $\br^M$, $\br^E$, $\br^R$, $\br^O$ are then added to the mean of the next observation as the noise terms $\epsilon^M$, $\epsilon^E$, $\epsilon^R$, $\epsilon^O$.  
If the length of the simulated study $D$ exceeds that of HeartSteps V2, the residual vector will be used cyclically.  
For example, if the residual vector's length is 80 (i.e., in the original HeartSteps V2 dataset, the user was in the study for 80 days before dropping out), then we will sequentially use the residuals from day 1 to 80, and then repeat from day 81 to 160.  
Since there is missing data in the log post 30-minute step count $M_{d, k}$, we cannot obtain the residual of $M_{d, k}$ at some decision times.  
In such cases, we use the residual vector $\br^M$ of observed $M_{d, k}$ as an empirical distribution, sampling a noise term $\epsilon^M_{d, k}$ from it with replacement.  
The initial reward $R_0$ and initial engagement $E_0$ are exogenous variables and are directly set to the initial values from HeartSteps V2.

To prevent extreme values, we truncate all variables when generating data with the testbed.
The lower and upper bounds are set as the minimum and maximum values of each variable in HeartSteps V2.
See Table~\ref{tab:limit.parameters} for the truncation parameters.
A variable $G$ is truncated as $\min \{ \max \{ G, \text{lower}_G \}, \text{upper}_G \}$, where $G$ can be $C_{d, k}$, $M_{d, k}$, $E_{d}$, $R_{d}$, or $O_{d}$.
Since $\widehat{\btheta}^C = 0$ and the residual vector $\br^C$ corresponds to $C_{d, k}$, the context $C_{d, k}$ naturally varies between the lower and upper bounds without truncation.

\begin{table}[!htbp]
    \caption{Truncation parameters.}
    \label{tab:limit.parameters}
    \centering
    \begin{tabular}{ccc}
        \toprule
        variable   & lower bound & upper bound \\
        \midrule
        $C_{d, k}$ & -2.575 & 1.714 \\
        $M_{d, k}$ & -2.259 & 1.684 \\
        $E_{d}$    & -1.439 & 3.349 \\
        $R_{d}$    & -2.285 & 4.042 \\
        $O_{d}$    & -0.851 & 7.230 \\
        \bottomrule
    \end{tabular}
\end{table}

\subsection{Model Diagnostics} \label{sec:model.diagnostics}

\paragraph{Estimated Coefficients.}
The histograms of the estimated coefficients $\widehat{\btheta}^M$, $\widehat{\btheta}^E$, $\widehat{\btheta}^R$, and $\widehat{\btheta}^O$ for the 42 users are displayed in Figure~\ref{fig:est.coefs}.
We observe that the estimated coefficients generally fall within the range of $[-1, 1]$. 
This confirms that after variable standardization, the influence of previous states on the next state remains within a reasonable range.
The autoregressive coefficients $\theta^E_1$ for $E_{d}$ and $\theta^R_7$ for $R_{d}$ are always nonnegative.
The coefficient $\theta^O_1$ is also nonnegative, as the reward $R_d$ is constructed using the emission $O_{d + 1}$.
The coefficient $\theta^M_3$ indicates a strong positive correlation between the prior 30-minute step count $C_{d, k}$ and the post 30-minute step count $M_{d, k}$.

\paragraph{Residual Plots.}
To check the residuals of the fitted models, we plot the residuals $\br^M$, $\br^E$, $\br^R$, and $\br^O$ against the predicted values $\widehat{M}_{d, k}$, $\widehat{E}_{d}$, $\widehat{R}_{d}$, and $\widehat{O}_{d}$ in Figures~\ref{fig:resid.plot.M}-\ref{fig:resid.plot.O}.
The residual plots indicate that there is no discernible trend or heterogeneous variance in the residuals.

\paragraph{Trend of Generated Data.}
To compare the original data with the generated data, we plot the original $M_{d, k}$, $E_{d}$, $R_{d}$, $O_{d}$ and randomly generated episodes for the first 4 users in HeartSteps V2 in Figures~\ref{fig:trend.M}-\ref{fig:trend.O}.
In the generated episodes, all actions are randomly taken with a probability of 0.5, and the other variables are generated following the procedure outlined in Section~\ref{sec:data.generation}.
Since HeartSteps V2 is a 3-month study and the new trial is expected to last 9 months, we generate both a 3-month episode and a 9-month episode.
We observe that the 3-month episode closely resembles the original episodes, due to the use of the residual vectors $\br^C$, $\br^M$, $\br^E$, $\br^R$, $\br^O$.
Even when extending the episode length to 9 months, the generated variables remain within a reasonable range.

Figures~\ref{fig:trend.M}-\ref{fig:trend.O} also show that the noise is much larger than the changes in the mean of the variables.
Recall that when generating artificial data in Algorithm~\ref{alg:brlsvi}, the initial distribution of $E_d$ and $R_d$ is derived from the empirical distribution across all users in HeartSteps V3.
This means that $E_d$ and $R_d$ are sampled from the full range of values observed in the previous trial.
This observation confirms that even if the interventions are delivered using different policies in the previous and new trials, the distributions of $E_d$ and $R_d$ in the artificial data—sampled from the previous trial—will closely resemble those in the new trial due to the high noise level.
Therefore, it is enough to generate one day of data in each artificial episode.

\begin{figure}[!htbp]
    \centering
    \includegraphics[width=\textwidth]{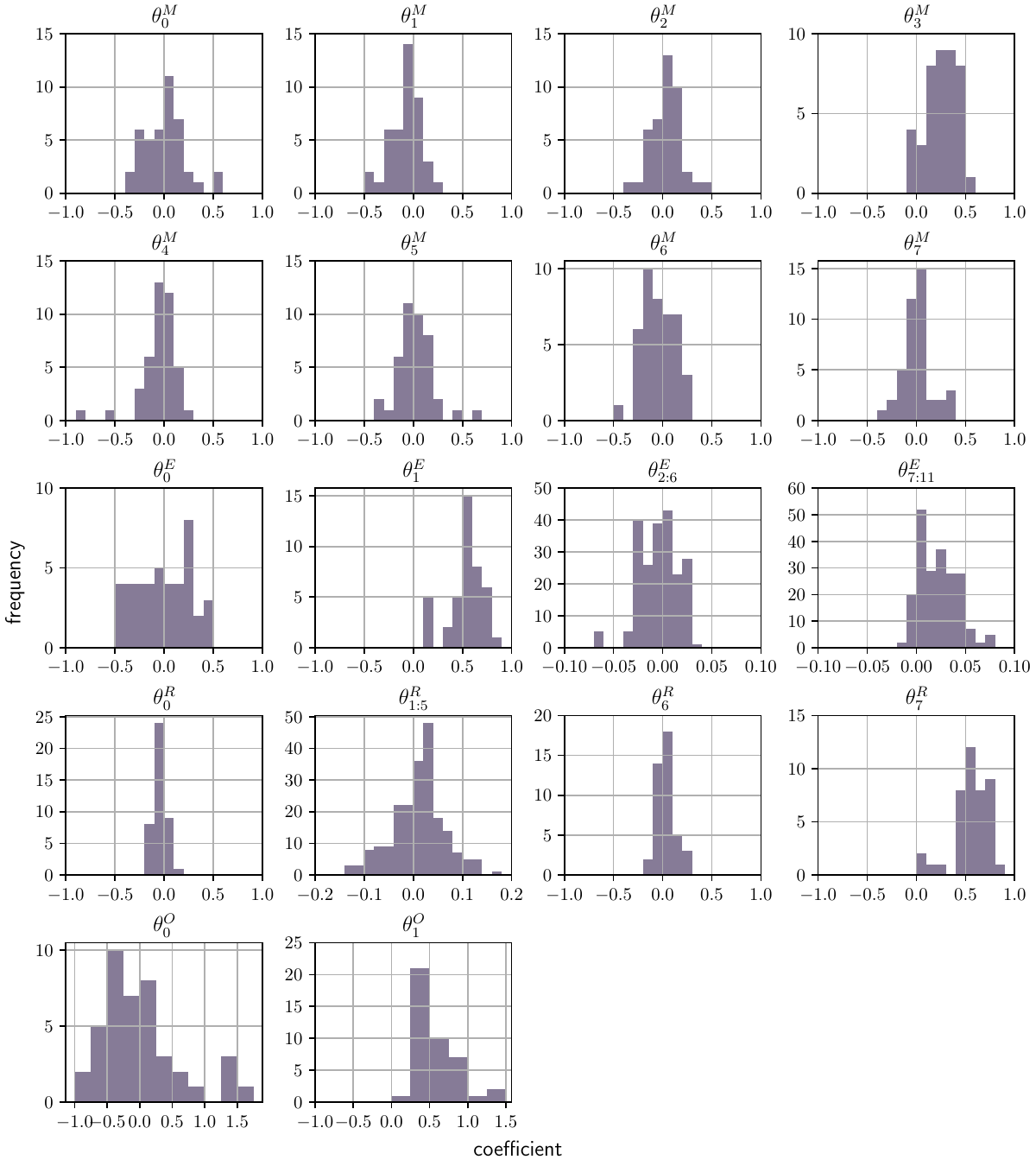}
    \caption{
    Histograms of the estimated coefficients $\widehat{\btheta}^M$, $\widehat{\btheta}^E$, $\widehat{\btheta}^R$, and $\widehat{\btheta}^O$ for the 42 users in the vanilla testbed. 
    The first two rows represent the coefficients $\widehat{\btheta}^M$, while the third to fifth rows represent the coefficients $\widehat{\btheta}^E$, $\widehat{\btheta}^R$, and $\widehat{\btheta}^O$, respectively.
    The coefficients $\theta^E_{2:6}$, $\theta^E_{7:11}$, and $\theta^R_{1:5}$, corresponding to $K$ actions or $K$ proximal outcomes, are plotted in the same histogram.
    }
    \label{fig:est.coefs}
\end{figure}

\begin{figure}[!htbp]
    \centering
    \includegraphics[width=\textwidth]{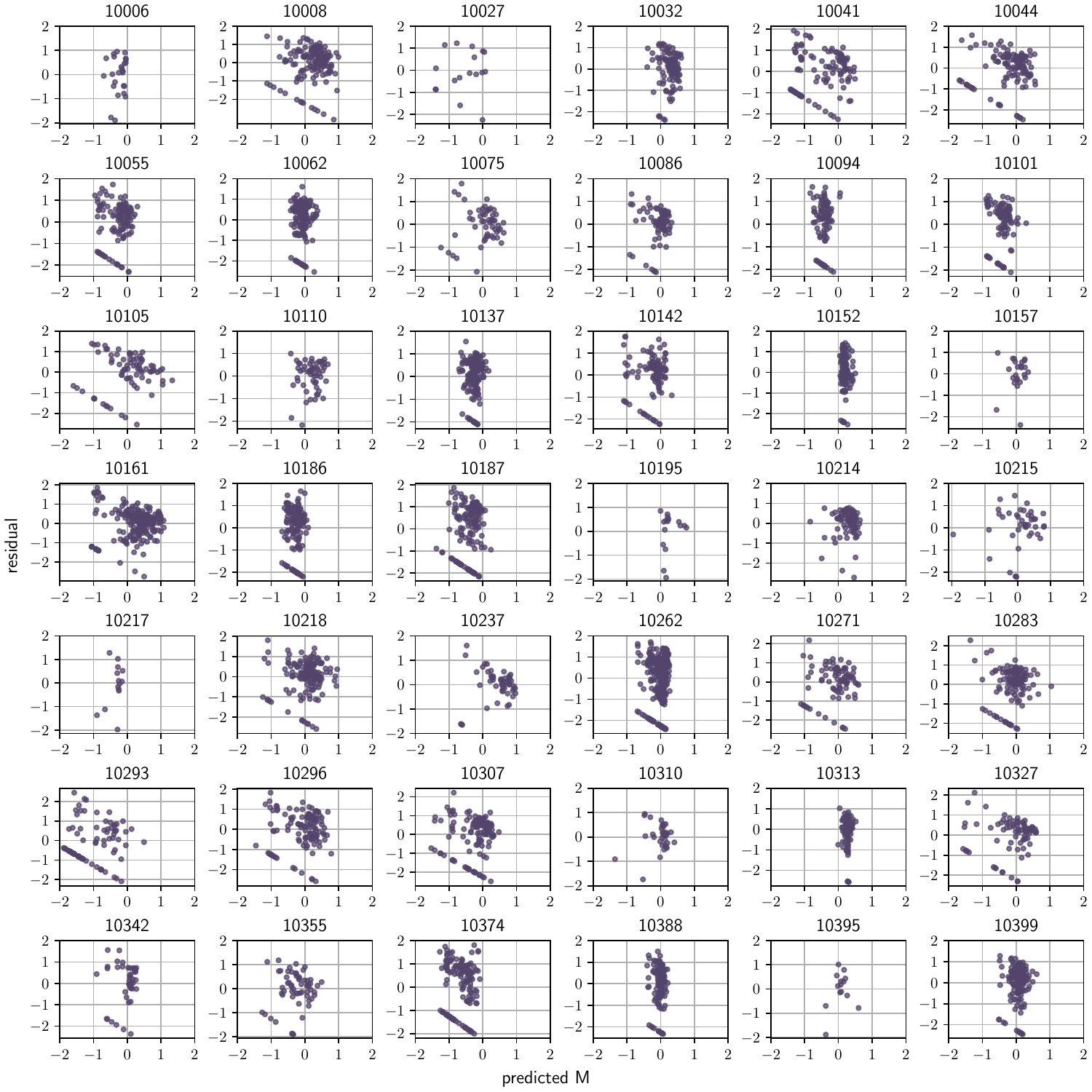}
    \caption{
    Plots of residuals $\br^M$ against the predicted values $\widehat{M}_{d, k}$ for the 42 users in HeartSteps V2.
    The line at the bottom of each subplot indicates the points where $M_{d, k}$ equals zero on the original scale (before adding a perturbation of 0.5, taking the logarithm, or conducting standardization).
    Although the step counts are zero-inflated variables, with less than 10\% of the step counts being zero, we use a linear regression model for simplicity.
    }
    \label{fig:resid.plot.M}
\end{figure}

\begin{figure}[!htbp]
    \centering
    \includegraphics[width=\textwidth]{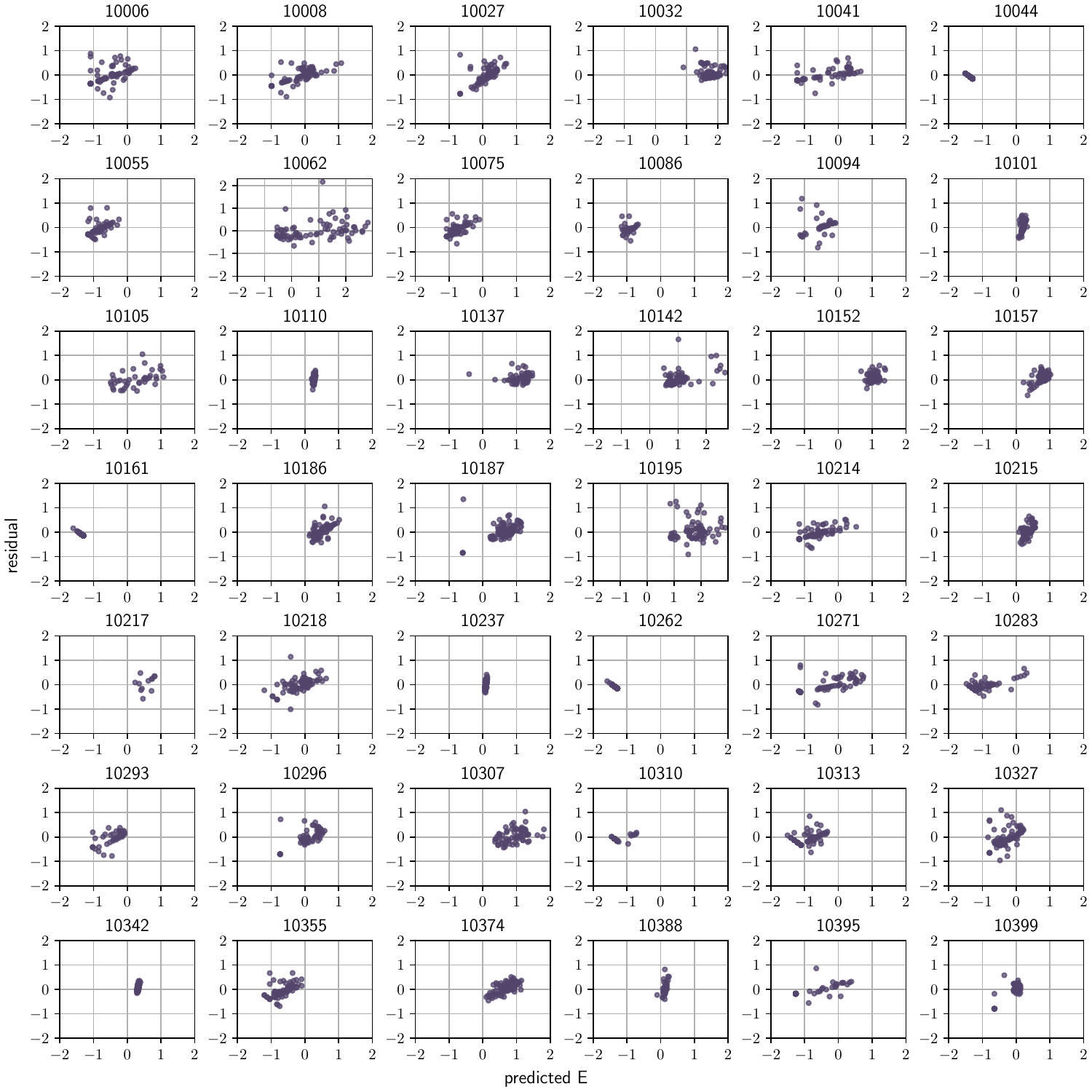}
    \caption{
    Plots of residuals $\br^E$ against the predicted values $\widehat{E}_{d, k}$ for the 42 users in HeartSteps V2.
    There are no records of app views for users 10044, 10161, and 10262, resulting in predicted values $\widehat{E}_{d, k}$ that are consistently close to the lower bound of $E$.
    }
    \label{fig:resid.plot.E}
\end{figure}

\begin{figure}[!htbp]
    \centering
    \includegraphics[width=\textwidth]{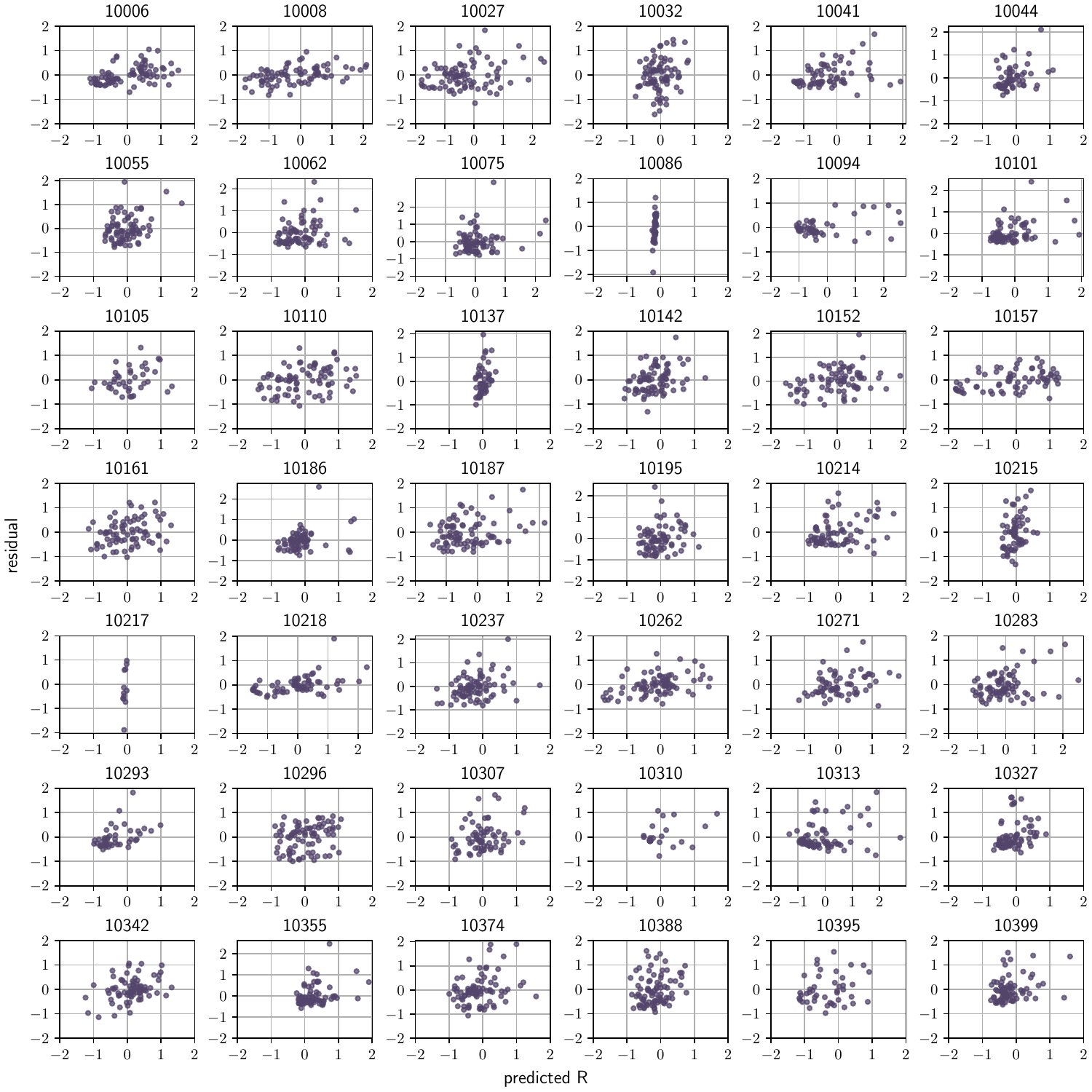}
    \caption{
    Plots of residuals $\br^R$ against the predicted values $\widehat{R}_{d, k}$ for the 42 users in HeartSteps V2.
    }
    \label{fig:resid.plot.R}
\end{figure}

\begin{figure}[!htbp]
    \centering
    \includegraphics[width=\textwidth]{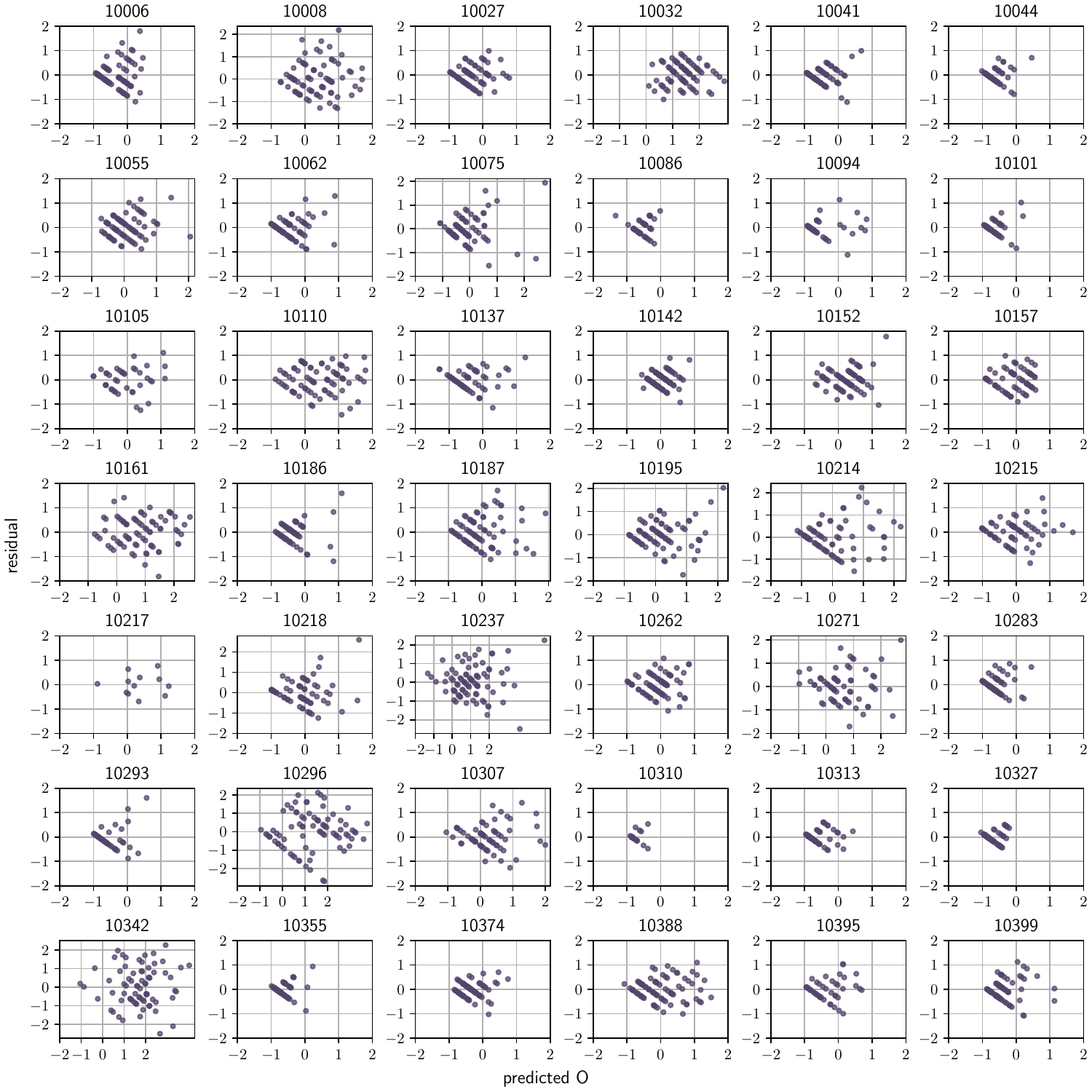}
    \caption{
    Plots of residuals $\br^O$ against the predicted values $\widehat{O}_{d, k}$ for the 42 users in HeartSteps V2.
    The residuals are scattered along parallel lines because the original values of $O_d$ before standardization are integers.
    }
    \label{fig:resid.plot.O}
\end{figure}

\begin{figure}[!htbp]
    \centering
    \begin{subfigure}{0.32\textwidth}
        \centering
        \includegraphics[width=\textwidth]{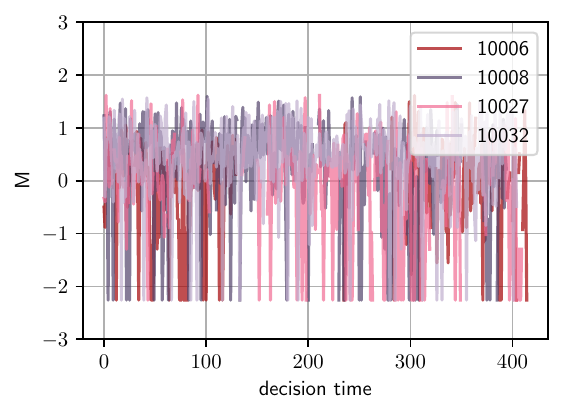}
        \caption{Original 3-month episode.}
    \end{subfigure}
    \hfill
    \begin{subfigure}{0.32\textwidth}
        \centering
        \includegraphics[width=\textwidth]{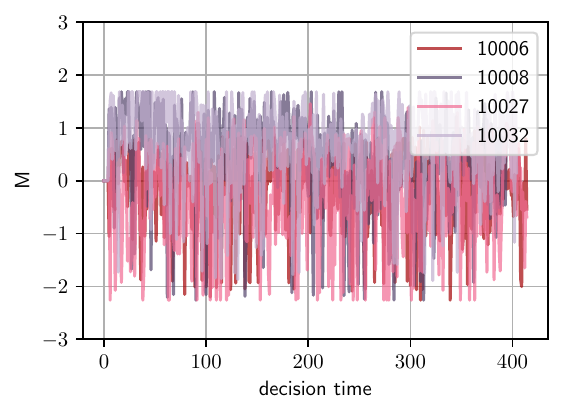}
        \caption{Generated 3-month episode.}
    \end{subfigure}
    \hfill
    \begin{subfigure}{0.32\textwidth}
        \centering
        \includegraphics[width=\textwidth]{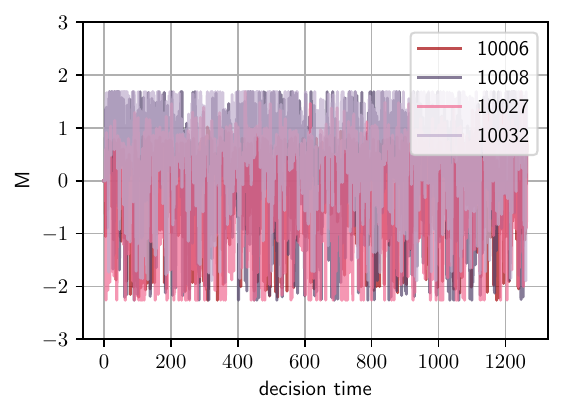}
        \caption{Generated 9-month episode.}
    \end{subfigure}
    \caption{Trends of the original and generated $M_{d, k}$ for the first four users in HeartSteps V2.}
    \label{fig:trend.M}
\end{figure}

\begin{figure}[!htbp]
    \centering
    \begin{subfigure}{0.32\textwidth}
        \centering
        \includegraphics[width=\textwidth]{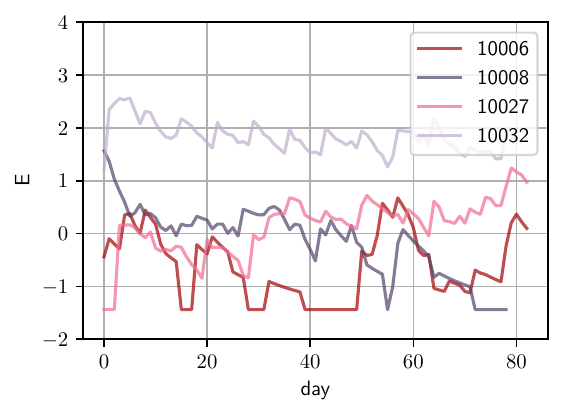}
        \caption{Original 3-month episode.}
    \end{subfigure}
    \hfill
    \begin{subfigure}{0.32\textwidth}
        \centering
        \includegraphics[width=\textwidth]{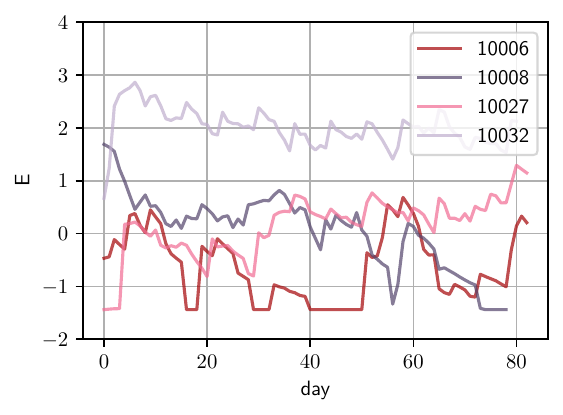}
        \caption{Generated 3-month episode.}
    \end{subfigure}
    \hfill
    \begin{subfigure}{0.32\textwidth}
        \centering
        \includegraphics[width=\textwidth]{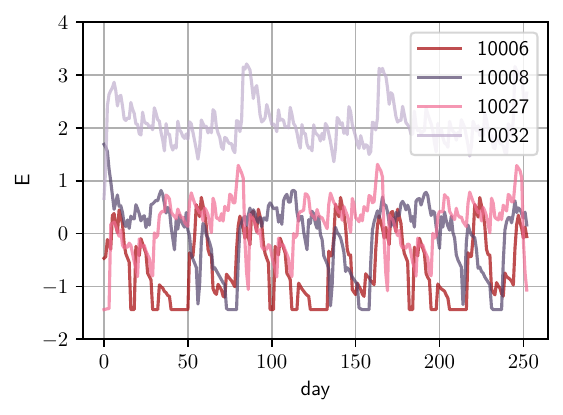}
        \caption{Generated 9-month episode.}
    \end{subfigure}
    \caption{Trends of the original and generated $E_{d}$ for the first four users in HeartSteps V2.}
    \label{fig:trend.E}
\end{figure}

\begin{figure}[!htbp]
    \centering
    \begin{subfigure}{0.32\textwidth}
        \centering
        \includegraphics[width=\textwidth]{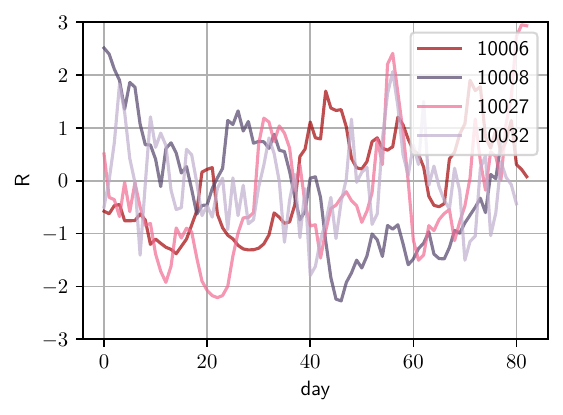}
        \caption{Original 3-month episode.}
    \end{subfigure}
    \hfill
    \begin{subfigure}{0.32\textwidth}
        \centering
        \includegraphics[width=\textwidth]{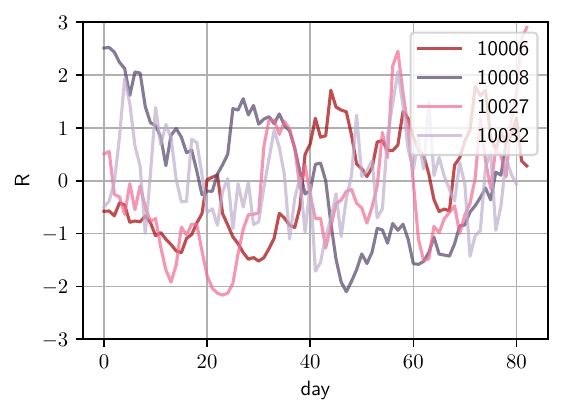}
        \caption{Generated 3-month episode.}
    \end{subfigure}
    \hfill
    \begin{subfigure}{0.32\textwidth}
        \centering
        \includegraphics[width=\textwidth]{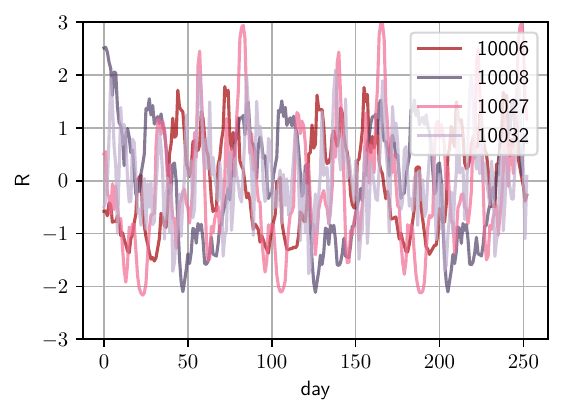}
        \caption{Generated 9-month episode.}
    \end{subfigure}
    \caption{Trends of the original and generated $R_{d}$ for the first four users in HeartSteps V2.}
    \label{fig:trend.R}
\end{figure}

\begin{figure}[!htbp]
    \centering
    \begin{subfigure}{0.32\textwidth}
        \centering
        \includegraphics[width=\textwidth]{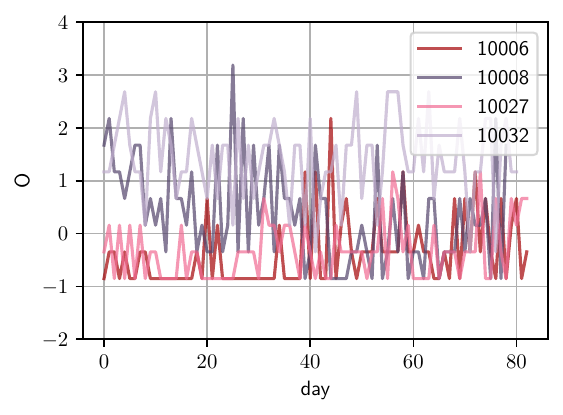}
        \caption{Original 3-month episode.}
    \end{subfigure}
    \hfill
    \begin{subfigure}{0.32\textwidth}
        \centering
        \includegraphics[width=\textwidth]{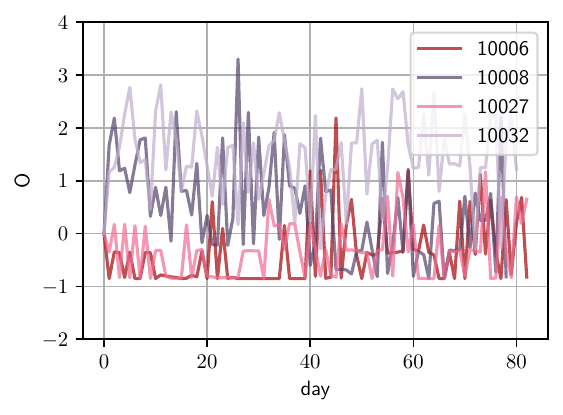}
        \caption{Generated 3-month episode.}
    \end{subfigure}
    \hfill
    \begin{subfigure}{0.32\textwidth}
        \centering
        \includegraphics[width=\textwidth]{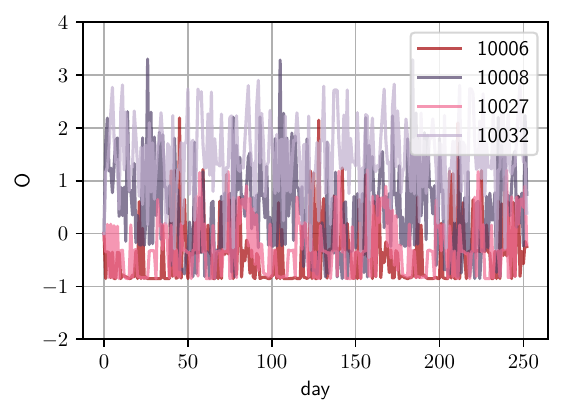}
        \caption{Generated 9-month episode.}
    \end{subfigure}
    \caption{Trends of the original and generated $O_{d}$ for the first four users in HeartSteps V2.}
    \label{fig:trend.O}
\end{figure}

\section{VARIANTS OF THE TESTBED} \label{sec:construct.testbed.variants}

\subsection{Standardized Treatment Effect (STE)} \label{sec:ste.estimate}

The standardized treatment effect (STE) is an indicator of the signal-to-noise ratio and reflects the difficulty of learning the optimal policy.
For each user, the STE is defined as
\begin{equation} \label{equ:ste}
    \frac{\bbE_{E_0, R_0 \sim \nu} \cV^{\bpi^*} (S_{1, 1}) - \bbE_{E_0, R_0 \sim \nu} \cV^{\bpi^0} (S_{1, 1})}{\sqrt{\Var_{E_0, R_0 \sim \nu} \cV^{\bpi^0} (S_{1, 1})}},
\end{equation}
where $\bpi^*$ is the optimal policy and $\bpi^0$ is the zero policy that always takes action zero.
The STE of a testbed is the average STE across all users.

When estimating the optimal policy of a given testbed, the environment is known to us. 
We can generate a large dataset as needed and find the optimal policy using flexible nonparametric methods like deep RL. 
The state vector is also $S^{\prime}_{d, k} = [E_{d - 1}, R_{d - 1}, M_{d, 1:(k-1)}, A_{d, 1:(k-1)}, C_{d, k}]$.
To apply standard methods that estimate the optimal policy for a stationary MDP, we expand the state vector as
\begin{equation*}
    \begin{split}
        \widetilde{S}^{\prime}_{d, k} 
        =& [E_{d - 1}, R_{d - 1}, \widetilde{M}_{d, 1:(K - 1)}, \widetilde{A}_{d, 1:(K - 1)}, C_{d, k}, \bbone (k = 2), \dots, \bbone (k = K)].
    \end{split}
\end{equation*}
Here, the indicators represent the time of day, with zeros added to pad the mediators to the same length.

To learn the optimal policy for each user, we use the offline deep Q-network (DQN) algorithm, implemented with the \texttt{d3rlpy} package \citep{d3rlpy}.
For each user, the training data consists of 5000 episodes.
Each episode contains $D = 252$ days and is randomly generated with actions taken with a probability of 0.5.
It is important to note that while this method offers a principled approach to identifying the optimal policy, we do not implement it in the online clinical trial.
Due to the large state space and the limited sample size, the learned policy is likely to have high variance in the trial.
This method is only employed when we have access to unlimited data in a known environment.

The mean and variance of the value function under the policies $\bpi^*$ and $\bpi^0$ are estimated using Monte Carlo methods.
For each user, the test data consists of 500 episodes.
Each episode contains $D = 252$ days, with actions taken according to the policy $\bpi^*$ or $\bpi^0$.
We then calculate the mean and variance of the total rewards over 9 months across the 500 episodes.

When generating training and test data, it is crucial to randomly sample the noise terms from the residual vectors $\br^C$, $\br^M$, $\br^E$, $\br^R$, and $\br^O$. 
Even when the signal is sufficiently large for the optimal policy to significantly outperform the zero policy over the 9-month horizon, the everyday noise remains much greater than the daily improvement in the reward. 
If the residual vector is used cyclically in the training data, each episode becomes nearly identical, causing an RL algorithm like DQN to struggle in effectively learning the policy, as the noise becomes indistinguishable from the signal.
Similarly, for test data, the variance of the zero policy estimated from it would be much smaller than the variance observed in a real trial, failing to accurately describe the STE and reflect the true difficulty of learning.

\subsection{Requirements for Constructing Variants of the Testbed} \label{sec:testbed.variant.requirements}

By plugging $M_{d, k}$ into the model of $R_d$, equation~(\ref{equ:testbed.model}) can be expressed as a vector autoregressive (VAR) process.
That is,
\begin{equation*}
\begin{split}
    R_{d} =& \theta^{R}_{0} + \sum_{k=1}^K \theta^{R}_{k} M_{d,k} + \theta^{R}_{K+1} E_{d} + \theta^{R}_{K+2} R_{d-1} + \epsilon^R_{d}, \\
    =& \theta^{R}_{0} 
    + \sum_{k=1}^K \theta^{R}_{k} \brck{\theta^{M}_{0} + \theta^{M}_{1} E_{d-1} + \theta^{M}_{2} R_{d-1} + \theta^{M}_{3} C_{d, k} + A_{d, k} \prth{\theta^{M}_{4} + \theta^{M}_{5} E_{d-1} + \theta^{M}_{6} R_{d-1} + \theta^{M}_{7} C_{d, k}} + \epsilon^M_{d, k}} \\
    & + \theta^{R}_{K+1} [\theta^{E}_{0} + \theta^{E}_{1} E_{d-1} + \sum_{k=1}^K \theta^{E}_{k+1} A_{d,k} + \sum_{k=1}^K \theta^{E}_{K + 1 + k} A_{d,k} E_{d-1} + \epsilon^E_{d}] 
    + \theta^{R}_{K+2} R_{d-1} + \epsilon^R_{d}, \\
    =& \brce{\theta^{R}_{0} + \sum_{k=1}^K \theta^{R}_{k} (\theta^{M}_{0} + \theta^{M}_{4} A_{d, k}) + \theta^{R}_{K+1} (\theta^{E}_{0} + \sum_{k=1}^K \theta^{E}_{k+1} A_{d,k})}
    + \brce{\sum_{k=1}^K \theta^{R}_{k} C_{d, k} (\theta^{M}_{3} + \theta^{M}_{7} A_{d, k})} \\
    & + E_{d-1} \brce{\sum_{k=1}^K \theta^{R}_{k} (\theta^{M}_{1} + \theta^{M}_{5} A_{d, k}) + \theta^{R}_{K+1} (\theta^{E}_{1} + \sum_{k=1}^K \theta^{E}_{K + 1 + k} A_{d,k})} \\
    & + R_{d-1} \brce{\theta^{R}_{K+2} + \sum_{k=1}^K \theta^{R}_{k} (\theta^{M}_{2} + \theta^{M}_{6} A_{d, k})}
    + \brce{\epsilon^R_{d} + \sum_{k=1}^K \theta^{R}_{k} \epsilon^M_{d, k} + \theta^{R}_{K+1} \epsilon^R_{d}} \\
    E_{d} &= \theta^{E}_{0} + \theta^{E}_{1} E_{d-1} + \sum_{k=1}^K \theta^{E}_{k+1} A_{d,k} + \sum_{k=1}^K \theta^{E}_{K + 1 + k} A_{d,k} E_{d-1} + \epsilon^E_{d}.
\end{split}
\end{equation*}
In the form of VAR,
\begin{equation*}
    \begin{bmatrix}
        R_d \\
        E_d \\
    \end{bmatrix}
    =
    \bphi
    + \bPhi
    \begin{bmatrix}
        R_{d-1} \\
        E_{d-1} \\
    \end{bmatrix}
    + \bvarepsilon_d,
\end{equation*}
where
\begin{align*}
    \bPhi = & 
    \begin{bmatrix}
          \theta^{R}_{K+2} + \sum_{k=1}^K \theta^{R}_{k} (\theta^{M}_{2} + \theta^{M}_{6} A_{d, k})
        & \sum_{k=1}^K \theta^{R}_{k} (\theta^{M}_{1} + \theta^{M}_{5} A_{d, k}) + \theta^{R}_{K+1} (\theta^{E}_{1} + \sum_{k=1}^K \theta^{E}_{K + 1 + k} A_{d,k}) \\
          0
        & \theta^{E}_{1} + \sum_{k=1}^K \theta^{E}_{K + 1 + k} A_{d,k}
    \end{bmatrix}, \\
    \bphi = & 
    \begin{bmatrix}
          \theta^{R}_{0} + \sum_{k=1}^K \theta^{R}_{k} (\theta^{M}_{0} + \theta^{M}_{4} A_{d, k}) + \theta^{R}_{K+1} (\theta^{E}_{0} + \sum_{k=1}^K \theta^{E}_{k+1} A_{d,k})
        + \sum_{k=1}^K \theta^{R}_{k} C_{d, k} (\theta^{M}_{3} + \theta^{M}_{7} A_{d, k}) \\
          \theta^{E}_{0} + \sum_{k=1}^K \theta^{E}_{k+1} A_{d,k} \\
    \end{bmatrix}, \\
    \bvarepsilon_d = & 
    \begin{bmatrix}
        \epsilon^R_{d} + \sum_{k=1}^K \theta^{R}_{k} \epsilon^M_{d, k} + \theta^{R}_{K+1} \epsilon^R_{d} \\
        \epsilon^E_{d}
    \end{bmatrix}.
\end{align*}

When constructing the testbed, it is essential to ensure that the following three requirements are met:
\begin{enumerate}
    \item The VAR model is stationary when either $A_{d,k} = 0$ or $A_{d,k} = 1$.
    \item In the limiting distribution, the mean of rewards $R_d$ should not exceed the maximum value observed in the real data. The rewards should fluctuate within a reasonable range, even without any variable truncation.
    \item The STE should adhere to the guidelines established in behavioral science.
\end{enumerate}
Since the vanilla testbed is directly fitted to the real data, these requirements are typically already satisfied.
However, we need to ensure that these requirements hold for variants of the testbed when the treatment effect is modified.

\paragraph{Stationary condition} \label{sec:stationary.condition}

The stationary condition for the VAR model requires that all the eigenvalues of $\bPhi$ lie inside the unit circle.
Since $\Phi_{2,1} = 0$, the eigenvalues of $\bPhi$ are simply $\lambda_1 = \Phi_{1, 1}$ and $\lambda_2 = \Phi_{2, 2}$.
Therefore, the VAR model is weakly stationary if
\begin{align*}
    \abs{ \theta^{R}_{K+2} + \sum_{k=1}^K \theta^{R}_{k} (\theta^{M}_{2} + \theta^{M}_{6} A_{d, k}) } &< 1, \\
    \abs{ \theta^{E}_{1} + \sum_{k=1}^K \theta^{E}_{K + 1 + k} A_{d,k} } &< 1,
\end{align*}
for all $A_{d, k} \in \{ 0, 1 \}$.

For each user $n \in \{1:N\}$, let $\btheta^M_n$, $\btheta^E_n$, $\btheta^R_n$, and $\btheta^O_n$ represent the coefficients in the vanilla testbed for this user.
Now consider the following changes, where $\xi_n$ is a shifting constant:
\begin{enumerate}
    \item For the arrow $A_{d, k} \to M_{d, k}$, shift $\theta^M_{4, n}$ to $\theta^M_{4, n} + \xi_n$.
    \item For the arrow $M_{d, k} \to R_d$, shift $\theta^R_{1:5, n}$ to $\theta^R_{1:5, n} + \xi_n$.
    \item For the arrow $A_{d, k} \to E_d$, shift $\theta^E_{2:6, n}$ to $\theta^E_{2:6, n} + \xi_n$.
    \item For the arrow $E_d \to R_d$, shift $\theta^R_{6, n}$ to $\theta^R_{6, n} + \xi_n$.
\end{enumerate}
Changes 1, 3, and 4 do not affect the eigenvalues.
For change 2, we need to ensure that $\abs{ \theta^{R}_{K+2} + \sum_{k=1}^K \theta^{R}_{k} (\theta^{M}_{2} + \theta^{M}_{6} A_{d,k}) } < 1$ when adjusting $\theta^R_{1:5, n}$ for $M_{d,k} \to R_d$.
If we consider the case where $A_{d,k} = 0 \text{ or } 1$ for all $k$ for simplicity, we have
\begin{equation*}
\begin{cases}
    \abs{ \theta^{R}_{K+2} + \sum_{k=1}^K (\theta^{R}_{k,n} + \xi_n) (\theta^{M}_{2} + \theta^{M}_{6}) } < 1, 
    & \text{when } A_{d,k} = 1 \text{ for all } k, \\
    \abs{ \theta^{R}_{K+2} + \sum_{k=1}^K (\theta^{R}_{k,n} + \xi_n) \theta^{M}_{2} } < 1,
    & \text{when } A_{d,k} = 0 \text{ for all } k. \\
\end{cases}
\end{equation*}

\paragraph{Limiting distribution} \label{sec:limiting.condition}

The expectation of the vector $[R_d, E_d]^T$ in the limiting distribution is given by
\begin{equation*}
    \balpha^{\text{stat}} = (\Ib - \bPhi)^{-1} \bphi.
\end{equation*}
Let $\bar{\balpha}^{\text{stat}}$ denote the expectation in the vanilla testbed and $\widetilde{\balpha}^{\text{stat}}$ denote the expectation in a testbed variant, where a positive or negative treatment effect has been added.

To ensure that the variables fluctuate within a reasonable range, suppose the requirement for the limiting distribution is $\balpha^{\text{stat}} \le \balpha^{\text{max}}$, where $\balpha^{\text{max}}$ represents the maximum value allowed for the expectation.
To check the limiting distribution of $R_d$ after varying $\theta^R_{1:5}$, the upper bound for $E_d$ can be fixed at $\alpha^{\text{max}}_2 = \bar{\alpha}^{\text{stat}}_2$.
Since $R_d$ and $E_d$ have been standardized, and the 99.7\% confidence interval for a standard normal distribution is $(-3, 3)$, we can set $\alpha^{\text{max}}_2 = 3$.
This condition is equivalent to
\begin{equation*}
    \widetilde{\bPhi} \balpha^{\text{max}} \le \balpha^{\text{max}} -  \widetilde{\bphi},
\end{equation*}
where $\widetilde{\bPhi}$ and $\widetilde{\bphi}$ are the parameter matrices of the testbed variant with varied $\theta^R_{1:5}$.
Thus, we only need to ensure that
\begin{equation*}
    \widetilde{\Phi}_{1,1} \alpha^{\text{max}}_1 
    + \widetilde{\Phi}_{1,2} \alpha^{\text{max}}_2 
    \le \alpha^{\text{max}}_1 - \widetilde{\phi}_1
\end{equation*}
after adjusting the coefficients for the reward $R_d$ to keep it within a reasonable range.

\paragraph{Standardized effect size condition} \label{sec:ste.condition}

We follow the steps below to examine the effect size:
\begin{enumerate}
    \item Choose a desired shifting constant $\xi$.
    \item Start by setting $\xi_n = \xi$ and check the stationary condition for user $n$. If the model for user $n$ becomes explosive under $\xi$, reduce $\xi_n$ to satisfy the stationary condition.
    \item Adjust the $R_{d-1} \to R_d$ coefficient $\theta^{R}_{K+2, n}$ if the expected rewards in the limiting distribution exceed the reasonable range, ensuring that the variables remain within acceptable limits.
    \item Estimate the effect size for this shifting constant $\xi$ as described in Section~\ref{sec:ste.estimate}.
\end{enumerate}
According to the guidelines in behavioral science \citep{cohen1988statistical}, an STE of 0.2, 0.5, and 0.8 corresponds to small, medium, and large effect sizes, respectively.

\subsection{Variants of the Testbed}

Based on the definition of STE in Section~\ref{sec:ste.estimate}, the STE of the vanilla testbed is 0.63, indicating a medium effect size.
We evaluate the performance of the proposed and baseline algorithms under different positive effects $A_{d, k} \to M_{d, k} \to R_d$ and negative effects $A_{d, k} \to E_{d} \to R_d$ by varying the coefficients in model~(\ref{equ:testbed.model}).
Specifically, we consider the following three changes. 
We choose to shift the coefficients by a constant rather than scaling them with a multiplier to introduce a common treatment effect across all users.

\paragraph{Increase the Positive Effect.}
We increase the positive effect $A_{d, k} \to M_{d, k} \to R_d$ through the proximal outcome $M_{d, k}$. 
Since the effect of $A_{d, k}$ on $M_{d, k}$ has been studied in the previous HeartSteps trial, our focus is primarily on the mediator effect from $M_{d, k} \to R_d$. 
To construct different variants of the testbed with varying strengths of $M_{d, k} \to R_d$, we increase the coefficients $\theta^R_{1:5}$. 
Let $\theta^R_{1:5, n}$ denote the coefficients in the vanilla testbed for user $n$.
For a constant $\xi$, we follow the procedures outlined in Section~\ref{sec:testbed.variant.requirements} to determine the shifting constant $\xi_n$ for each user and adjust $\theta^R_{1:5, n}$ to $\theta^R_{1:5, n} + \xi_n$.
In the current testbed, the stationary conditions and the requirements for the limiting distribution are satisfied for all $\xi$ values ranging from 0 to 0.05.
Therefore, in practice, we set $\xi_n = \xi$ for all users $n$.
Figure~\ref{subfig:ste.variant.MR} summarizes the estimated STE for $\xi = 0$ to 0.05.
Based on Figure~\ref{subfig:ste.variant.MR}, we select $\xi = 0.03$ so that this testbed variant has a large STE of 0.78.

\paragraph{Increase the Negative Effect.}
We increase the negative effect $A_{d,k} \to E_{d} \to R_d$ through the engagement $E_{d}$. 
To achieve this, we simultaneously decrease the coefficients $\theta^E_{2:6}$ for the arrow $E_{d} \to R_d$ and increase the coefficient $\theta^R_{6}$ for the arrow $A_{d,k} \to E_{d}$. 
According to Section~\ref{sec:testbed.variant.requirements}, modifying the arrows $A_{d,k} \to E_{d} \to R_d$ does not impact the stationary condition.
Thus, in practice, we set the constant $\xi_n = \xi$ for all users $n$.
Let $\theta^E_{2:6,n}$ and $\theta^R_{6,n}$ represent the coefficients in the vanilla testbed for user $n$.
We adjust $\theta^E_{2:6,n}$ to $\theta^E_{2:6,n} - \xi$ and $\theta^R_{6,n}$ to $\theta^R_{6,n} + 5\xi$.
Figure~\ref{subfig:ste.variant.AER} summarizes the estimated STE for $\xi = 0$ to 0.02.
Based on Figure~\ref{subfig:ste.variant.AER}, we select $\xi = 0.02$ so that this testbed variant has a small STE of 0.38.

\paragraph{Increase the Positive and Negative Effect.}
We simultaneously increase the positive effect and the negative effect while keeping the STE approximately the same as in the vanilla testbed.
Fix the coefficient $\theta^R_{1:5}$ at $\theta^R_{1:5, n} + 0.03$ for each user $n$, ensuring that the STE remains around 0.8 without any changes in the path $A_{d,k} \to E_{d} \to R_d$.
As noted in Section~\ref{sec:testbed.variant.requirements}, modifying the arrows $A_{d,k} \to E_{d} \to R_d$ does not affect the stationary condition.
Therefore, we set the constant $\xi_n = \xi$ for all users $n$ in practice.
Next, we shift $\theta^E_{2:6,n}$ to $\theta^E_{2:6,n} - \xi$ and shift $\theta^R_{6,n}$ to $\theta^R_{6,n} + 5\xi$.
Figure~\ref{subfig:ste.variant.AMER} summarizes the estimated STE for $\xi = 0$ to 0.02.
Based on Figure~\ref{subfig:ste.variant.AMER}, we select $\xi = 0.02$ so that this testbed variant has a medium STE of 0.59.

\begin{figure}[!htbp]
    \centering
    \begin{subfigure}{0.49\textwidth}
        \centering
        \includegraphics[width=0.7\textwidth]{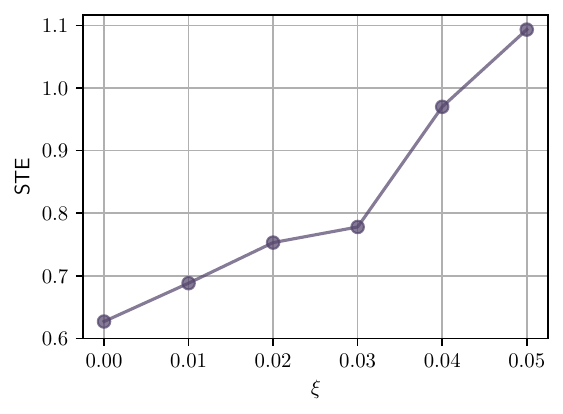}
        \caption{
        The STE for varying strengths of $M_{d,k} \to R_d$. 
        The coefficients $\theta^R_{1:5, n}$ are shifted to $\theta^R_{1:5, n} + \xi$ for each user $n$.
        \\ \;
        }
        \label{subfig:ste.variant.MR}
    \end{subfigure}
    \hfill
    \begin{subfigure}{0.49\textwidth}
        \centering
        \includegraphics[width=0.7\textwidth]{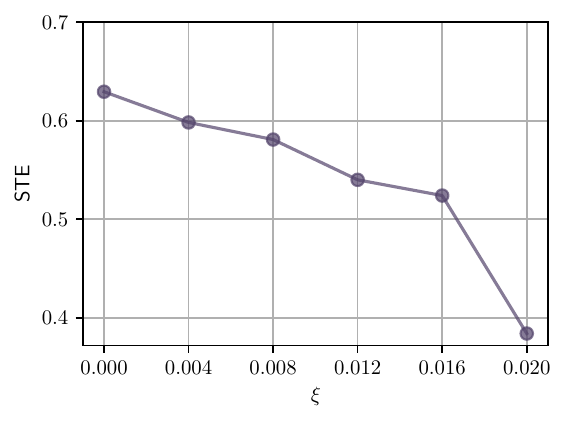}
        \caption{
        The STE for varying strengths of $A_{d,k} \to E_{d} \to R_d$. 
        The coefficients $\theta^E_{2:6, n}$ are shifted to $\theta^E_{2:6, n} - \xi$, and $\theta^R_{6, n}$ is shifted to $\theta^R_{6, n} + 5\xi$ for each user $n$.
        }
        \label{subfig:ste.variant.AER}
    \end{subfigure}
    \hfill
    \begin{subfigure}{0.49\textwidth}
        \centering
        \includegraphics[width=0.7\textwidth]{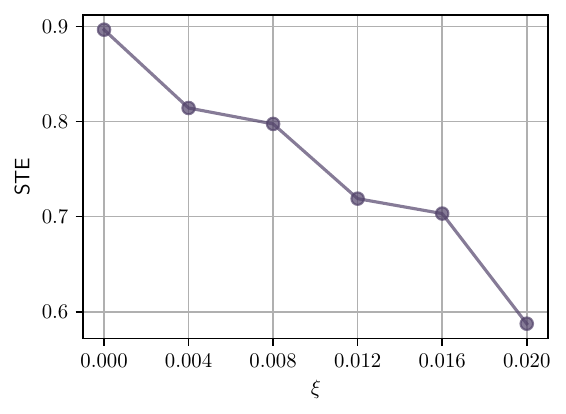}
        \caption{
        The STE for varying strengths of $A_{d,k} \to E_{d} \to R_d$ under increased positive effect $\theta^R_{1:5, n} + 0.04$.
        The coefficients $\theta^E_{2:6, n}$ are shifted to $\theta^E_{2:6, n} - \xi$, and $\theta^R_{6, n}$ is shifted to $\theta^R_{6, n} + 5\xi$ for each user $n$.
        }
        \label{subfig:ste.variant.AMER}
    \end{subfigure}
    \hfill
    \begin{subfigure}{0.49\textwidth}
        \centering
        \includegraphics[width=0.7\textwidth]{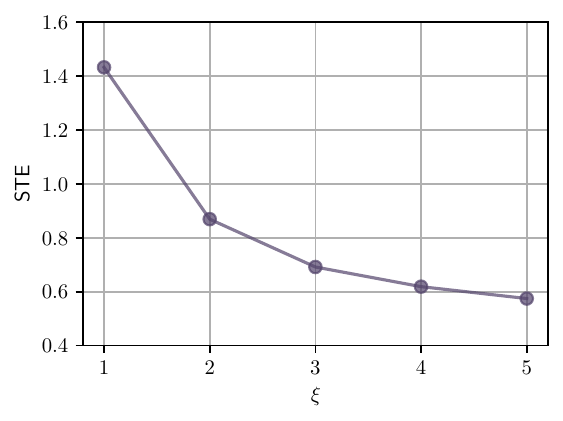}
        \caption{
        The STE for varying strengths of the direct effect $A_{d,k} \to R_d$ under a misspecified DAG. 
        The coefficients $\theta^R_{8:12, n}$ are scaled to $\theta^R_{8:12, n} / \xi$ for each user $n$.
        \\ \;
        }
        \label{subfig:ste.variant.AR}
    \end{subfigure}
    \caption{The STE for different variants of the testbed.}
    \label{fig:ste.variants}
\end{figure}

\subsection{Misspecified DAG}

To examine the impact of misspecified assumptions in Algorithm~\ref{alg:brlsvi}, we construct two variants of the testbed that violate the assumptions in the DAG in Figure~\ref{fig:dag}.

\paragraph{Add an arrow $R_{d-1} \to E_d$.}
We assume that $E_d$ satisfies
\begin{equation*}
    E_{d} = \theta^{E}_{0} + \theta^{E}_{1} E_{d-1} + \sum_{k=1}^K \theta^{E}_{k+1} A_{d,k} + \sum_{k=1}^K \theta^{E}_{K+1+k} A_{d,k} E_{d-1} + \theta^{E}_{2K+2} R_{d-1} + \epsilon^E_{d}.
\end{equation*}
All other assumptions remain the same as in model~(\ref{equ:testbed.model}).
We fit the model as described in Section~\ref{sec:model.fitting}.
The histograms of the estimated coefficients $\widehat{\btheta}^M$, $\widehat{\btheta}^E$, $\widehat{\btheta}^R$, and $\widehat{\btheta}^O$ for the 42 users are shown in Figure~\ref{fig:est.coefs.RE}.
The estimated STE of this testbed variant is 0.60, corresponding to a medium STE. 
No positive or negative effects are added in this variant.

\paragraph{Add arrows $A_{d, 1:K} \to R_d$.}
We assume that $R_d$ satisfies
\begin{equation*}
    R_{d} = \theta^{R}_{0} + \sum_{k=1}^K \theta^{R}_{k} M_{d,k} + \theta^{R}_{K+1} E_{d-1} + \theta^{R}_{K+2} R_{d-1} + \sum_{k=1}^K \theta^{R}_{K+2+k} A_{d,k} + \epsilon^R_{d}.
\end{equation*}
All other assumptions remain the same as in model~(\ref{equ:testbed.model}).
We fit the model as described in Section~\ref{sec:model.fitting}.
In the model for $R_d$, the adjacency matrix $\Omega^R$ is defined s.t. only the entries corresponding to $M_{d, k}$ and $M_{d, k^{\prime}}$ or $A_{d, k}$ and $A_{d, k^{\prime}}$ for $k \ne k^{\prime}$ are equal to one.
All other entries are set to zero.
The histograms of the estimated coefficients $\widehat{\btheta}^M$, $\widehat{\btheta}^E$, $\widehat{\btheta}^R$, and $\widehat{\btheta}^O$ for the 42 users are shown in Figure~\ref{fig:est.coefs.AR}.
Without any other modifications, the STE is 1.43, which exceeds the scientifically plausible range.
Therefore, we divide $\theta^R_{8:12}$ by a constant $\xi$, so that $\theta^R_{8:12,n}$ is adjusted to $\theta^R_{8:12,n} / \xi_n$, where $\xi_n = \xi$ for all $n$.
Figure~\ref{subfig:ste.variant.AR} summarizes the estimated STE for $\xi = 1$ to 5.
Based on Figure~\ref{subfig:ste.variant.AR}, we select $\xi = 5$ to achieve a medium STE of 0.57 for this testbed variant.

\paragraph{Add arrows $R_{d-1}, E_{d-1} \to C_{d, 1:K}$.}
We assume that $C_{d, k}$ satisfies
\begin{equation*}
    C_{d, k} = \theta^{C}_{0} + \theta^{C}_{1} E_{d-1} + \theta^{C}_{2} R_{d-1} + \epsilon^C_{d, k}.
\end{equation*}
All other assumptions remain the same as in model~(\ref{equ:testbed.model}).
We fit the model as described in Section~\ref{sec:model.fitting}.
The histograms of the estimated coefficients $\widehat{\btheta}^M$, $\widehat{\btheta}^E$, $\widehat{\btheta}^R$, $\widehat{\btheta}^O$ and $\widehat{\btheta}^C$ for the 42 users are shown in Figure~\ref{fig:est.coefs.RC}.
The estimated STE of this testbed variant is 0.58, corresponding to a medium STE. 
No positive or negative effects are added in this variant.

\begin{figure}[!htbp]
    \centering
    \includegraphics[width=\textwidth]{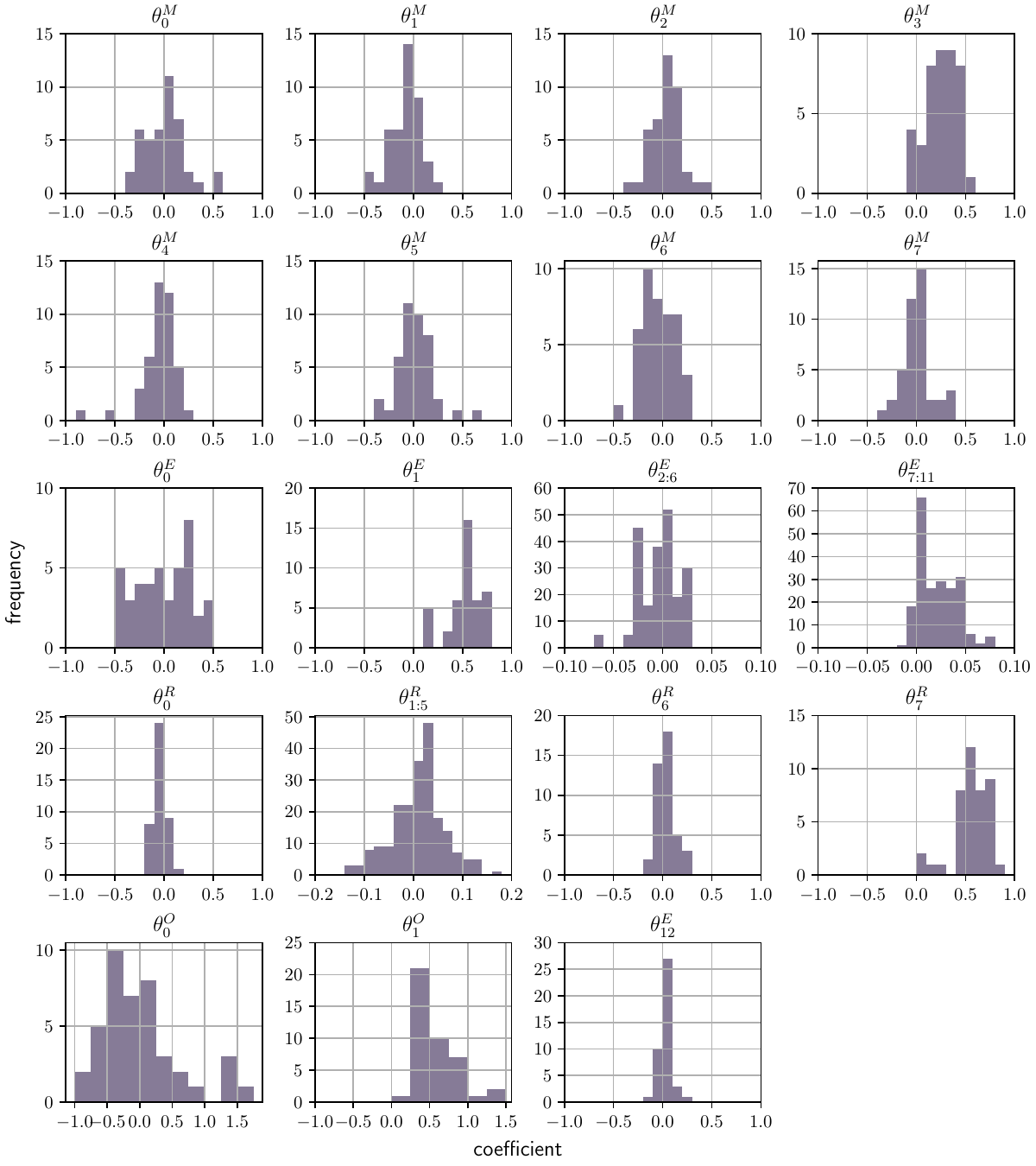}
    \caption{
    Histograms of the estimated coefficients $\widehat{\btheta}^M$, $\widehat{\btheta}^E$, $\widehat{\btheta}^R$, and $\widehat{\btheta}^O$ for the 42 users in the testbed Variant (\subref{subfig:v5}). 
    The first two rows represent the coefficients $\widehat{\btheta}^M$, while the third to fifth rows represent the coefficients $\widehat{\btheta}^E$, $\widehat{\btheta}^R$, and $\widehat{\btheta}^O$, respectively.
    The coefficients $\theta^E_{2:6}$, $\theta^E_{7:11}$, and $\theta^R_{1:5}$, corresponding to $K$ actions or $K$ proximal outcomes, are plotted in the same histogram.
    The additional parameter $\theta^E_{12}$ that represents the effect from $R_{d-1}$ to $E_d$ is shown in the last subplot.
    }
    \label{fig:est.coefs.RE}
\end{figure}

\begin{figure}[!htbp]
    \centering
    \includegraphics[width=\textwidth]{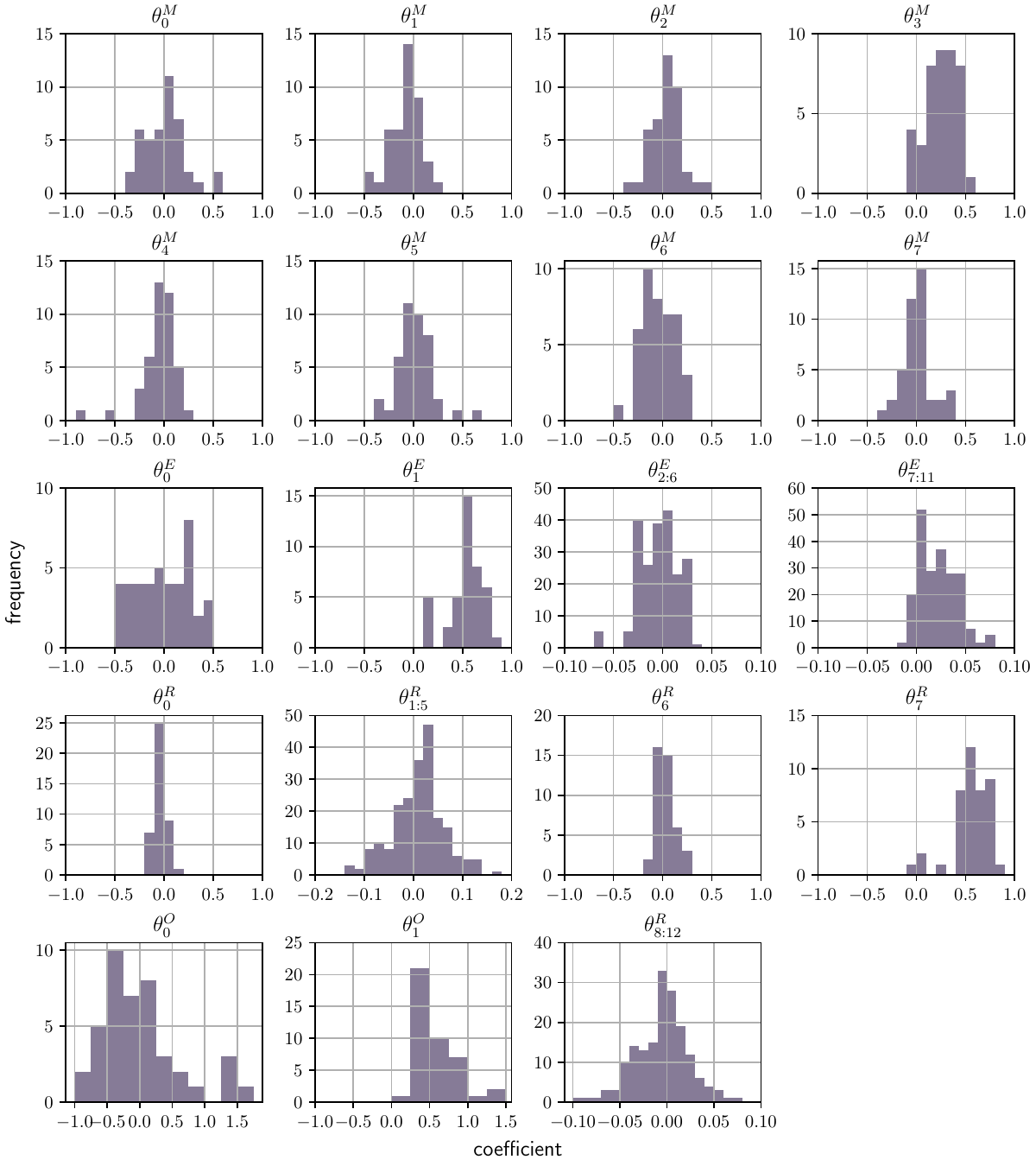}
    \caption{
    Histograms of the estimated coefficients $\widehat{\btheta}^M$, $\widehat{\btheta}^E$, $\widehat{\btheta}^R$, and $\widehat{\btheta}^O$ for the 42 users in the testbed Variant (\subref{subfig:v6}). 
    The first two rows represent the coefficients $\widehat{\btheta}^M$, while the third to fifth rows represent the coefficients $\widehat{\btheta}^E$, $\widehat{\btheta}^R$, and $\widehat{\btheta}^O$, respectively.
    The coefficients $\theta^E_{2:6}$, $\theta^E_{7:11}$, and $\theta^R_{1:5}$, corresponding to $K$ actions or $K$ proximal outcomes, are plotted in the same histogram.
    The additional parameters $\theta^R_{8:12}$ that represent the effect from $A_{d, 1:K}$ to $R_d$ are shown in the last subplot.
    }
    \label{fig:est.coefs.AR}
\end{figure}

\begin{figure}[!htbp]
    \centering
    \includegraphics[width=0.97\textwidth]{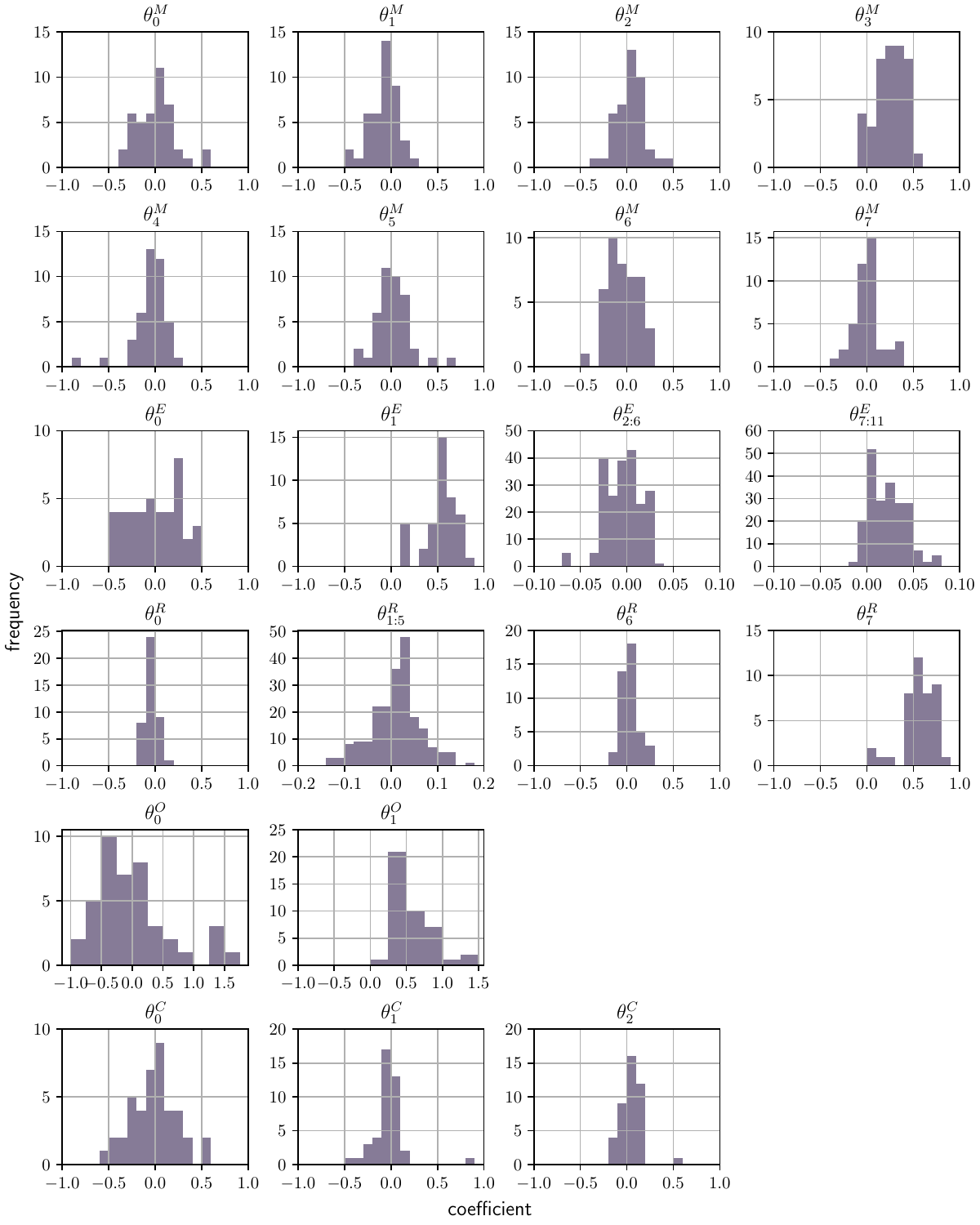}
    \caption{
    Histograms of the estimated coefficients $\widehat{\btheta}^M$, $\widehat{\btheta}^E$, $\widehat{\btheta}^R$, $\widehat{\btheta}^O$, and $\widehat{\btheta}^C$ for the 42 users in the testbed Variant (\subref{subfig:v6}). 
    The first two rows represent the coefficients $\widehat{\btheta}^M$, while the third to fifth rows represent the coefficients $\widehat{\btheta}^E$, $\widehat{\btheta}^R$, and $\widehat{\btheta}^O$, respectively.
    The coefficients $\theta^E_{2:6}$, $\theta^E_{7:11}$, and $\theta^R_{1:5}$, corresponding to $K$ actions or $K$ proximal outcomes, are plotted in the same histogram.
    The additional parameters $\theta^C_{0:2}$ are shown in the last line.
    }
    \label{fig:est.coefs.RC}
\end{figure}

\subsection{Previous Mediators Affect the Optimal Actions} \label{sec:interaction.effect.MA}

To demonstrate the advantage of constructing the state $S^{\prime}_{d, k}$ with previous mediators $M_{d, 1:(k-1)}$, we assume that
\begin{equation} \label{equ:testbed.model.MA}
    \begin{split}
        M_{d, k} = & \theta^{M,k}_{0} + \theta^{M,k}_{1} E_{d-1} + \theta^{M,k}_{2} R_{d-1} + \theta^{M,k}_{3} C_{d, k} \\
        & + A_{d, k} \prth{\theta^{M,k}_{4} + \theta^{M,k}_{5} E_{d-1} + \theta^{M,k}_{6} R_{d-1} + \theta^{M,k}_{7} C_{d, k} + \sum_{j=1}^{k - 1} \theta^{M,k}_{7 + j} M_{d, j}} + \epsilon^M_{d, k}.
    \end{split}
\end{equation}
All other assumptions remain the same as in model~(\ref{equ:testbed.model}).
Since the coefficients in (\ref{equ:testbed.model.MA}) depend on the time $k$, we fit a ridge regression separately for each $k$ with $L_2$ penalty.
After fitting the model, all coefficients have been clipped to the range $[-1, 1]$.
Due to the limited sample size at each time $k$ of each user, we have also included the data when users are not available when estimating the coefficients.
The histograms of the estimated coefficients $\widehat{\btheta}^{M,k}$ are shown in Figure~\ref{fig:est.coefs.MA}.
When generating the step count $M_{d, k}$ using the testbed, the truncation in (\ref{equ:step.truncation}) is removed in this variant.
Similar as described in Section~\ref{sec:data.generation}, the residuals $\br^{M, k}$ are sequentially used as the noise terms $\epsilon^M_{d, k}$.
If the residual is missing at some time, the noise $\epsilon^M_{d, k}$ is sampled from observed values in $\br^{M, k}$ at the same time $k$.

\begin{figure}[!htbp]
    \centering
    \includegraphics[width=0.68\textwidth]{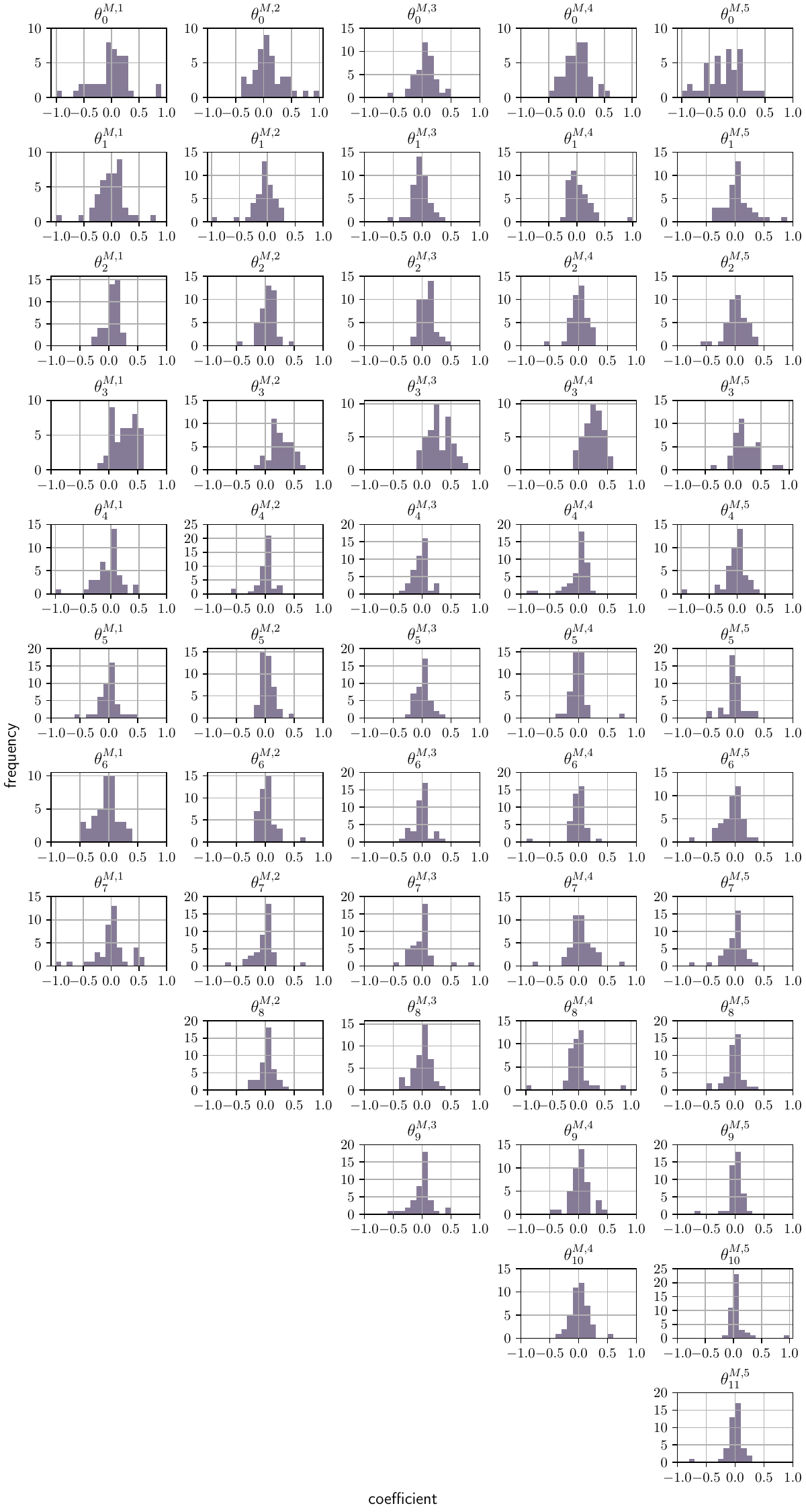}
    \caption{
    Histograms of the estimated coefficients $\widehat{\btheta}^M$ for the 42 users under model (\ref{equ:testbed.model.MA}). 
    Each column represents the coefficients of a time $k$.
    }
    \label{fig:est.coefs.MA}
\end{figure}

\end{document}